\documentclass{article}

\usepackage{microtype}
\usepackage{graphicx}
\usepackage{subfigure}
\usepackage{booktabs} 

\usepackage{hyperref}

\usepackage[accepted]{icml2024}

\usepackage{amsmath}
\usepackage{amssymb}
\usepackage{mathtools}
\usepackage{amsthm}

\usepackage[capitalize,noabbrev]{cleveref}

\theoremstyle{plain}
\newtheorem{theorem}{Theorem}[section]
\newtheorem{proposition}[theorem]{Proposition}
\newtheorem{lemma}[theorem]{Lemma}

\theoremstyle{definition}
\newtheorem{definition}[theorem]{Definition}
\newtheorem{assumption}[theorem]{Assumption}
\newtheorem{abstraction}[theorem]{Abstraction}
\newtheorem{example}[theorem]{Example}
\theoremstyle{remark}

\usepackage[textsize=tiny]{todonotes}

\usepackage{enumitem}
\def\cA{{\mathcal{A}}}

\def\cD{{\mathcal{D}}}
\def\cE{{\mathcal{E}}}

\def\cH{{\mathcal{H}}}

\def\cM{{\mathcal{M}}}

\def\cP{{\mathcal{P}}}

\def\cV{{\mathcal{V}}}

\def\scH{{\mathscr{H}}}

\def\scM{{\mathscr{M}}}

\def\bbR{{\mathbb{R}}}

\newcommand{\E}{\mathbb{E}}
\newcommand{\op}[1]{\operatorname{#1}}

\DeclareMathOperator*{\argmin}{arg\,min}

\usepackage{bm}
\usepackage{mathrsfs}
\usepackage{dsfont}
\usepackage{rotating}

\newcommand{\yang}[1]{#1}
\newcommand{\fangcong}[1]{#1}
\newcommand{\Bij}[2]{\operatorname{Bij}(#1, #2)}

\icmltitlerunning{Relational Learning in Pre-Trained Models: A Theory from Hypergraph Recovery Perspective}

\begin{document}

\twocolumn[
\icmltitle{Relational Learning in Pre-Trained Models: \\
A Theory from Hypergraph Recovery Perspective}



\icmlsetsymbol{equal}{*}

\begin{icmlauthorlist}
\icmlauthor{Yang Chen}{sist}
\icmlauthor{Cong Fang}{sist,iai}
\icmlauthor{Zhouchen Lin}{sist,iai,pazhou}
\icmlauthor{Bing Liu}{uic}
\end{icmlauthorlist}

\icmlaffiliation{sist}{National Key Lab of General Artificial Intelligence, School of Intelligence Science and Technology, Peking University}
\icmlaffiliation{iai}{Institute for Artificial Intelligence, Peking University}
\icmlaffiliation{pazhou}{Pazhou Laboratory (Huangpu), Guangzhou, China}
\icmlaffiliation{uic}{Department of Computer Science, University of Illinois Chicago}

\icmlcorrespondingauthor{Cong Fang}{fangcong@pku.edu.cn}
\icmlcorrespondingauthor{Zhouchen Lin}{zlin@pku.edu.cn}

\icmlkeywords{Relational Learning, Pre-Trained Model, Hypergraph Recovery}

\vskip 0.3in
]



\printAffiliationsAndNotice{}  

\begin{abstract}
Foundation Models (FMs) have demonstrated remarkable insights into the relational dynamics of the world, leading to the crucial question: \emph{how do these models acquire an understanding of world hybrid relations?}  
Traditional statistical learning, particularly for prediction problems, may overlook the rich and inherently structured information from the data, especially regarding
the relationships between objects. We introduce a mathematical model that formalizes relational learning as hypergraph recovery to study pre-training of FMs. 
In our framework, the world is represented as a hypergraph, with data abstracted as random samples from hyperedges. We theoretically examine the feasibility of a Pre-Trained Model (PTM) to recover this hypergraph and analyze the data efficiency in a minimax near-optimal style.
By integrating rich graph theories into the realm of PTMs, our  mathematical framework offers powerful tools for an in-depth understanding of pre-training from a unique perspective and can be used  under various scenarios. As an example, we  extend the framework to entity alignment in multimodal learning.
\end{abstract}

\section{Introduction}\label{sec:introduction}


Foundation Models (FMs) \citep{bommasani2021opportunities, openai2023gpt4} have emerged as transformative forces in the realm of artificial intelligence, demonstrating impressive performance in various real-world tasks such as knowledge retrieval \citep{liu2023reta}, mathematics problem solving \citep{frieder2023mathematical}, coding \citep{zhang2022planning}, commonsense reasoning \citep{rajani2019explain, zhao2023large}, and text-to-image generation \citep{ramesh2021dalle, li2023gligen}. During interactions with humans, FMs seem to exhibit an understanding of real-world entities to a certain degree, engaging in reasoning based on these entities \citep{bubeck2023sparks}. For example, FMs can deduce the entity ``table" from descriptions of objects placed on it, such as a cup, book, or computer, which raises a fundamental question: \emph{how do FMs learn real-world entities from pre-training?}

To investigate the learning of entities via pre-training, a formidable challenge is to formalize how the relationships between the entities are learned from data. Traditional statistical learning, \yang{such as PAC \citep{valiant1984theory, mohri2018foundations}}, particularly in classification problems, typically treats data as pairs of objects and their corresponding labels, focusing primarily on predicting these absolute labels. However, this approach may overlook the richer, more nuanced information that data inherently carry, especially regarding the relationships between objects. For instance, an image of a camel does not just represent the animal; it may also encapsulate its context, like a desert background, offering deeper relational insights on the camel and the context objects. Similarly, in natural language processing, the meaning of a sentence transcends the mere sum of its words, revealing complex interdependencies between the entities represented by the words. \fangcong{At the same time, PTMs, such
as LLMs, often respond to complex relationships between
objects.} \yang{Recognizing this, a new mathematical model is essential to capture these critical, yet often overlooked, facets of relational learning in pre-training, crucial for understanding the capabilities and generalization of the PTMs.}

In this work, we propose a novel mathematical framework based on hypergraph recovery to more fully capture the essence of relational learning. Specifically, we abstract the world as a hypergraph: entities are nodes, and relationships between entities are hyperedges. Each hyperedge is assigned a weight, signifying the strength of the corresponding relation. We formulate relational learning from pre-training as hypergraph recovery of the world hypergraph using the information of data. We model data generation as random sampling from the hyperedges. 
This data generation process mirrors real-world data collection, where a sample represents a perception of a relation between entities, with stronger relations having a higher likelihood of being observed and recorded. 
Our framework presents two-fold advantages: 1) In contrast to traditional statistical learning, our framework adopts a more nuanced approach. It goes beyond merely capturing individual labels within each data sample, delving into the interrelations between entities. This method yields a richer and more holistic understanding of relational learning in pre-training scenarios. 2) Additionally, the framework integrates rich graph theories into the field of PTMs. This integration invokes powerful analytical tools, providing a novel perspective for relational learning.

Based on the framework, we can answer two important questions about relational learning in PTMs: 1) Identification: Does the data provide sufficient information for relational learning? 2) Data efficiency: If so, what is the essential amount of data required? 
For the first question, we approach it as an estimation problem within a hypergraph framework and give an affirmative answer by demonstrating that the hypergraph can be identified from sufficient hyperedge samples. 
To address the second question, we first establish a lower bound $\Omega\left(\frac{m}{\epsilon^2}\right)$ for $\epsilon$-approximate relational learning of the hypergraph with $m$ hyperedges. 
We further investigate how a model learns relations via Masked Modeling (MM), a common practical pre-training algorithm \citep{kenton2019bert, he2022masked}. In the hypergraph recovery framework, an MM PTM learns a set of relative weight ratios between certain entity relations. We show that MM achieves the near-optimal (in terms of approximation error) sample complexity $\Tilde{O}\left(\frac{m}{\epsilon^2}\right)$, matching the information theoretical lower bound if logarithmic factors are neglected.

\yang{Our hypergraph framework is adaptable to scenarios necessitating the capture of entity relations, including multimodal entity alignment \citep{chen2020mmea, zhao2023multimodal}, social network privacy \citep{korolova2008link}, and relational reinforcement learning \citep{zambaldi2018deep}, etc., allowing for an analysis of key relational learning from pre-training data. We focus on multimodal entity alignment, demonstrating feasible alignment across modalities using sufficient unlabeled data, achieved through hypergraph matching. Although aligning without labeled pairs is theoretically possible, practical computational constraints necessitate labeled pairs to reduce complexity.}

We conduct experiments to back up the validity of our hypergraph formulation for relational learning in PTMs. In the first experiment of synthetic relational learning, we create synthetic entities whose relations compose weighted graphs, showing the power of MM for learning the synthetic relations.
In the second experiment, we examine real-world relational learning of LLMs by evaluating their relational subgraphs and measuring how well the evaluated subgraphs align with the real world. Our results show that the evaluated relations do align with the real world to some degree and more powerful models exhibit better alignment. 

We list the contributions of the paper as follows:
\begin{itemize}[leftmargin=*]
\item We propose a new mathematical model to formalize relational learning in PTMs, which is grounded in the principles of hypergraph recovery.
\item We demonstrate the feasibility of a learning model achieving relational learning and establish a minimax lower bound for the sample complexity involved. Additionally, we show that pre-training using Masked Modeling (MM) approaches near-optimal data efficiency in terms of approximation error within our framework.
\item We extend our framework to entity alignment in multimodal learning. We 
show the feasibility of entity alignment without labeled pairs and demonstrate the role of labeled pairs in reducing the computational complexity.
\end{itemize}

\begin{figure*}[th]
    \centering
    \includegraphics[width=\textwidth]{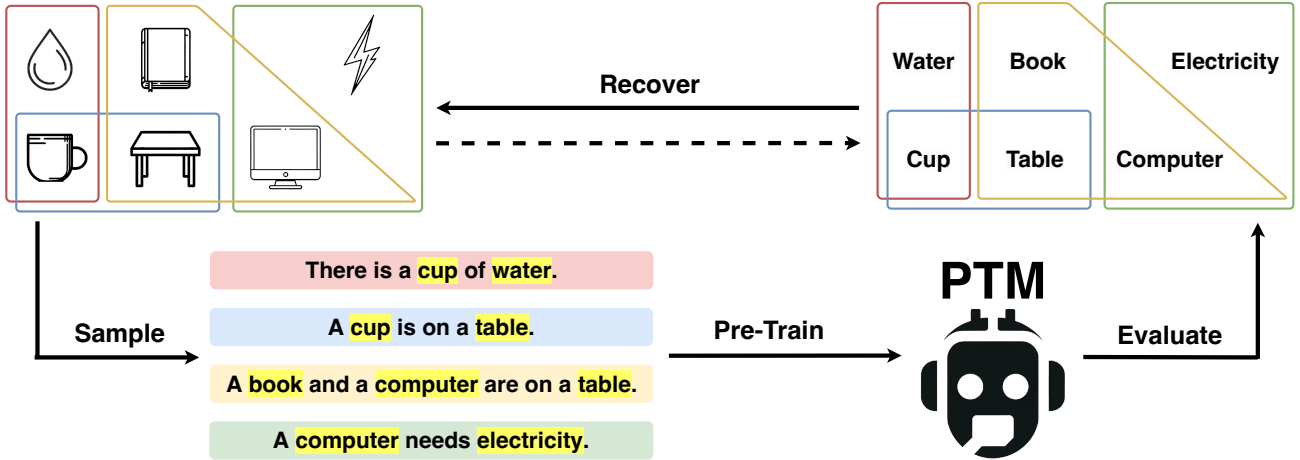}
    \caption{Our hypergraph recovery framework for relational learning in PTMs. The relational model of the world is viewed as a hypergraph. Data are generated by sampling hyperedges from the world relational model and mapping them to perception domains. PTMs learn the entity relations from the data. Recovered relational hypergraphs can be evaluated from the PTMs.}
    \label{fig:framework}
\end{figure*}
\section{Related Work}\label{sec:related}

\textbf{Graph Models.}
Graphs have long been used to characterize structures of data. For instances, 
parsing graphs use graphs to represent the grammatical dependencies of text, 
\citep{chomsky2014aspects, chen2014fast, hewitt2019structural}. 
Semantic networks model the semantic relationships between words and entities by graphical representations \citep{miller1995wordnet, speer2017conceptnet}.
Knowledge graphs represent knowledge as entities and complex relationships within graphs \citep{suchanek2007yago, lin2015learning, dettmers2018convolutional}. 
Following a similar philosophy, we model the concepts and the relations in the world as a weighted hypergraph and pre-training data as samples of hyperedges from the hypergraph. 
\fangcong{Our formulation is, instead, a simplified mathematical model to explain how pre-training can learn the complex relations in the world.}  

\textbf{Combinatorial Statistics.}
Combinatorial statistics studies the statistical properties of data with discrete structures. The most related topic in combinatorial statistics to this work is random graph isomorphism. These works model real-world problems, namely, DNA shotgun assembly \citep{idury1995new}, protein matching \citep{zaslavskiy2009global}, social network privacy \citep{korolova2008link}, etc., by random graph problems such as shotgun assembly \citep{mossel2017shotgun, ding2023shotgun} and random graph matching \citep{cullina2016improved, barak2019nearly, ding2021efficient}, exploiting both the combinatorial and statistical properties of the data. 
Our work takes a step to build the connections between combinatorial statistics and PTM capabilities, harnessing mathematical tools from the former to enhance our understanding of PTMs.

\textbf{Relational Learning.}
Relational learning focuses on identifying the relationships among entities \citep{struyf2010relational}. To understand and
exploit the relational structure of data, various relational learning techniques and methods are employed, including inductive logic programming \citep{de2008logical}, probabilistic logic learning \citep{de2008probabilistic}, relational reinforcement learning \citep{dvzeroski2001relational, zambaldi2018relational}, graph neural networks \citep{chen2021topological, fey2023relational}, etc.
While these works aim to capture entity relations more precisely, our research is dedicated to exploring the emergence of relational learning from pre-training in theory. 

\textbf{Theories of PTMs.}
Various theoretical frameworks have been proposed to elucidate the mechanisms by which PTMs leverage pre-training data and tasks to achieve generalization. 
Multi-task learning suggests that PTMs acquire generalizable representations through simultaneous training on diverse tasks \citep{ando2005framework, xie2020n, hu2021near, chen2022active, yang2022nearly}, under the assumption that these representations are the invariant components across the various tasks. Meta-learning posits that PTMs develop the ability to learn efficiently, postulating that certain meta parameters exist that enable fast adaptation to new tasks, with optimization processes geared towards these meta parameters \citep{finn2017model, finn2018probabilistic, tripuraneni2021provable}. In certain in-context learning scenarios, some in-context learning theories propose that PTMs internalize optimization or learning algorithms, facilitating task and distribution generalization \citep{akyurek2022learning, li2023transformers, von2023transformers}. This work diverges by explicitly modeling generalizable knowledge as a relational hypergraph of the world, framing pre-training as a process of hypergraph recovery.

\vspace{-0.1in}
\section{Preliminary}\label{sec:preliminary}
\textbf{Hypergraph.}
A hypergraph $\cH$ is a tuple $(\cV, \cE)$ where $\cV$ is a finite set called \emph{nodes} and $\cE$ is a family of subsets of $\cV$ called \emph{hyperedges} \citep{bretto2013hypergraph}. A weighted hypergraph $\cH$, denoted by a tuple $(\cV, \cE, w)$, is a hypergraph equipped with an additional weight function $w:\cE\mapsto\bbR_{\ge 0}$. The line graph of the hypergraph $\cH$, denoted by $L(\cH)$, is the graph whose node set is the set of the hyperedges of $\cH$ and edge set is the set of pairs of the hyperedges that intersect.
Consider transformations between hypergraphs.
Suppose that $\phi:\cV\mapsto\cV'$ is a bijection from $\cV$ to a set of nodes $\cV'$. For a hyperedge $e=\{v_1,\dots,v_k\}$, we use $\phi(e)$ to denote the hyperedge $\{\phi(v_1),\dots,\phi(v_k)\}$.
We use $\phi(\cH)$ to denote
the hypergraph $\cH'=(\cV',\cE',w')$ where $\cE'=\{\phi(e)\mid e\in\cE\}$ and 
$w'(e')=w(\phi^{-1}(e))$. 
We write $\cH_1\cong\cH_2$ if $\cH_1$ equals to $\cH_2$ up to some bijection, i.e., there exists a bijection $\phi$ such that $\phi(\cH_1)=\cH_2$.
To measure the differences between two hypergraphs $\cH_1=(\cV_1, \cE_1, w_1)$ and $\cH_2=(\cV_2, \cE_2, w_2)$, we consider the following dissimilarity measure
\begin{equation}\label{eq;dissimilarity}
    d(\cH_1,\cH_2) = \sum_{e\in\cE_1\cup\cE_2}|\bar{w}_1(e)-\bar{w}_2(e)|,
\end{equation}
where the weight function $\bar{w}_i(e)=w_i(e)$ if $e\in\cE_i$ and $\bar{w}_i(e)=0$ otherwise, $i=1,2$. This measure corresponds to the dissimilarity between two graphs constructed from the hypergraphs by the star expansion algorithm \citep{surana2021hypergraph} and captures the hyperedge weight differences between the hypergraphs.

\textbf{Notation.}
We use $A^*$ to denote the Kleene closure of set $A$, i.e., $A^*=\bigcup_{i=0}^{\infty} A^i$ where $A^0 = \{\varepsilon\}$ (the set consisting of only the empty sequence) and $A^i=\{(a_1,\dots, a_i)\mid a_j\in A,\,j=1,\dots,i\}$. We use $\op{Bij}(A, B)$ to denote the set of all bijections from set $A$ to set $B$. The notation $O(k)$ (resp., $\Omega(k)$) represents the upper bound (resp., the lower bound) of $C\cdot k$ for some constant $C$.
\section{Hypergraph Recovery Framework}\label{sec:framework}
This section introduces a mathematical framework of hypergraph recovery for relational learning in PTMs and how it could emerge from pre-training. We first model the entities and their relations in the world as a weighted hypergraph. 

\begin{abstraction}[Relational Model of the World]\label{abs:rel-mw}
    The relational model of the world is a hypergraph $\cH_0=(\cV_0, \cE_0, w_0)$, where each node $v\in\cV_0$ represents an entity, each hyperedge $e\in\cE_0$ represents a relation between entities, and the weight function $w_0: \cE\mapsto\bbR$ represents the strength of the relations.
    Without loss of generality, we assume the weight function is normalized, i.e., $\sum_{e\in\cE_0} w_0(e)=1$.
    We further assume that $|\cV_0|=n$ and $|\cE_0|=m$. 
\end{abstraction}

Since data is the perception of the world, we formalize the data generation as sampling 
from the relational hypergraph of the world, as described in \cref{abs:dg}. 

\begin{abstraction}[Data Generation]\label{abs:dg}
    In the data generation process, the entities are mapped to a perception domain (e.g., language and vision). We denote the perception mapping by $\phi_0$. In this work, we consider the perception mapping $\phi_0$ as a bijection, which keeps the structure of the relational hypergraph $\cH_0$. Each data point $e$ is a perception of the relations in the domain, corresponding to a hyperedge sampled i.i.d. from the hypergraph $\phi_0(\cH_0)$ according to the weights, i.e., $e \sim P_w(e)= w(e) = w_0(\phi_0^{-1}(e))$.
\end{abstraction}

Under this model, we define relational learning as follows.
    
\begin{definition}[Relational Learning]\label{def:rel}
    A hypergraph $\cH=(\cV, \cE, w)$ achieves relational learning for the relational model of the world if $\cH\cong\cH_0$, i.e., there exists a bijection $\phi: \cV\mapsto\cV_0$ such that $\phi(\cH)=\cH_0$.
\end{definition}

In practice, we have only finite samples and it is unrealistic to expect that the estimated relational hypergraph is completely the same as the relational model of the world.
We further define $\epsilon$-approximate relational learning to consider the approximation error of estimation with finite samples.

\begin{definition}[$\epsilon$-Approximate Relational Learning]\label{def:appx-rel}
    A hypergraph $\cH=(\cV, \cE, w)$ achieves $\epsilon$-approximate relational learning for the relational model of the world if there exists a bijection $\phi: \cV\mapsto\cV_0$ such that $d(\phi(\cH),\cH_0)\leq \epsilon$.
\end{definition}

We also say that a model $\cM$ achieves ($\epsilon$-approximate) relational learning if we can reconstruct a hypergraph that 
($\epsilon$-approximate) relational learning from the model.

\begin{definition}[($\epsilon$-Approximate) Relational Learning of Models]\label{def:rel-model}
   A model $\cM$ achieves ($\epsilon$-approximate) relational learning if there exists a testing algorithm $\cA_{\text{test}}: \scM
    \mapsto \scH$ can estimate hypergraphs from models such that $\cA_{\text{test}}(\cM)=\cH_\cM$ achieves ($\epsilon$-approximate) relational learning. Here, $\scM$ and $\scH$ denote the sets of all models and all hypergraphs of interest, respectively. 
\end{definition}

For PTMs, a typical process of relational learning is as follows: a pre-training algorithm $\cA_{\text{pre}}$ learns a model $\cM$ from a dataset $D$ and a testing algorithm $\cA_{\text{test}}$ examines whether the model achieves relational learning, i.e.,
\begin{equation}\label{eq:pipeline}
    \cH_0 \xrightarrow{\text{Sample}} D \overset{\cA_{\text{pre}}}{\longrightarrow} M \overset{\cA_{\text{test}}}{\longrightarrow} \cH.
\end{equation}

From the information perspective, whether ($\epsilon$-approximate) relational learning is achievable from a dataset $D$ is equivalent to whether there exists a pre-training algorithm and a testing algorithm that can reconstruct a relational hypergraph equal to the relational hypergraph of the world (up to some bijection). \yang{The pre-training algorithm and the testing algorithm are expected to work well for a class of target relational hypergraphs. This goal can be captured by the following minimax formula:
\begin{equation}\label{eq:minimax-infomation}
    \inf_{\cA_{\text{pre}}, \cA_{\text{test}}}\sup_{\cH_0\in\scH_0} 
    d\left(\cA_{\text{test}}\left(\cA_{\text{pre}}(D)\right), \phi_0\left(\cH_0\right)\right) \leq \epsilon,
\end{equation}
where the $\scH_0$ is the set of target relational hypergraphs.}

When we consider whether a model pre-trained by a certain algorithm can achieve relational learning, we need to consider how the pre-training algorithm can utilize the data. In this work, we consider Masked Modeling (MM), a common pre-training method that is widely used in various fields. In MM, a model is pre-trained to predict a sample $e$ based on an input $e^{-}$ that is generated by masked several tokens in $e$ according to a masking strategy $\pi=\pi(e^-\mid e)$.

\begin{abstraction}[Masked Modeling]\label{abs:mm}
     Given a masked input $e^-$, a model $\cM$ pretrained by MM complements it and outputs $e$, reflecting the model's belief $\cM(e\mid e^-)$ on 
     \begin{equation*}
         P(e\mid e^-)=\frac{w_0(\phi_0^{-1}(e))\pi(e\mid e^-)}{\sum_{e'} w_0(\phi_0^{-1}(e'))\pi(e^-\mid e')}.
     \end{equation*}
     The model predicts a hyperedge $e\sim\cM(e\mid e^-)$. With a slight abuse of notation, we denote the prediction of $\cM$ given $e^-$ by $\cM(e^-)$.
\end{abstraction}

For two hyperedges $e_1, e_2$ such that $\pi(e^-\mid e_1) > 0$ and $\pi(e^-\mid e_2) > 0$, we can further infer their relative weights from the MM model $\cM$ as $\frac{\hat{w}(e_1)}{\hat{w}(e_2)}=\frac{M(e_1\mid e^-)\pi(e^-\mid e_2)}{M(e_2\mid e^-)\pi(e^-\mid e_1)}$.
To capture such relations between two hyperedges, we define $e_1\overset{\pi}{\leftrightarrow} e_2$ if there exists a masked hyperedge $e^{-}$ such that $\pi(e^{-}\mid e_1) > 0$ and $\pi(e^{-}\mid e_2) > 0$. For the sake of notational simplicity and in cases where it does not lead to ambiguity, we use $e_1 \leftrightarrow e_2$ without the superscript $\pi$. Therefore, under our framework, we can view MM as learning the relative weights between $\leftrightarrow$ related hyperedges. 

We also abstract the data generation process of MM.
\begin{abstraction}[Masked Modeling Data Generation]\label{abs:mmdg}
    In the data generation of MM, each hyperedge $e_t$ is first sampled i.i.d. from $P_w(e)$ where $P_w(e) = w_0(\phi_0^{-1}(e))$, for all $t=1,\dots, N$. For each hyperedge $e_t$, $K$ masked hyperedges $\{e_{tk}^-\}_{k=1}^K$ are generated i.i.d. by a masking strategy $\pi$, i.e., $e_{tk}^-\sim \pi(e_{tk}^-\mid e_{tk})$ where $e_{tk}=e_t$, for all $1\leq k\leq K$. The dataset for MM is $D=\{(e_{tk},e_{tk}^-)\}_{1\leq t\leq N, 1\leq k\leq K}$.
\end{abstraction}

Under Abstractions~\ref{abs:mm} and~\ref{abs:mmdg}, an MM model $\cM$ pre-trained on $D$ with a loss $\ell$ is
\begin{equation}\label{eq:mm}
    \cM = \argmin_{\cM'\in\scM} \sum_{t=1}^N \sum_{k=1}^K \ell(\cM'(e_{tk}^-), e_{tk}).
\end{equation}

For an MM pre-trained model to achieve relational learning, it needs to learn relative weights from an MM dataset such that these relative weights amount to the recovery of the relational hypergraph $\cH_0$. 
Denote the MM pre-training algorithm in \eqref{eq:mm} by $\cA_{\text{MM}}$ under Abstractions~\ref{abs:mm} and~\ref{abs:mmdg}. Following \eqref{eq:pipeline} and \eqref{eq:minimax-infomation}, this is to consider
\begin{equation}\label{eq:minimax-mm}
    \inf_{\cA_{\text{test}}}\sup_{\cH_0} 
    d\left(\cA_{\text{test}}\left(\cA_{\text{MM}}(D)\right), \phi_0\left(\cH_0\right)\right) \leq \epsilon.
\end{equation}
\section{Main Results for Entity Relational Learning}\label{sec:entity-relational-learning}

\subsection{Identification}

We first consider whether identifying the relational hypergraph $\cH_0$ from a pre-training dataset is possible at the population level. The following theorem affirms the feasibility of relational learning if sufficient data are available.

\begin{theorem}[Identifiability]\label{thm:identifiability}
    Under Abstractions~\ref{abs:rel-mw} and~\ref{abs:dg}, suppose that ${e_t}$ is a generated data sequence. Let $D_N$ be the dataset consisting of the first $N$ elements of the sequences, i.e., $D_N=(e_1,\dots, e_N)$.
    Then there exist an pre-training algorithm $\cA_{\text{pre}}$ and a testing algorithm $\cA_{\text{test}}$, $\cA=\cA_{\text{test}}\left(\cA_{\text{pre}}(\cdot)\right):\cE^* \mapsto \scH$ such that $\cA(D_N)$ converges to a hypergraph $\cH$ that achieves relational learning as $N\to\infty$ almost surely, i.e., $\cA(D_N)\overset{\text{a.s.}}{\to}\cH\cong\cH_0$.
\end{theorem}

\yang{
\cref{thm:identifiability} asserts the asymptotic identifiability of the target hypergraph as the dataset size approaches infinity. The proof of \cref{thm:identifiability} leverages the law of large numbers to show that the distance between the estimated hypergraph and the actual relational hypergraph converges to $0$. For detailed proof, refer to \cref{appendix:proof}.}

\subsection{Data Efficiency}

Since relational learning is feasible at the population level, we then consider the data efficiency to achieve $\epsilon$-approximate relational learning at the sample level. 
We first consider an information theoretical lower bound of the sample complexity to achieve $\epsilon$-approximate relational learning.

\begin{theorem}[Information Theoretical Lower Bound]\label{thm:itlb}
    Under Abstractions~\ref{abs:rel-mw} and~\ref{abs:dg} and assuming that the generated dataset $D$ is of size $|D|=N\ge m$ with $m$ sufficiently large, the minimax risk of  reconstruction error satisfies
    \begin{equation*}
        \inf_{\cA_{\text{pre}},\cA_{\text{test}}} \sup_{\cH_0} \E_{D}\left[d(\cA_{\text{test}}\left(\cA_{\text{pre}}(D)\right),\phi_0(\cH_0))\right] \ge \frac{1}{16}\sqrt{\frac{m}{N}}.
    \end{equation*}
\end{theorem}

\cref{thm:itlb} presents an information theoretical lower bound $\Omega\left(\frac{m}{\epsilon^2}\right)$ of the sample complexity for $\epsilon$-approximate relational learning. This lower bound is derived from the sample complexity lower of the discrete distribution estimation problem under $\ell_1$ distance, by a reduction from the estimation problem to an approximate relational learning problem. 
The lower bound highlights that the number of the hyperedges $m$ is an important factor in the difficulty of relational learning.

Now we consider the data efficiency of MM to achieve $\epsilon$-approximate relational learning.
We assume that the model $\cM$ is expressive enough to fit the pre-training data, i.e., for a MM dataset $D$, the model $\cM$ pre-trained on $\cD$ satisfies
\begin{equation}\label{eq:mm-loss}
    \cM = \argmin \sum_{t=1}^N \sum_{k=1}^K \ell(\cM'(e_{tk}^-), e_{tk}).
\end{equation}

To characterize the sample complexity, we introduce the following additional assumptions.

\begin{assumption}[Range ratio of the weight function]\label{as:bwf}
    The range ratio of the weight function is
    $\kappa = \frac{\max_{e\in\cE} w(e)}{\min_{e\in\cE} w(e)}$.
\end{assumption}

\begin{assumption}[Bound on the masking strategy]\label{as:bms}
    For each hyperedge $e\in\cE$, the support set of masked hyperedges is upper bounded, i.e., $|\op{supp}\pi(\cdot\mid e)| < C_\pi$ for some constant $C_\pi$. For each $e\in\cE$ and $e^-\in\op{supp}\pi(\cdot\mid e)$, the probability $\pi(e^-\mid e)$ is lower bounded by some constant $c_\pi$.
\end{assumption}

\begin{assumption}[Bound on the MM path length]\label{as:bpath}
    For any hyperedges $e,e'\in\cE$, there exists a path bounded by $L$ such that $e=e_1\leftrightarrow e_2 \leftrightarrow \dots \leftrightarrow e_\ell=e'$.
\end{assumption}

\begin{figure*}[th]
    \centering
    \includegraphics[width=0.95\textwidth]{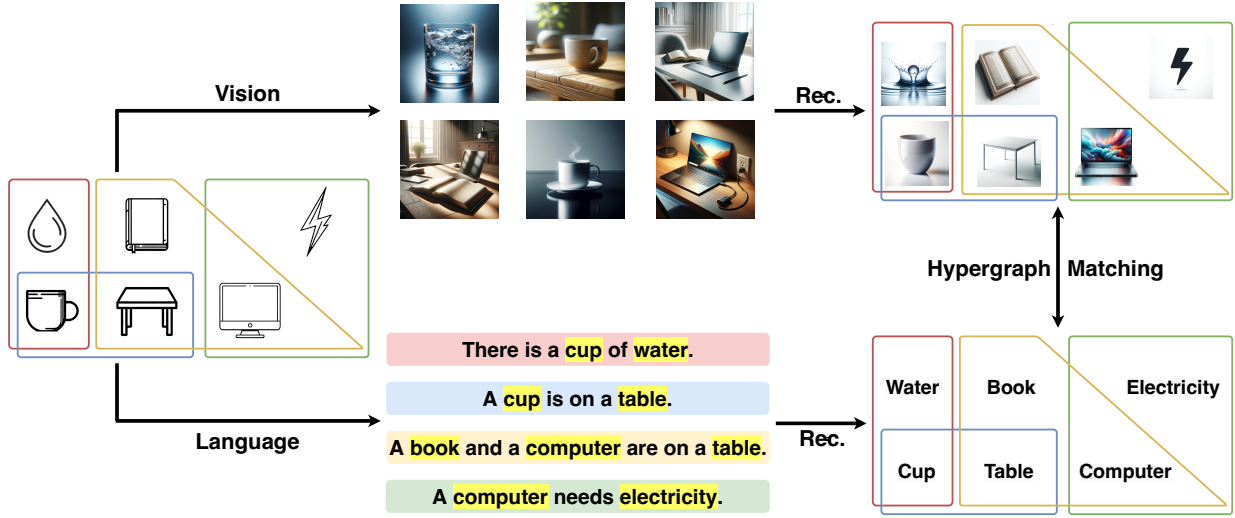}
    \caption{Extension of our hypergraph framework to entity alignment in multimodal learning (taking vision and language for illustration). The relational hypergraphs in different modalities can be reconstructed from data. The entities from different modalities can be aligned by matching the relational hypergraphs. ``Rec." represents ``Reconstruct".}
    \label{fig:entity-alignment}
\end{figure*}

\cref{as:bwf} bounds the weights of each hyperedge within a certain range. Assumption~\ref{as:bms} bounds the complexity of the masking strategy by limiting the support set of masked hyperedges and setting a minimum probability threshold for potentially masked hyperedges. \cref{as:bpath} bounds the connectivity complexity among the hyperedges under the masking strategy.

We analyze the sample complexity for the PTM pre-trained by MM $\cM$ to achieve $\epsilon$-approximate relational learning with cross-entropy loss in \cref{thm:ubmm}. 

\begin{theorem}[Upper Bound by MM]\label{thm:ubmm}
    Suppose that $\cM$ is an FM pre-trained by MM on a dataset $D$ with cross-entropy loss. Then $\cM$ achieves $\epsilon$-approximate relational learning
    with probability at least $1-\delta$ if 
    \begin{equation}
        \begin{gathered}
            K  \ge \frac{2^{14} m^2 \kappa^2 L^2}{c_\pi^2 \epsilon^2}\log \frac{6 m C_\pi}{\delta},\\
            N \ge \max\left\{\frac{2 m \kappa}{c_\pi}\log \frac{3 m C_\pi}{\delta}, \frac{8 m}{\epsilon^2}\log\frac{6 m}{\delta}\right\}.
        \end{gathered}
    \end{equation}
\end{theorem}

\yang{
In scenarios defined by specific problems and masking strategies, the term $\tilde{O}\left(\frac{m}{\epsilon^2}\right)$ predominates at low approximation errors, especially when $\epsilon=o\left(\sqrt{\frac{c_\pi}{\kappa}}\right)$. 
This aligns with the information theoretical lower bound $\Omega\left(\frac{m}{\epsilon^2}\right)$ in \cref{thm:itlb}, disregarding the logarithmic factor. This suggests that MM is near-optimal in data efficiency.}

To prove \cref{thm:ubmm}, we design an algorithm that computes the relative weights between the pairs of the hyperedges along $e_1\leftrightarrow\dots\leftrightarrow e_{\ell}$ paths. By normalization, we obtain an estimation of the hyperedge weights and further a recovered hypergraph from the relative weights. We show that when the dataset $D$ is sufficiently large, the model $\cM$ can learn all the relative weights well enough and therefore the reconstructed hypergraph is a good approximation for the relational hypergraph $\cH_0$ (up to some bijection).

\cref{thm:ubmm} reveals that the data efficiency to achieve relational learning is predominantly influenced by three factors: the number of hyperedges $m$, the range ratio of the weight function $\kappa$, and the upper bound of the MM path lengths $L$. The number of hyperedges $m$ and the range ratio of the weight function $\kappa$ characterize the complexity of the world relational hypergraph, i.e., the hypergraph with more hyperedges and a larger range ratio requires more samples to be recovered by MM. 
The MM path length bound $L$ reflects the connectivity under the masking strategy $\pi$, influencing how MM learns the relative weights between hyperedges.
Efficient recovery of the relational hypergraph is contingent on a small $L$, indicating well-connected hyperedges; a large $L$ suggests inefficiency in recovery.  This aligns with empirical observations that effective MM performance requires masking a sufficient proportion of each sample \citep{he2022masked, wettig2023should}. 

\begin{figure*}[th]\label{fig:srl}
    \centering
    \subfigure[Different numbers of edges.]{
        \includegraphics[width=0.31\linewidth]{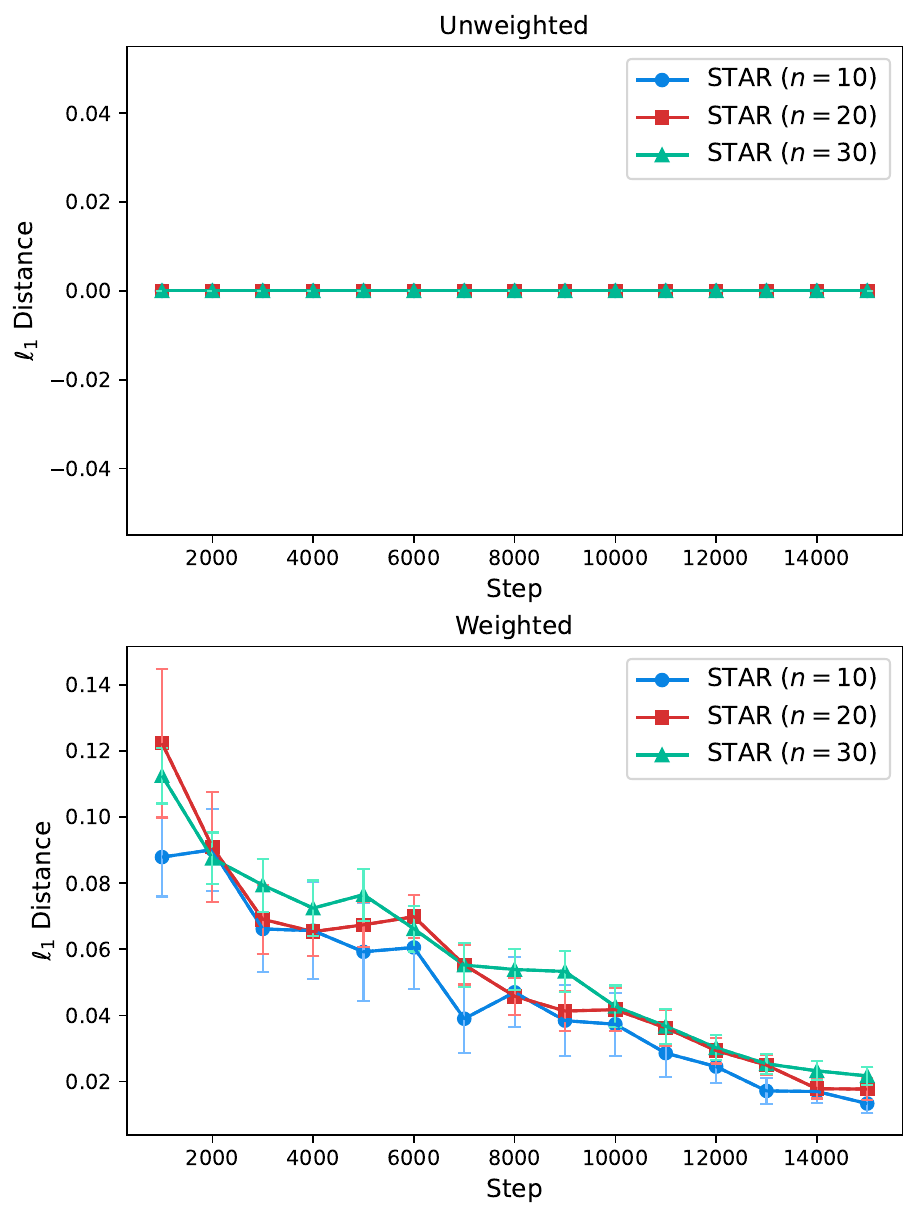}
    }
    \subfigure[Different range ratios.]{
        \includegraphics[width=0.31\linewidth]{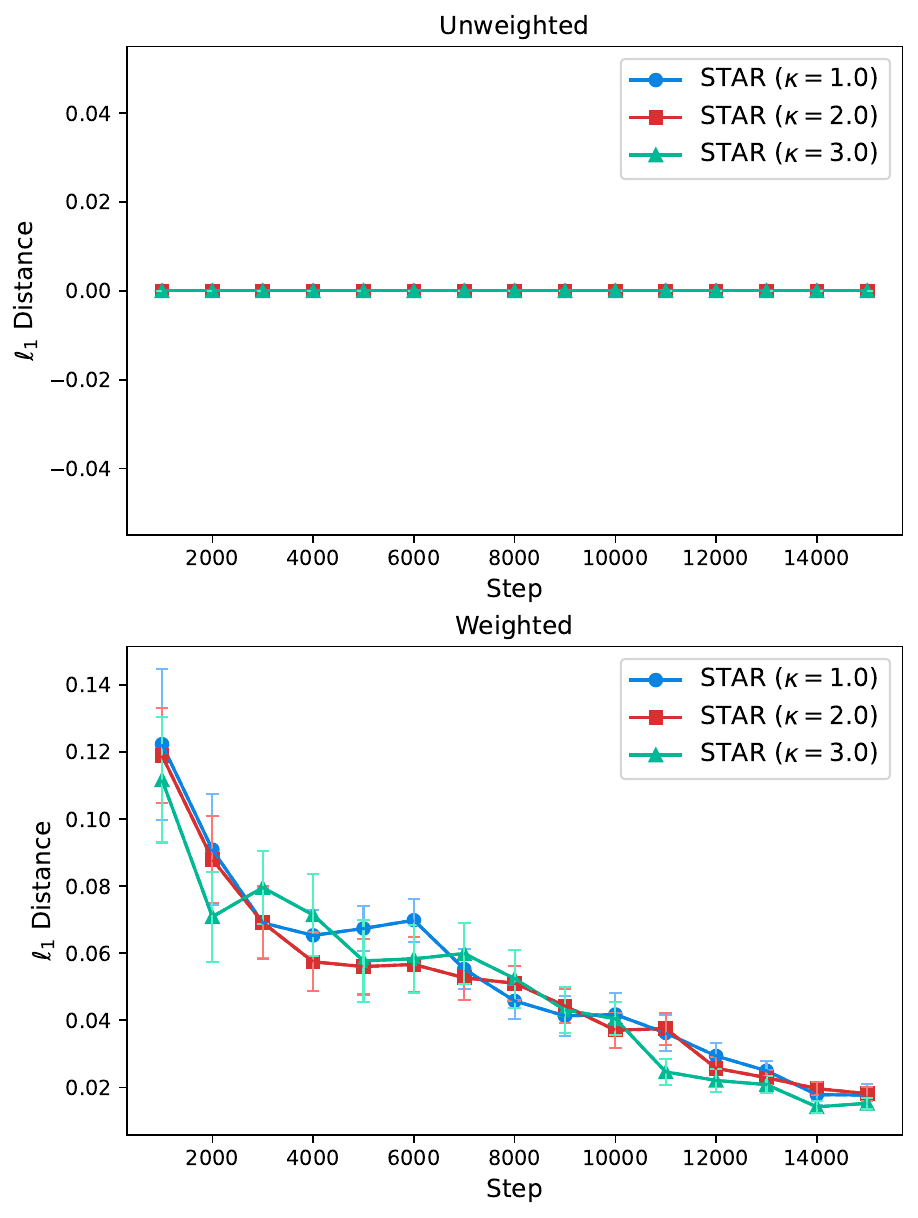}
    }
    \subfigure[Different MM path lengths.]{
        \includegraphics[width=0.31\linewidth]{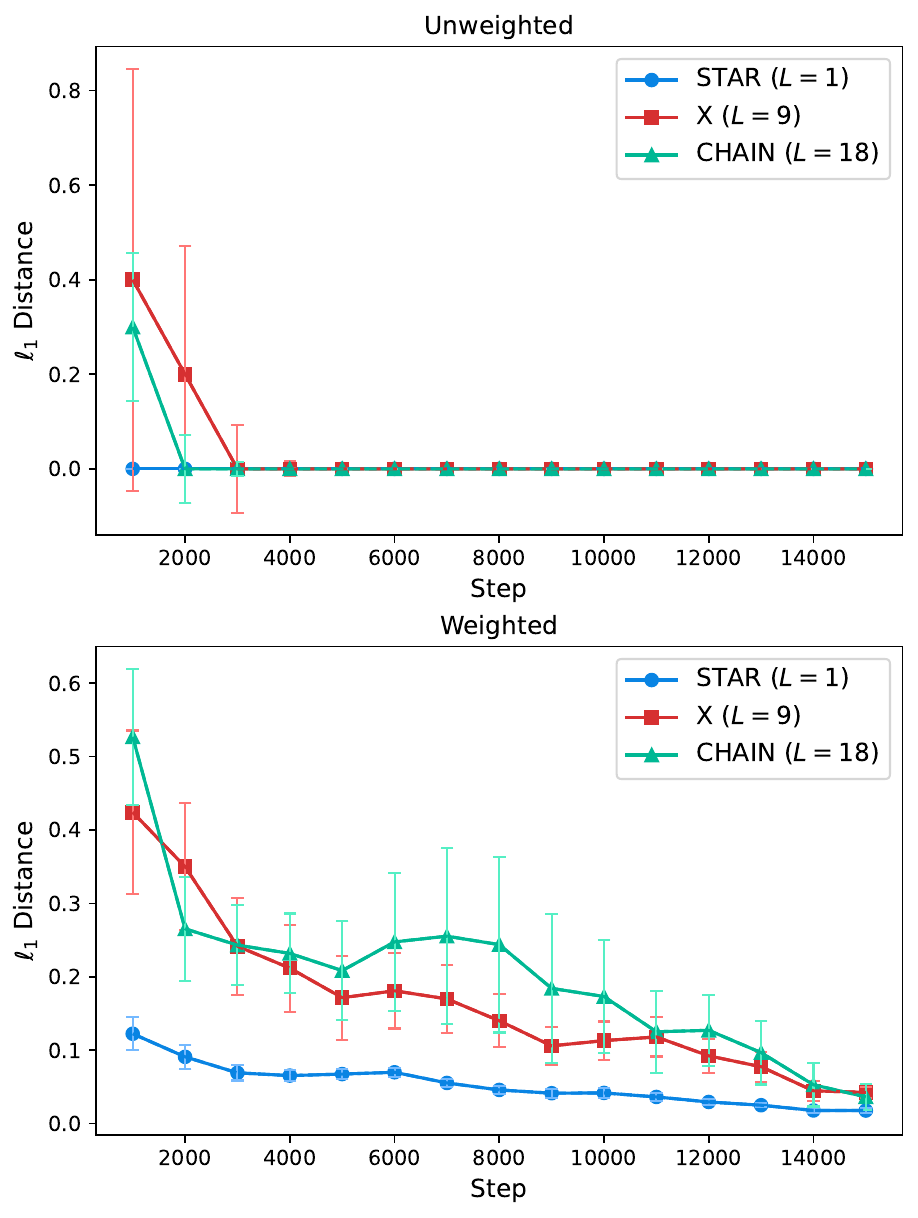}
    }
    \caption{Evaluation results of synthetic relational learning. (a) STAR graphs with different numbers of edges ($m=n-1$). (b) STAR graphs with different range ratios. (c) Graphs with different MM path lengths. For each, the experiments are repeated for $5$ times and the evaluation results are averaged over the $5$ trials.}
    \vskip -0.3in
\end{figure*}

\vspace{-0.2in}
\section{Main Results for Entity Alignment}\label{sec:entity-alignment}
We further extend our framework to encompass entity alignment within the realm of multimodal learning. In this context, the relational models associated with different modalities are interpreted as distinct representations or ``images" of the relational model of the world, each shaped by its unique perception mapping. Although our focus here is on two modalities for illustrative purposes, the principles and methodologies we discuss are readily generalizable to scenarios involving a greater number of modalities.

Concretely, the relational hypergraph in modality $i$ is mapped from $\cH_0$ by the perception $\phi_i$, i.e., $\cH_i = \phi_i (\cH_0)$ for $i=1,2$. Entity alignment is to find a bijection $\phi\in\Bij{\cV_1}{\cV_2}$ such that $\phi(\cH_1)=\cH_2$. 
The data supporting entity alignment consists of three parts: $D_1$, $D_2$, and $D_2$. Here, $D_1$ and $D_2$ represent data from the two individual modalities, while $D_{12}$ comprises labeled pairs that denote corresponding relationships across the modalities. For example, in aligning entities between visual and linguistic modalities, the data includes images, text, and labeled pairs that link images with their textual descriptions.



Assuming the data from each modality are sufficient, we can recover the relational hypergraphs $\cH_1$ and $\cH_2$. 
Entity alignment is achieved by solving the optimization problem:
\begin{equation}\label{eq:alignment}
\phi^* = \argmin_{\phi \in \Bij{\cV_1}{\cV_2}}d(\phi(\cH_1), \cH_2)
\end{equation}
\vspace{-0.02in}
\yang{Practically, labeled pairs are typically necessary to address the computational difficulty of the graph isomorphism problem in \eqref{eq:alignment}, as no polynomial-time solution has been found to date \citep{babai2016graph, neuen2018exponential}. Labeled pairs are external information that pinpoints partial correspondences between the entities of different modalities, potentially reducing the computational complexity. For example, the labeled pairs can reduce dimensions of Weisfeiler-Lehman methods required \citep{cai1992optimal} or prune search trees in individualization-refinement algorithms \citep{mckay2014practical} (See \cref{appendix:entity-alignment} for further illustration).
}

When the underlying hypergraph structure has no automorphism, it is possible to align the entities without estimating the weighted relational hypergraph in each domain. For instance, we can first estimate the underlying unweighted hypergraphs and then align the entities by solving the graph isomorphism problem for these unweighted hypergraphs. This approach can enhance relational learning in multimodal models, as the fusion of data from different modalities can complement and augment the information within each modality. \cref{prop:multimodal-learning} describes the information gain brought by the fusion of two modalities.

\begin{proposition}\label{prop:multimodal-learning}
    Suppose that $D_i$ is the dataset in modality $i$ for $i=1,2$. Assume that the entity alignment $\phi^*\in\Bij{\cV_1}{\cV_2}$ has been estimated in prior. Suppose that $\cM$ is a multimodal pre-trained model by MM on the datasets $D_1$ and $D_2$. Then $\cM$ achieves $\epsilon$-approximate relational learning with probability at least $1-\delta$ if 
    \begin{equation*}
        \begin{gathered}
            K_1 + K_2  \ge \frac{2^{14} m^2 \kappa^2 L^2}{c_\pi^2 \epsilon^2}\log \frac{6 m C_\pi}{\delta},\\
            N_1 + N_2 \ge \max\left\{\frac{2 m \kappa}{c_\pi}\log \frac{3 m C_\pi}{\delta}, \frac{8 m}{\epsilon^2}\log\frac{6 m}{\delta}\right\}.
        \end{gathered}
    \end{equation*}
\end{proposition}
\begin{figure*}[th]
    \centering
    \subfigure[Ground Truth]{
        \includegraphics[width=0.23\linewidth]{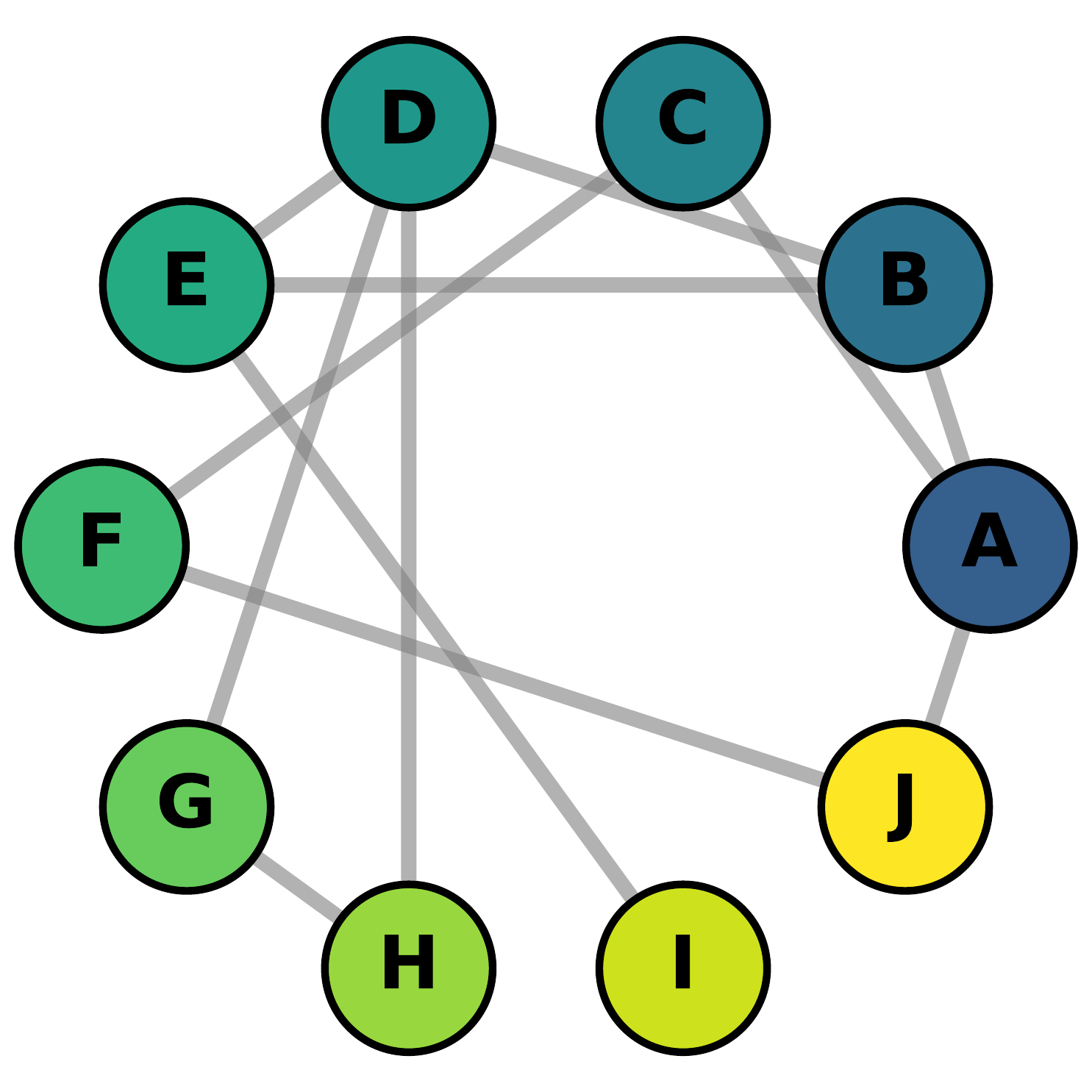}
    }
    \subfigure[LLAMA-2-70B]{
        \includegraphics[width=0.23\linewidth]{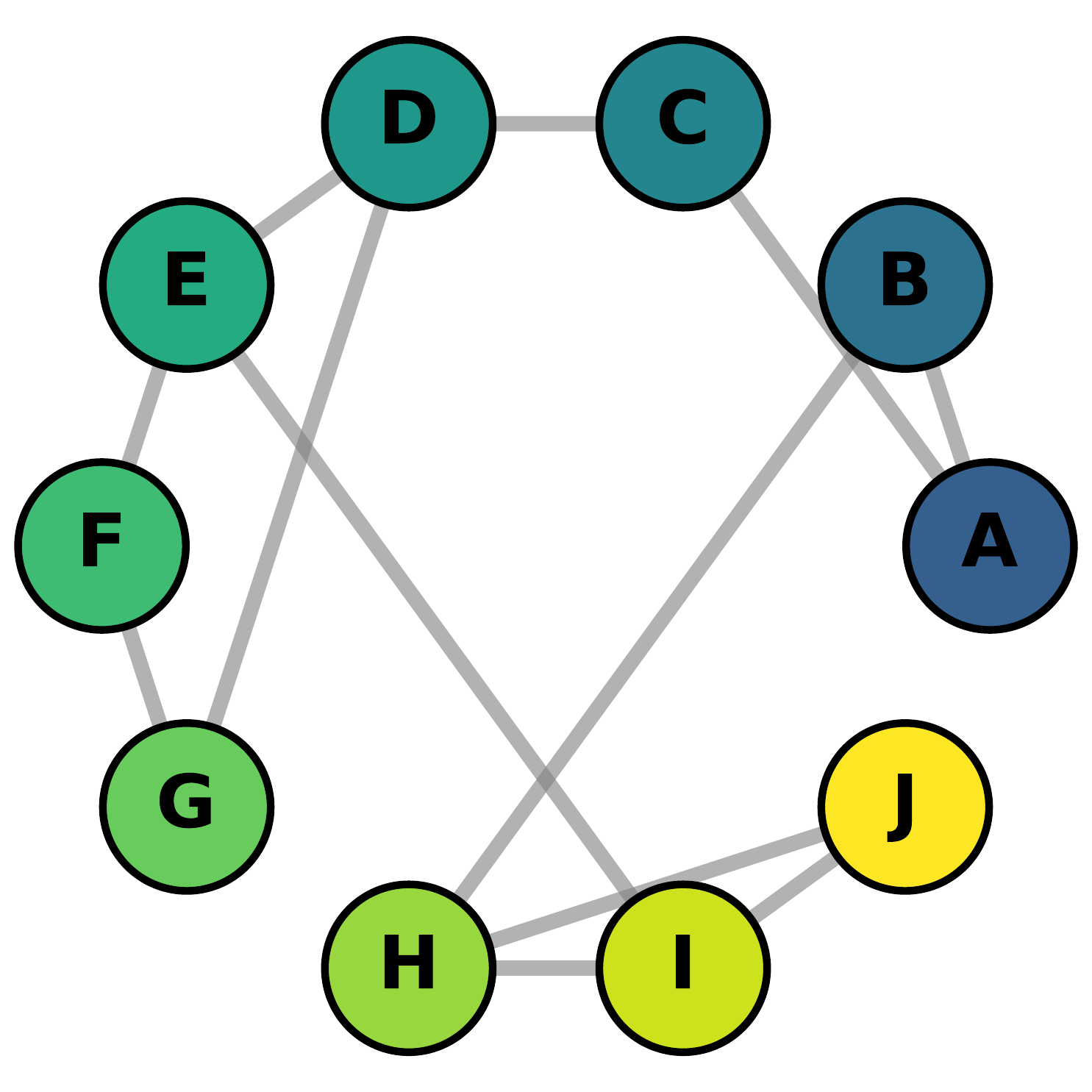}
    }
    \subfigure[GPT-3.5]{
        \includegraphics[width=0.23\linewidth]{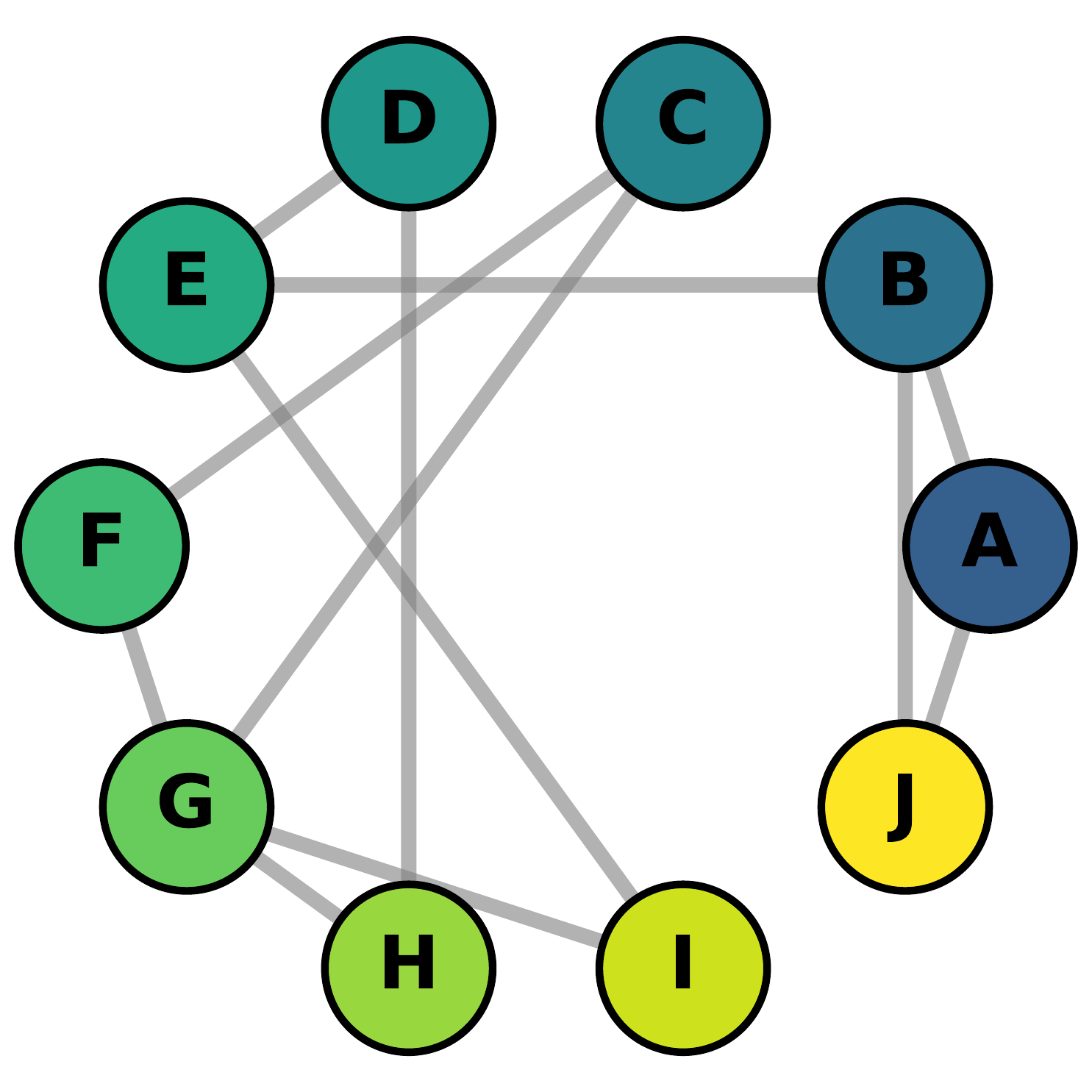}
    }
    \subfigure[GPT-4]{
        \includegraphics[width=0.23\linewidth]{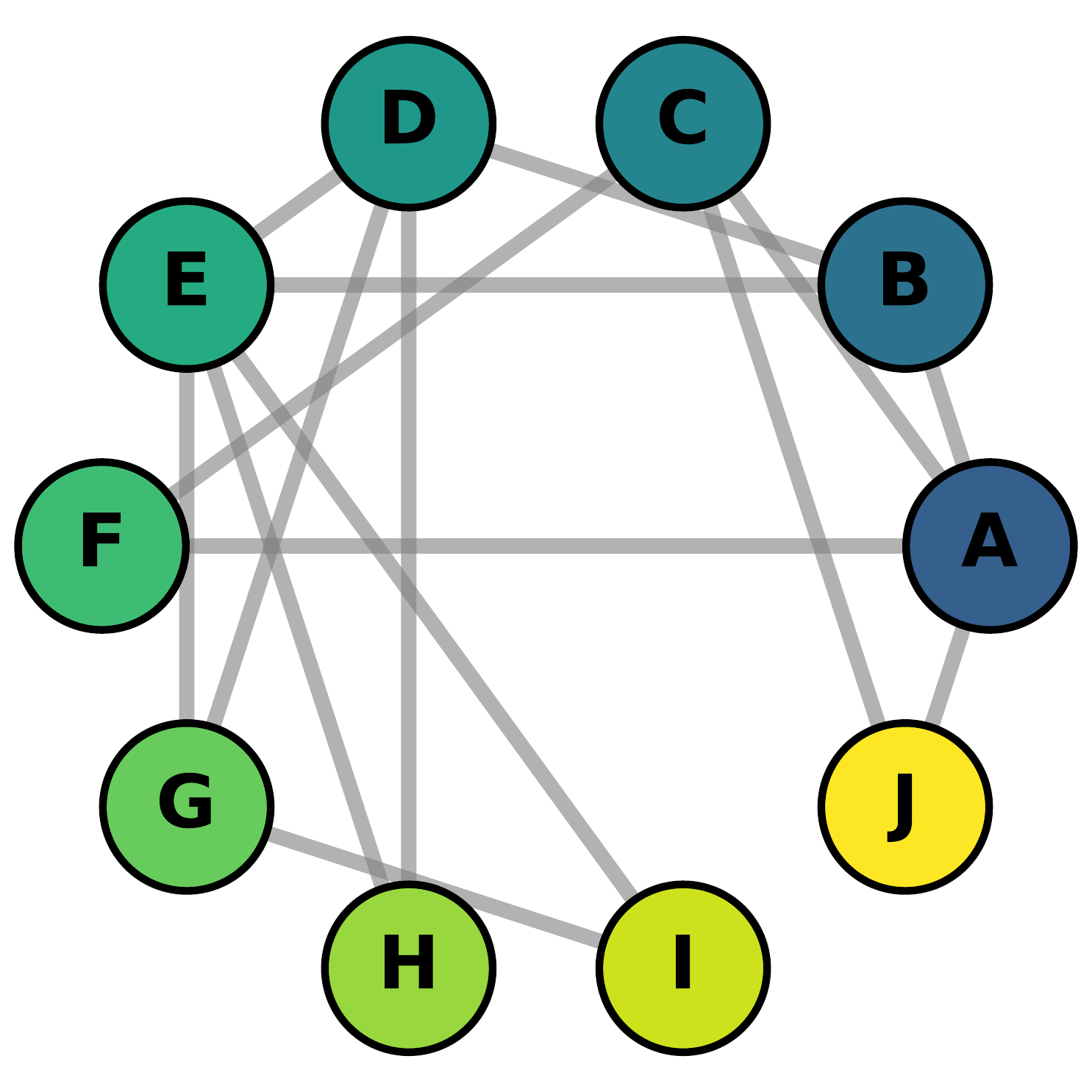}
    }
    \caption{Evaluation results of different LLMs for the real-world relational subgraph generated from the source word ``table". We use different letters to represent different entities (see \cref{subappendix:rwre} for their correspondences). The graphs (from left to right) are the ground truth (extracted from ConceptNet), evaluation results of LLAMA-2-70B, GPT-3.5, and GPT-4, respectively. }
    
    \label{fig:rwre}
\end{figure*}

\begin{table*}[th]
\caption{Summary of the comparison results. The subgraphs are generated from different source entities with $k=2$ and $d=3$. The corresponding evaluated graphs are generated from the outputs of different LLMs. The dissimilarity between each pair of the extracted subgraph $\cH$ and the estimated graph $\cH'$ are measured by their normalized $L_1$ distance, i.e., $\frac{\|\cH-\cH'\|_1}{\|\cH \|_1}$ where we slightly abuse the notations $\cH$ and $\cH'$ to denote their adjacent matrices.}
\label{tab:exp-res-k2d3}
\begin{center}
\begin{sc}
\begin{tabular}{lcccccccccc}
\toprule
 & cake & dog & fly & human & jacket & orange & paper & sea & table & zoo \\
\midrule
Llama-2-70B & $1.00$ & $\bm{0.67}$ & $1.25$ & $\bm{1.00}$ & $1.33$ & $1.33$ & $\bm{0.75}$ & $\bm{0.83}$ & $1.25$ & $1.67$ \\
GPT-3.5     & $\bm{0.67}$ & $1.00$ & $\bm{1.00}$ & $1.25$ & $\bm{1.00}$ & $1.33$ & $\bm{0.75}$ & $\bm{0.83}$ & $1.00$ & $\bm{1.00}$ \\
GPT-4       & $\bm{0.67}$ & $\bm{0.67}$ & $\bm{1.00}$ & $1.50$ & $1.33$ & $\bm{1.00}$ & $\bm{0.75}$ & $\bm{0.83}$ & $\bm{0.75}$ & $1.33$ \\
\bottomrule
\end{tabular}
\end{sc}
\end{center}

\end{table*}
\vspace{-0.2in}

\section{Experiments}\label{sec:experiment}

We conduct two experiments to show empirically that relational learning in PTMs could be seen as relational hypergraph recovery. We consider two settings: synthetic relational learning and real-world relation evaluation.
\vspace{-0.02in}
\subsection{Synthetic Relational Learning}\label{subsec:experiments-srl}
\vspace{-0.03in}
In synthetic relational learning, we train PTMs with text consisting of synthetic entities, whose underlying data distribution corresponds to a graph. We show that PTMs can learn the relations between these synthetic entities. To generate data for synthetic relational learning, we first construct a graph, whose nodes are entities (represented by tokens) and edges are relations. We attach edges with random weights and normalize the weights. To generate a training dataset, we sample edges i.i.d. according to the distribution corresponding to the normalized edge weights. We consider masked language modeling \citep{kenton2019bert}. 
For evaluation, we query the PTM with each synthetic entity to retrieve information about its related entities and the weights of the relations. We reconstruct a graph with the query results and compare the reconstructed graph with the true underlying graph. We conduct experiments for different graphs, with different numbers of edges, range ratios, and MM path lengths, corresponding to the factors that influence the sample complexity of entity relational learning. 
More details of the synthetic relational learning experiments can be found in \cref{subappdendix:srl}. 
The evaluation results are shown in \cref{fig:srl}. 
\yang{Our results show that the reconstruction errors of both the unweighted sketch graph and the weighted graph decrease as the training goes on. This the PTMs learn the synthetic relations gradually via MM pre-training. Additionally, the results suggest that larger numbers of edges and larger MM path lengths lead to more steps to converge, which coincides with our theoretical analysis in \cref{thm:ubmm}. The effect of the range ratios on the convergence of relational learning is not obvious in our experiments. This may suggest a gap between the theoretical upper bound and the actual convergence rate in the experiments in terms of the range ratio.}
\vspace{-0.02in}
\subsection{Real-World Relation Evaluation}\label{subsec:experiments-rwre}
\vspace{-0.03in}
In real-world relation evaluation, we test whether LLMs such as ChatGPT and GPT-4 learn entities and their relations that align with the real world. We use subgraphs extracted from ConceptNet \citep{speer2017conceptnet} as baselines of the real-world relations graphs. 
For evaluation, we input the chosen entities to LLMs and ask them to choose top-related ones for each entity. We then construct a graph whose nodes are the entities and edges are those top-related pairs. 
We compare the subgraph extracted from ConceptNet and the graph evaluated from LLMs. If an LLM learns real-world relations, we expect it to produce a similar graph as the one extracted from ConceptNet. \cref{tab:exp-res-k2d3} summarizes some comparison results of the extracted subgraphs generated by different source entities and the corresponding evaluated graphs. In \cref{fig:rwre}, we visualize the result of the source entity ``table". More results are presented in \cref{subappendix:rwre}. We find that GPT-4 achieves the best overall performance among the evaluated LLMs and GPT-3.5 performs slightly better than LLAMA-2-70B. The results suggest different LLMs have different degrees of relational learning and more powerful models seem to understand entity relations better in the sense of relational subgraph reconstruction. Note that we only consider unweighted graphs here because it is difficult to evaluate the relation weights from LLMs accurately. Our results illustrate that the LLMs do organize entities similarly to real-world entities. 

\section{Conclusion and Outlook}\label{sec:conclusion}

Abstracting the entity relations in the world as a hypergraph, we formalize relational learning in pre-trained models as recovery of the world relational hypergraph. Under the formulation, we show the relational hypergraph is identifiable provided sufficient data at the population level. We also study the sample efficiency and extend the framework to entity alignment in multimodal learning.


While only extending in multimodal learning in this paper, our framework is a general analysis tool.    Understanding the capabilities and generalization
potential of the PTM is crucial in our field. We would say that PTMs, such as LLMs,  often responding to complex relationships between objects, urgently require new mathematical foundations to have a deeper study.  This paper paves a new way to study PTM from a unique perspective by capturing the overlooked data information using a hypergraph. Our framework can be potentially used under various scenarios and impacts on application fields. For example, for data and computational efficiency,  it is interesting to design more efficient learning algorithms or architectures, such as for multimodal learning. More broadly, for safety, traditional works about adversarial attack and defense theories often focus on several classes that need to be protected. Our framework is not restricted to classification problems and may impose a potential on the entity concept and even human value level. Further, based on the hypergraph, it is promising to understand the reasoning and causality capabilities of PTMs.

\section*{Acknowledgements}
C. Fang and Z. Lin were supported by National Key R\&D Program of China (2022ZD0160300).
Z. Lin was additionally supported by the NSF China (No. 62276004) and Qualcomm.
\section*{Impact Statement}
This paper presents work whose goal is to advance the field of Machine Learning. There are many potential societal consequences of our work, none which we feel must be specifically highlighted here.

\bibliography{ms}
\bibliographystyle{icml2024}

\newpage
\appendix
\onecolumn
\section{Proof}\label{appendix:proof}

\subsection{Proof of Theorem~\ref{thm:identifiability}}
We can consider the combined algorithm $\cA=\cA_{\text{test}}\circ\cA_{\text{pre}}$ directly.
We design an algorithm (Algorithm~\ref{alg:estimate-from-dataset}) that recovers hypergraphs from dataset and show the reconstructed hypergraph converges to $\cH_0$ up to some bijection almost surely by the law of large numbers. 
Denote the hypergraph recovered from $\cD_N$ by $\cH_N$.
Define random variables $X_N = d(\phi_0^{-1}(\cH_N), \cH_0)$ for $N=1,2,\dots$.
It remains to show $X_N \overset{a.s.}{\to} 0$.

For any $\epsilon > 0$, define
\begin{equation}
    E_N := \{\omega\in\Omega:X_N(\omega) > \epsilon\},
\end{equation}
where $\Omega$ is the sample space.

Let 
\begin{equation}
    Y_{e,t} = \begin{cases}
        1 & x_t = \phi_0(e),\\
        0 & \text{otherwise}.
    \end{cases}
\end{equation}
Then we have 
\begin{equation}
    \begin{aligned}
        P(E_N) &= P\left(\sum_{e\in\cE_0}\left|\frac{1}{N}\sum_{t=1}^N Y_{e,t} - w_0(e)\right| > \epsilon\right)\\
               &\leq P\left(\bigcup_{e\in\cE_0}\left|\frac{1}{N}\sum_{t=1}^N Y_{e,t} - w_0(e)\right| >  \frac{\epsilon}{m}\right)\\
               &\overset{(a)}{\leq} \sum_{e\in\cE_0} P\left(\left|\frac{1}{N}\sum_{t=1}^N Y_{e,t} - w_0(e)\right| > \frac{\epsilon}{m}\right)\\
               &\overset{(b)}{\leq} 2 m \exp\left(-\frac{2N\epsilon^2}{m^2}\right),
    \end{aligned}
\end{equation}
where the inequality (a) is due to union bound and the inequality (b) is due to Hoeffding's Inequality.

Notice that 
\begin{equation}
    \sum_{N=1}^{\infty} P(E_N) \leq \frac{2m\exp\left(-2\epsilon^2/m^2\right)}{1-\exp\left(-2\epsilon^2/m^2\right)} < \infty.
\end{equation}
By the first Borel-Cantelli lemma \citep[Chapter 2]{durrett2019probability}, we have
\begin{equation}
    P\left(\limsup_{N\to\infty}E_N\right) = 0.
\end{equation}
Equivalently, we have 
\begin{equation}\label{eq:as-g-epsilon}
    P\left(\lim_{N\to\infty} X_N > \epsilon\right) = 0.
\end{equation}
Since \eqref{eq:as-g-epsilon} holds for any  $\epsilon > 0$, we have $P(\lim_{n\to\infty}X_N=0)=1$, i.e., $X_N \overset{a.s.}{\to} 0$.

\begin{algorithm}[tbhp]
   \caption{Hypergraph Estimation from Datasets}
   \label{alg:estimate-from-dataset}
   \begin{algorithmic}
       \STATE {\bfseries Input:} a dataset $D$, a candidate hyperedge set $\cE_0$, and a masking strategy $\pi$.
       \STATE
       \STATE Initialize $\cE=\{\}$, $\cV={}$, and $\tilde{w}=0$.
       \FOR{$x\in D$}
            \STATE $\cE = \cE \cup \{x\}$
            \STATE $\cV = \cV \cup x$
            \STATE $\tilde{w}(x)=\tilde{w}(x)+1$
       \ENDFOR
       \STATE Compute $W=\sum_{e\in\cE}\tilde{w}(e)$.
       \STATE $w=\tilde{w}/{W}$.
       \STATE
       \STATE Return $\cH=(\cV,\cE,w)$.
    \end{algorithmic}
\end{algorithm}

\subsection{Proof of Theorem~\ref{thm:itlb}}
We prove the information theoretical lower bound by constructing a reduction from finite distribution estimation under $\ell_1$ distance to concept understanding.

For any unknown finite distribution $P=(p_1,\dots,p_m)$ on $\{1,\dots,m\}$, we construct a world model $\cH_0 = (\cV_0, \cE_0, w_0)$ as follows:
\begin{enumerate}
    \item $\cV_0 = \{v_1, \dots, v_{m+1}\}$;
    \item $\cE_0 = \{\{v_1, v_2\}, \dots, \{v_m, v_{m+1}\}\}$;
    \item $w_0(\{v_i, v_{i+1}\})=p_i$.
\end{enumerate}
For a dataset $D'={x_k}_{k=1}^N$ sampled from $P$. convert it to a dataset $D=\{\{v_{x_k}, v_{x_k + 1}\}\}_{k=1}^N$ for hypergraph recovery. For an algorithm $\cA$, apply it to the dataset $D$ and we obtain an estimation $\cH=\cA(D)=(\cV,\cE,w)$ for for the world model $\cH_0$. We then compute an estimation $P'$ for the finite distribution $P$, where $P'=(p_1',\dots, p_m')$ and 
\begin{equation}
    p_i' = w(\{v_i, v_{i+1}\}).
\end{equation}

Denote the minimax risk of estimating a finite distribution on $\{1,\dots,m\}$ with a dataset of size $N$ as $R(m,N)$.
Denote the minimax risk of estimating a hypergraph $\cH_0$ of $m$ hyperedges with a dataset of size $N$ as $R_{\cH}(m,N)$. Then we have
\begin{equation}\label{eq:reduction}
    \begin{aligned}
        R(m,N)
        \leq & \inf_{\cA} \sup_{P\in\cP_m} \sum_{i=1}^m\|p'_i - p_i\| \\
        = & \inf_{\cA} \sup_{\cH_0\in\cH_m} \sum_{e\in\cE_0}\|w(e) - w_0(e)\|\\
        = & \inf_{\cA} \sup_{\cH_0\in\cH_m} d(\cH, \cH_0)\\
        = & R_{\cH}(m,N), 
    \end{aligned}
\end{equation}
where the first inequality is due to the definition of the minimax risk $R(m,N)$.

According to Theorem~2 in \citet{han2015minimax}, we have
\begin{equation}\label{eq:minimax-l1-lb}
    R(m,N) \ge \max_{0<\zeta\leq 1} F(\zeta),
\end{equation}
where 
\begin{equation}
    \begin{aligned}
        F(\zeta) = \frac{1}{8}\sqrt{\frac{e m}{((1+\zeta)N}}\mathds{1}\left(\frac{(1+\zeta)N}{m}>\frac{e}{16}\right) \\
             + \exp\left(-\frac{2(1+\zeta)N}{m}\right)\mathds{1}\left(\frac{(1+\zeta)N}{m}\leq\frac{e}{16}\right) \\
             - \exp\left(-\frac{\zeta^2 N}{24}\right) - 12\exp\left(-\frac{\zeta^2 m}{32\ln^2 m}\right).
    \end{aligned}
\end{equation}

Combining \eqref{eq:reduction} and \eqref{eq:minimax-l1-lb} and letting $\zeta=1$, we have
\begin{equation}
    R_{\cH}(m,N) 
    \ge F(1) 
    \ge \frac{1}{8}\sqrt{\frac{e m}{ 2 N}} - \exp\left(-\frac{N}{24}\right) - 12\exp\left(-\frac{ m}{32\ln^2 m}\right)\ge \frac{1}{16}\sqrt{\frac{m}{N}}.
\end{equation}

\subsection{Proof of Theorem~\ref{thm:ubmm}}

\begin{lemma}\label{lm:rr}
    Suppose that $P_0$ is a finite distribution on $[m_0]=\{1,\dots,m_0\}$ whose range ratio is $\kappa_0$. Then 
    \begin{equation}
        \begin{aligned}
            \min_{i\in [m_0]} P_0(i) &\ge \frac{1}{m_0 \kappa_0}\\
            \max_{i\in [m_0]} P_0(i) &\leq \frac{\kappa_0}{m_0 + \kappa_0 - 1}
        \end{aligned}
    \end{equation}
\end{lemma}

\begin{proof}[Proof of Lemma~\ref{lm:rr}]
    Let $B_1 := \min_{i\in [m_0]} P_0(i)$ and $B_2 := \max_{i\in [m_0]} P_0(i)$. By the definitions, we have
    \begin{equation*}
        \begin{gathered}
            B_1 + (m - 1) B_2 \ge 1\\
            B_2 + (m - 1) B_1 \leq 1.
        \end{gathered}
    \end{equation*}
    By the definition of range ratio, i.e. $\kappa_0\frac{B_2}{B_1}$, we further have
    \begin{equation*}
        \begin{gathered}
            B_1 + (m_0 - 1) \kappa_0 B_1 \ge 1\\
            B_2 + \frac{m_0 - 1}{\kappa_0} B_2 \leq 1.
        \end{gathered}
    \end{equation*}
    This implies
    \begin{equation*}
        \begin{gathered}
            B_1 \ge \frac{1}{m_0 \kappa_0 + 1 - \kappa_0} \ge \frac{1}{m_0 \kappa_0}\\
            B_2 \leq \frac{\kappa_0}{m_0 + \kappa_0 - 1}.
        \end{gathered}
    \end{equation*}
\end{proof}

\begin{lemma}\label{lm:plug-in-estimation}
 Suppose that $\{X_t\}$ is a sequence of random variables sampled i.i.d. from a categorical distribution $\op{Cat}(K, \bm p)$ where $\bm p = (p_1,\dots,p_K)$. Then we have 
 \begin{equation}
     P\left(\sum_{k=1}^K \left|\frac{1}{T}\sum_{t=1}^T \mathds{1}\left(X_t = k\right)-p_k\right| \leq \epsilon\right) \ge 1 - \delta
 \end{equation}
 if 
 \begin{equation}
     T \ge \frac{2 K}{\epsilon^2}\log\frac{2 K}{\delta}.
 \end{equation}
\end{lemma}

\begin{proof}[Proof of Lemma~\ref{lm:plug-in-estimation}]
    Let $S:=\sum_{k=1}^K \sqrt{p_k (1 - p_k)}$ and $\epsilon_k := \frac{\sqrt{p_k (1 - p_k)}}{S} \epsilon$ for $k=1,\dots,K$. Then we have 
    \begin{equation}\label{eq:plug-in-estimation-main}
        \begin{aligned}
            &P\left(\sum_{k=1}^K \left|\frac{1}{T}\sum_{t=1}^T \mathds{1}\left(X_t = k\right)-p_k\right| \ge \epsilon\right) \\
            \overset{(a)}{\leq}
            & \sum_{k=1}^K P\left( \left|\frac{1}{T}\sum_{t=1}^T \mathds{1}\left(X_t = k\right)-p_k\right| \ge \epsilon_k\right) \\
            \overset{(b)}{\leq}
            & \sum_{k=1}^K 2 \exp\left(-\frac{T \epsilon_k^2}{2 p_k (1-p_k)}\right)\\
            \leq
            &2 K \exp\left(-\frac{T \epsilon^2}{ 2 S^2}\right),
        \end{aligned}
    \end{equation}
    where the inequality (a) is due to union bound and the inequality (b) is due to Chernoff bound.

    According to the concavity of the function $f(x)=\sqrt{x (1 - x)}$, we have 
    \begin{equation}\label{eq:ub-S}
        S = K\cdot \frac{1}{K} \sum_{k=1}^K f(p_k)
         \leq K f\left(\frac{1}{K} \sum_{k=1}^K p_k\right)
         = K f\left(\frac{1}{K}\right)
         = \sqrt{K-1} < \sqrt{K}.
    \end{equation}

    Combining \eqref{eq:plug-in-estimation-main} and \eqref{eq:ub-S}, we obtain the desired result.
\end{proof}

We provide a constructive proof of Theorem~\ref{thm:ubmm} by designing an algorithm that recover hypergraphs from MM pre-trained models. The algorithm includes two Phases: underlying hypergraph estimation and weight estimation. In Phase 1, we estimate the underlying hypergraph by evaluating the probability of the MM pre-trained model output and selecting all hyperedges of positive probabilities. In Phase 2, we evaluate a sequence of relative weights between the hypergraphs.
We estimate the weight function by those relative weights and a normalization. The algorithm is presented in Algorithm~\ref{alg:estimate-from-model}. Specially, we implement the weight estimation algorithm in a breadth-first style (Algorithm~\ref{alg:bf-weight-estimation}). We utilize the data structure queue to implement the algorithm. A queue $Q$ supports two operations: $Q.\op{push\_back}(x)$ that pushes the element $x$ to the back of the queue $Q$ and $Q.\op{pop\_front}(x)$ that removes and returns the front of the queue $Q$.

\begin{algorithm}[tbhp]
   \caption{Hypergraph Estimation from MM Pre-Trained Models}
   \label{alg:estimate-from-model}
   \begin{algorithmic}
       \STATE {\bfseries Input:} a MM pre-trained model $\cM$, a candidate hyperedge set $\cE_0$, and a masking strategy $\pi$.
       \STATE
       \STATE \textit{// Phase 1: underlying hypergraph estimation}
       \STATE Initialize $\cE=\{\}$.
       \FOR{$e\in\cE_0$}
       \STATE Apply $\pi$ to $e$ and get a masked hyperedge $e^{-}$.
       \IF{$M(e\mid e^{-}) > 0$}
       \STATE $\cE = \cE\cup \{e\}$.
       \ENDIF
       \ENDFOR
       \STATE $\cV=\cup_{e\in\cE} e$.
       \STATE 
       \STATE \textit{// Phase 2: weight estimation}
       \STATE Initialize $\tilde{w}(e)=0$ for all $e\in\cE$.
       \STATE Select $e_0$ from $\cE$ and let $\tilde{w}(e_0)=1$.
       \STATE $\tilde{w}=\textsc{BFWeightEstimation}(e_0,\cE,\cM,\pi,\tilde{w})$ (Algorithm~\ref{alg:bf-weight-estimation}).
       \STATE Compute $W=\sum_{e\in\cE}\tilde{w}(e)$.
       \STATE $w=\tilde{w}/{W}$.
       \STATE
       \STATE Return $\cH=(\cV,\cE,w)$.
    \end{algorithmic}
\end{algorithm}


\begin{algorithm}[tbhp]
   \caption{$\textsc{BFWeightEstimation}(e_{\text{init}},\cE,\cM,\pi,\tilde{w})$}
   \label{alg:bf-weight-estimation}
   \begin{algorithmic}
        \STATE {\bfseries Input:} a selected hyperedge $e_{\text{init}}$, a hyperedge set $\cE$, a MM pre-trained model $\cM$, a masking strategy $\pi$, and a weight function $\tilde{w}$.
        \STATE
        \STATE Initialize an empty queue $Q$.
        \STATE $Q.\op{push\_back}(e_{\text{init}})$.
        \WHILE{$Q$ is not empty}
            \STATE $e = Q.\op{pop\_front}()$.
             \FOR{$e'\in\cE$ such that $e\overset{\pi}{\leftrightarrow} e'$}
                \IF{$\tilde{w}(e') > 0$}
                    \STATE Continue.
                \ENDIF
             
                \STATE $\tilde{w}(e')=\frac{\pi(e^{-}\mid e)\cM(e'\mid e^{-})}{\pi(e^{-}\mid e')\cM(e\mid e^{-})}\tilde{w}(e)$.
                \STATE $Q.\op{push\_back}(e')$.
            \ENDFOR
        \ENDWHILE
        \STATE
        \STATE Return $\tilde{w}$.
   \end{algorithmic}
\end{algorithm}

We first show that the underlying hypergraph can be recovered with high probability in Phase 1. We denote $\min_{e\in\cE_0} w_0(e)$ and $\max_{e\in\cE_0} w_0(e)$ by $c_w$ and $C_w$, respectively.
By the definition of the model $\cM$, it suffices to show that each hyperedge $e$ and possible masked hypergraphs $e^{-}$ (i.e., $\pi(e^{-}\mid e)> 0$) are covered by the training dataset $\cD$.
According to the data generation process, each sample in the dataset $\cD$ corresponds to a pair of $(e,e^{-})$ sampled from the distribution $P((e,e^{-}))=P_w(e)\pi(e^{-}\mid e)$. With slight abuse of notation, we write $(e,e^{-})\in\cD$ if $\cD$ contains the corresponding sample of the pair $(e,e^{-})$. Denote the support set of $P((e,e^{-}))$ by $S_\pi$. By Assumptions~\ref{as:bwf} and \ref{as:bms}, we have $|S_{\pi}| \leq m C_\pi$ and $P(e,e^{-}) \ge c_w c_{\pi}$ for all $(e,e^-)\in S_{\pi}$. Denote the event that the underlying hypergraph $\cH_1$ recovered in Phase 1 satisfies $\cH\sim\cH_0$ by $E_1$. Then we can obtain
\begin{equation}\label{eq:phase1}
    \begin{aligned}
        P(E_1^c)=P\left(\exists (e,e^{-})\in S_{\pi}, (e,e^{-})\not\in\cD\right) 
        &\leq \sum_{(e,e^{-})\in S_{\pi}} P\left((e,e^{-})\not\in\cD\right)\\
        &\leq |S_{\pi}|\min_{(e,e^{-})\in S_{\pi}} P\left((e,e^{-})\not\in\cD\right)\\
        &\leq m C_\pi (1 - c_w c_\pi)^{N}.
    \end{aligned}
\end{equation}

We then consider the weight estimation process in Phase 2,
supposing that the underlying hypergraph $\cH_1$ recovered in Phase 1 satisfies $\cH\sim\cH_0$ and the isomorphism mapping from $\cH$ to $\cH_0$ as $\phi$.
Notice that if we replace $\cM$ with $\cM_0$ in Algorithm~\ref{alg:bf-weight-estimation}, the estimated weight function $w$ satisfies 
$w(e) = w_0(\phi(e))$ for all $e\in\cE$. 
Since we train by MM with cross-entropy loss, we have
\begin{equation}\label{eq:mm-ce}
    \cM(e\mid e^{-}) = \frac{\sum_{t=1}^N \sum_{k=1}^K \mathds{1}(e_{tk} = e, e_{tk}^- = e^-)}{\sum_{e\in\cE}\sum_{t=1}^N \sum_{k=1}^K \mathds{1}(e_{tk} = e, e_{tk}^- = e^-)}.
\end{equation}

We first consider only randomness over sampling masked hyperedges for given hyperedges. Denote the number of $e$ in $\{e_t\}_{t=1}^N$ by $f_N(e)$.
For any $e\in\cE$, $e^-\sim\pi(\cdot\mid e)$ and $\epsilon_1 > 0$, we have
\begin{equation}
    \begin{aligned}
        &P\left(\left|\frac{1}{N K}\sum_{t=1}^N \sum_{k=1}^K \mathds{1}(e_{tk} = e, e_{tk}^- = e^-) - \frac{f_N(e)}{N}\pi(e^-\mid e)\right|\ge \frac{f_N(e)}{N}\pi(e^-\mid e)\epsilon_1\right)\\
        =&
        P\left(\left|\frac{1}{K}\sum_{k=1}^K\left[\frac{1}{N}\sum_{t=1}^N \mathds{1}(e_{tk} = e, e_{tk}^- = e^-)\right] - \frac{f_N(e)}{N}\pi(e^-\mid e)\right|\ge \frac{f_N(e)}{N}\pi(e^-\mid e)\epsilon_1\right)\\
        \overset{(a)}{\leq}&
        2\exp\left[-2K\left(\frac{f_N(e)}{N}\pi(e^-\mid e)\epsilon_1\right)^2\right],
    \end{aligned}
\end{equation}
where the inequality (a) is due to Hoeffding's inequality. By union bound, we have 
\begin{equation}
    \begin{aligned}
        &P\left(\exists (e,e^-), \left|\frac{1}{N K}\sum_{t=1}^N \sum_{k=1}^K \mathds{1}(e_{tk} = e, e_{tk}^- = e^-) - \frac{f_N(e)}{N}\pi(e^-\mid e)\right|\ge \frac{f_N(e)}{N}\pi(e^-\mid e)\epsilon_1\right)\\
        \leq
        &\sum_{(e,e^-)} 2\exp\left[-2K\left(\frac{f_N(e)}{N}\pi(e^-\mid e)\epsilon_1\right)^2\right].
    \end{aligned}
\end{equation}

When $\left|\frac{1}{N K}\sum_{t=1}^N \sum_{k=1}^K \mathds{1}(e_{tk} = e, e_{tk}^- = e^-) - \frac{f_N(e)}{N}\pi(e^-\mid e)\right|\ge \frac{f_N(e)}{N}\pi(e^-\mid e)\epsilon_1$ holds for all pairs of $(e,e^-)$, for any $e,e'$ such that $e\leftrightarrow e'$ with $e^-$ being the common masked hyperedge, we have
\begin{equation}
    \begin{aligned}
        \left|\frac{\tilde{w}(e)}{\tilde{w}(e')} - \frac{f_N(e)}{f_N(e')}\right|
        =
        &
        \left|\frac{\cM(e\mid e^-) \pi(e^-\mid e')}{\cM(e'\mid e^-) \pi(e^-\mid e)} - \frac{f_N(e)}{f_N(e')}\right|\\
        \leq &
        \left(\frac{1+\epsilon_1}{1-\epsilon_1} - 1\right) \frac{f_N(e)}{f_N(e')}\\
        = &
        \epsilon_2 \frac{f_N(e)}{f_N(e')},
    \end{aligned}
\end{equation}
where $\epsilon_2 := \frac{1+\epsilon_1}{1-\epsilon_1} - 1 = \frac{2\epsilon_1}{1 - \epsilon_1}$. This implies
\begin{equation}
    (1 - \epsilon_2) \frac{f_N(e)}{f_N(e')} \leq \frac{\tilde{w}(e)}{\tilde{w}(e')} 
    \leq (1 + \epsilon_2) \frac{f_N(e)}{f_N(e')}.
\end{equation}
By Assumption~\ref{as:bpath}, for any $e\in\cE$, there exists a path $e_{\text{init}}=e^{(1)}\leftrightarrow\cdots\leftrightarrow e^{(\ell)}=e$, $\ell\leq L$ and we have
\begin{equation}
    (1 - \epsilon_2)^L \frac{f_N(e)}{f_N(e_{\text{init}})} \leq \frac{\tilde{w}(e)}{\tilde{w}(e_{\text{init}})}=\tilde{w}(e)
    \leq (1 + \epsilon_2)^L \frac{f_N(e)}{f_N(e_{\text{init}})}.
\end{equation}

Notice that
\begin{equation}\label{eq:range-w}
    \begin{aligned}
    w(e) =& \frac{\tilde{w}(e)}{\sum_{e'\in\cE}\tilde{w}(e')}\\
         =& \frac{\tilde{w}(e)/\tilde{w}(e_{\text{init}})}{\sum_{e'\in\cE}\tilde{w}(e') / \tilde{w}(e_{\text{init}})}\\
         \in& \left[\frac{(1-\epsilon_2)^L}{(1+\epsilon_2)^L} \cdot
         \frac{f_N(e)}{N}, \frac{(1+\epsilon_2)^L}{(1-\epsilon_2)^L} \cdot
         \frac{f_N(e)}{N}\right]
    \end{aligned}
\end{equation}

We then obtain
\begin{equation}\label{eq:phase2-analysis}
    \begin{aligned}
        \|w - w_0\circ\phi\|_1 
        =& 
        \sum_{e\in\cE} |w(e)-w_0(\phi(e))| \\
        =&
        \sum_{e\in\cE} \left|w(e)-\frac{f_N(e)}{N} + \frac{f_N(e)}{N} - w_0(\phi(e))\right|\\
        \leq &
         \sum_{e\in\cE} \left|w(e)-\frac{f_N(e)}{N}\right| + \sum_{e\in\cE}\left|\frac{f_N(e)}{N} - w_0(\phi(e))\right|\\
         \overset{(a)}{\leq} &
         \left[\frac{(1+\epsilon_2)^L}{(1-\epsilon_2)^L} - 1\right]\sum_{e\in\cE} \frac{f_N(e)}{N} 
         + \sum_{e\in\cE}\left|\frac{f_N(e)}{N} - w_0(\phi(e))\right|\\
         \overset{(b)}{=} &
         \left[\frac{(1+\epsilon_2)^L}{(1-\epsilon_2)^L} - 1\right]
         + \sum_{e\in\cE}\left|\frac{f_N(e)}{N} - w_0(\phi(e))\right|,
    \end{aligned}
\end{equation}
where the inequality (a) is due to \eqref{eq:range-w}
and the equality (b) is due to $\sum_{e\in\cE} f_N(e)=N$.
Note that $\frac{(1+\epsilon_2)^L}{(1-\epsilon_2)^L} - 1 \leq \frac{\epsilon}{2}$ if $\epsilon_1\leq\frac{\epsilon}{64 L}$ for $\epsilon$ sufficiently small. 
By \eqref{eq:phase2-analysis} and Lemma~\ref{lm:plug-in-estimation}, with $\epsilon_1 = \frac{\epsilon}{64 L}$, we have
\begin{equation}\label{eq:phase2}
    \begin{aligned}
        &P\left(E_1 \land \|w - w_0\circ\phi\|_1 \ge \epsilon\right) \\
        \leq
        & 
        P\left(\sum_{e\in\cE}\left|\frac{f_N(e)}{N} - w_0(\phi(e))\right|\ge \frac{\epsilon}{2}\right) 
        + 
        P\left(\sum_{e\in\cE}\left|\frac{f_N(e)}{N} - w_0(\phi(e))\right|\leq \frac{\epsilon}{2}\right. \\
        &\left.\land \exists (e,e^-), \left|\frac{1}{N K}\sum_{t=1}^N \sum_{k=1}^K \mathds{1}(e_{tk} = e, e_{tk}^- = e^-) - \frac{f_N(e)}{N}\pi(e^-\mid e)\right|\ge \frac{f_N(e)}{N}\pi(e^-\mid e)\epsilon_1\right) \\
        \leq &
        \sum_{(e,e^-)} 2\exp\left[-2K\left(\frac{f_N(e)}{N}\pi(e^-\mid e)\epsilon_1\right)^2\right]
        +
        2 m \exp\left(-\frac{N\epsilon^2}{8 m}\right)\\
        \overset{(a)}{\leq}
        &
        \sum_{(e,e^-)} 2\exp\left[-2K\left(\frac{c_w}{2}\pi(e^-\mid e)\epsilon_1\right)^2\right]
        +
        2 m \exp\left(-\frac{N\epsilon^2}{8 m}\right)\\
        \leq
        &
         2 m C_\pi \exp\left[-2K\left(\frac{c_w c_\pi}{128L}\epsilon\right)^2\right]
        +
        2 m \exp\left(-\frac{N\epsilon^2}{8 m}\right),
    \end{aligned}
\end{equation}
where the inequality (a) is due to $\frac{f_N(e)}{N}\ge c_w - \frac{\epsilon}{2} \ge \frac{c_w}{2}$ when $\sum_{e\in\cE}\left|\frac{f_N(e)}{N}-w_0(\phi(e))\right|\leq\frac{\epsilon}{2}$ holds and $\epsilon$ is sufficiently small.

Combining \eqref{eq:phase1} and \eqref{eq:phase2}, we have
\begin{equation}
    \begin{aligned}
        & P\left(\|w - w_0\circ\phi\|_1 \leq \epsilon\right) \\
        \ge
        & 1 - P(E_1^c) - P(E_1 \land \|w - w_0\circ\phi\|_1 \ge \epsilon) \\
        \ge
        & 1 - mC_\pi (1 - c_w c_\pi)^N 
        - 2 m C_\pi \exp\left[-2K\left(\frac{c_w c_\pi}{128L}\epsilon\right)^2\right]
        -
        2 m \exp\left(-\frac{N\epsilon^2}{8 m}\right)\\
        \ge& 1 - \delta,
    \end{aligned}
\end{equation}
if
\begin{equation}
    \begin{gathered}
        mC_\pi (1 - c_w c_\pi)^N \leq \frac{\delta}{3}, \\
        2 m C_\pi \exp\left[-2K\left(\frac{c_w c_\pi}{128L}\epsilon\right)^2\right] \leq \frac{\delta}{3}, \\
        2 m \exp\left(-\frac{N\epsilon^2}{8 m}\right) \leq \frac{\delta}{3}.
    \end{gathered}
\end{equation}
After simplification, we have
\begin{equation}
    \begin{gathered}
        K  \ge \frac{2^{14} m^2 \kappa^2 L^2}{ c_\pi^2 \epsilon^2}\log \frac{6 m C_\pi}{\delta},\\
        N \ge \max\left\{\frac{2 m \kappa}{c_\pi}\log \frac{3 m C_\pi}{\delta}, \frac{8 m}{\epsilon^2}\log\frac{6 m}{\delta}\right\}.
    \end{gathered}
\end{equation}

\subsection{Proof of \cref{prop:multimodal-learning}}

\cref{prop:multimodal-learning} is directly implication of \cref{thm:ubmm} in the multimodal model with the prior entity alignment $\phi^*$. More concretely, we can generate a dataset $D'=\phi^*(D_1)\cup D_2$ with $N'=N_1+N_2,K'=K_1+K_2$ by the entity alignment $\phi^*$. Applying \cref{thm:ubmm} to the dataset $D'$, we obtain \cref{prop:multimodal-learning}.

\newpage
\section{Entity Alignment}\label{appendix:entity-alignment}

While we show that entity alignment is feasible without labeled pairs in theory, labeled pairs are important in practice. A possible reason is that solving the entity alignment problem is computational challenging, no known polynomial algorithms addressing the problem. The role of the labeled pairs might be reducing the inherent complexity required to solve the computational problem. Here are two examples of how the labeled pairs can help to solve the alignment problem more efficiently.

\begin{example}\label{ex:all}
    When all $m$ labeled pairs for the hyperedges are available, we can efficiently determine the alignment mapping between entities by leveraging hyperedges as identifiers. 
    More concretely, we assign a unique number as the identifier to each hyperedge. Subsequently, each node is labeled with a tuple containing the identifiers of the hyperedges it belongs to, 
    arranged in descending order. The nodes within each hypergraph are then organized into sequences based on their lexicographic order. Correspondence between entities is established through the alignment of nodes at identical positions within these sequences.
    The entire alignment process is of computational complexity $\Tilde{O}(m n)$.
\end{example}

\cref{ex:all} shows that we can align entities efficiently given all $m$ labeled pairs for the hyperedges.
This also means that as long as we can find the graph matching between the line graphs of the hypergraphs, we can also
align the hypergraphs with only polynomial extra computational overhead. Therefore, we can focus on the graph matching problem of 
the line graphs of the hypergraphs.

WL test serves as a potent heuristic for graph matching, demonstrating efficacy across a wide range of graphs.
Nonetheless, certain graphs challenge the capabilities of low-dimensional WL tests, leading to their failure \citep{cai1992optimal}. Although higher-dimensional WL tests may achieve accurate graph matching, they impose significantly greater computational demands. Labeled pairs could help to overcome this dilemma.

\begin{example}\label{ex:wl}
    Frucht graph (\cref{fig:frucht}) is a regular graph without non-trivial automorphism \citep{frucht1939herstellung}. 
    1-WL does not work for Frucht graph because of its regularity. 
    While higher-dimensional WL tests are applicable, they are significantly less efficient.
    However, if a labeled pair is identified, one can exclude the nodes in the label pair from both graphs 
    and apply the 1-WL test to the resulting subgraphs, leading to efficient graph matching.
\end{example}

\begin{figure*}[h]
    \centering
    \includegraphics[width=0.3\linewidth]{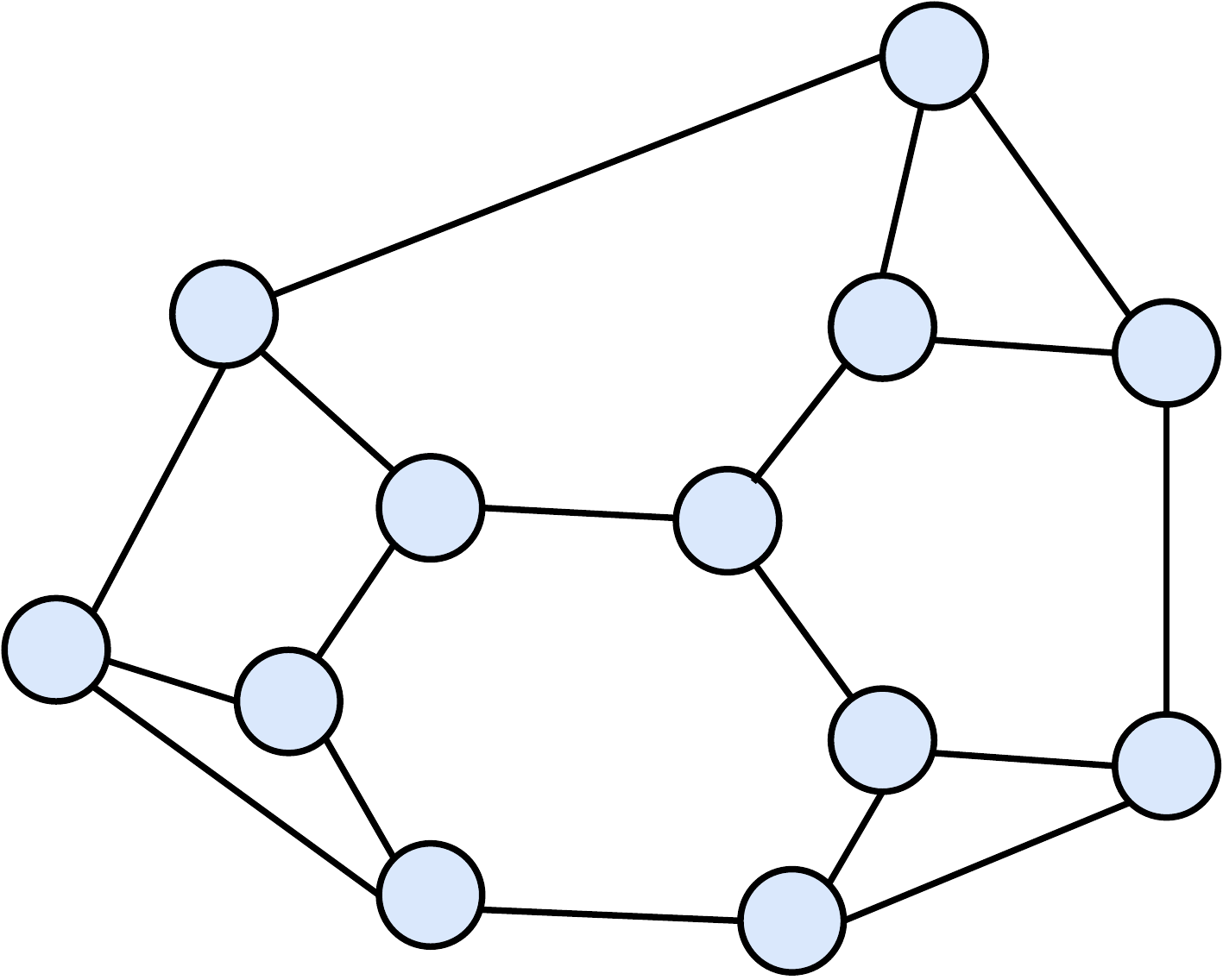}
    \caption{Frucht graph.}
    \label{fig:frucht}
\end{figure*}

\newpage
\section{Experiments}\label{appendix:experiments}

\subsection{Synthetic Relational Learning}\label{subappdendix:srl}

\subsection{Data}

\subsubsection{Graph Structures}
When the number of nodes is $n$, the different graph structures (\cref{fig:graph-structures}) are
\begin{itemize}
    \item STAR: 
    \begin{itemize}
        \item[-] $\cV=\{0,1,\dots,n-1\}$;
        \item[-] $\cE=\{\{0,i\}\mid i=1,\dots,n-1\}$;
    \end{itemize}
    \item X: 
    \begin{itemize}
        \item[-] $\cV=\{0,1,\dots,n-1\}$;
        \item[-] $\cE=\{\{0,k\}\mid k=1,2,3,4\}\cup\{\{4i+k,4i+k+4\}\mid 4i+k+4 \leq n-1\}$;
    \end{itemize}
    \item CHAIN: 
    \begin{itemize}
        \item[-] $\cV=\{0,1,\dots,n-1\}$;
        \item[-] $\cE=\{\{i,i+1\}\mid i=0,\dots,n-2\}$.
    \end{itemize}
\end{itemize}

\begin{figure*}[th]
    \centering
    \subfigure[STAR.]{
        \includegraphics[width=0.31\linewidth]{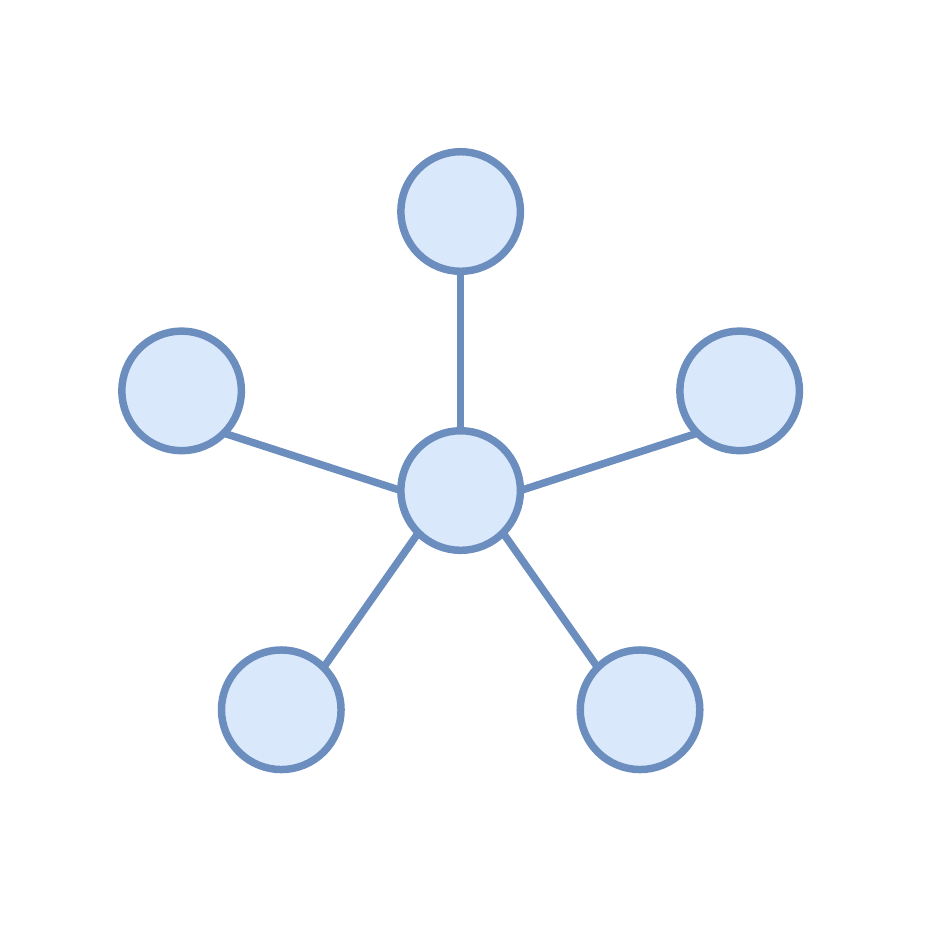}
    }
    \subfigure[X.]{
        \includegraphics[width=0.31\linewidth]{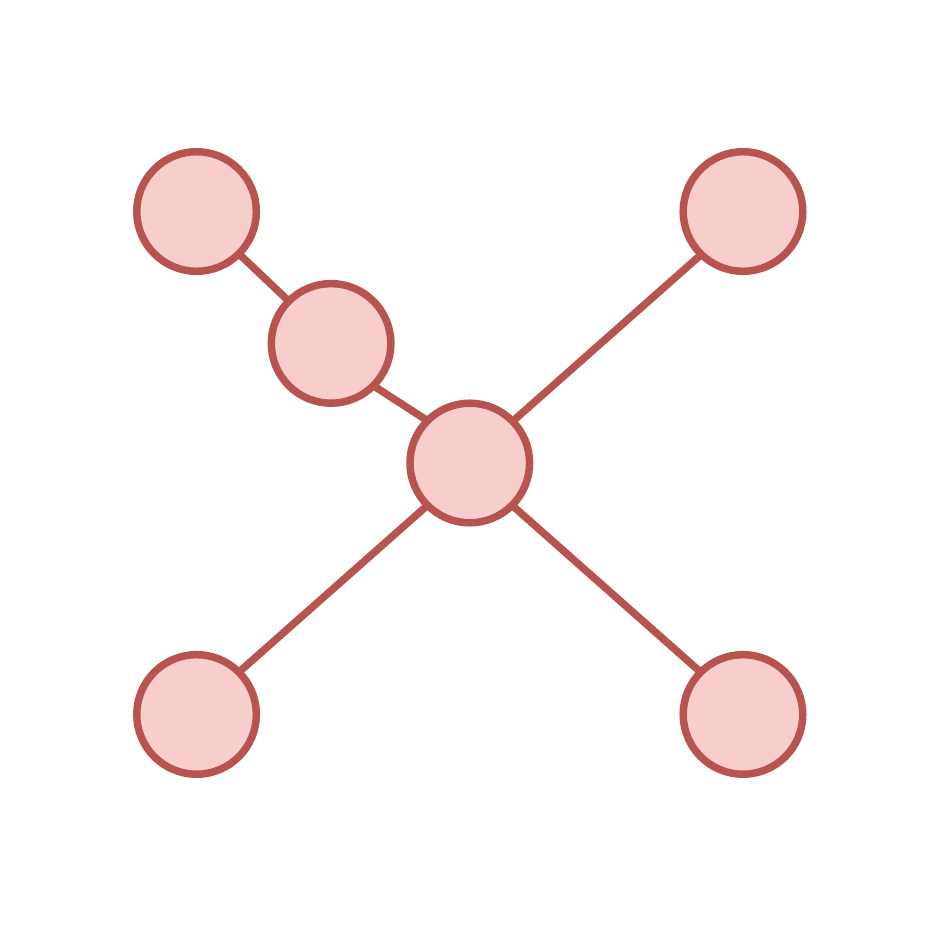}
    }
    \subfigure[CHAIN]{
        \includegraphics[width=0.31\linewidth]{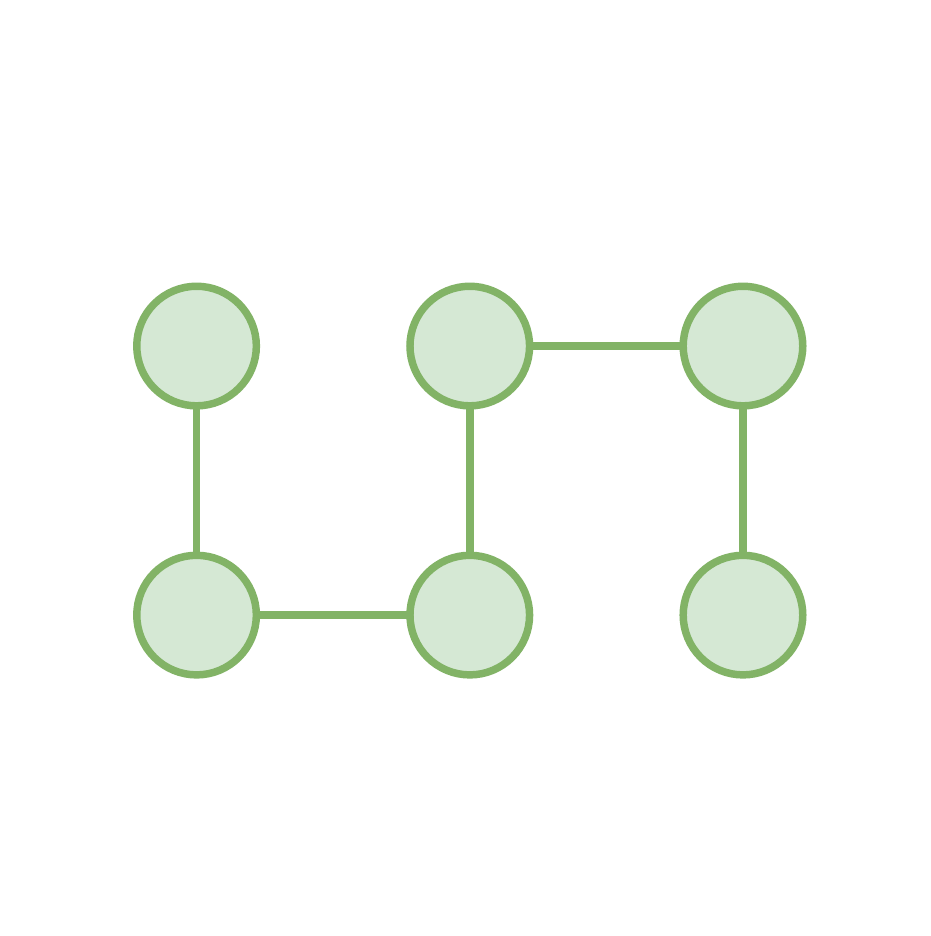}
    }
    \caption{Different graph structures ($n=6$).}
    \label{fig:graph-structures}
\end{figure*}

\subsubsection{Data Generation}
Each node of the graph is attached with a token, starting from ``a" and following the order of tokens of BERT's tokenizer. Each edge is assigned a weight, sampled from $\{w_{\min}, w_{\max}\}$. Specifically, we use $w_{\min}=1.0, w_{\max}=1.0$ for $\kappa=1.0$, $w_{\min}=1.0, w_{\max}=10.0$ for $\kappa=10.0$, and $w_{\min}=1.0, w_{\max}=100.0$ for $\kappa=100.0$ in our experiments. Then the weights of the graph are normalized. When generating data, we first sample an edge from the graph, with probability proportional to the the weights. We then concatenate the tokens of the edges with a random order. Tokens are separated by spaces to avoid that they are combined by the tokenizer. For each graph, we generate $100000$ samples for each graph, with $80000$ samples for training, $10000$ samples for validation, and $10000$ samples for testing.

\subsubsection{Model}
We choose BERT as our underlying PTM.
We use the implementation of HuggingFace \citep{wolf2020transformers} with the default tokenizer and the default configuration of BERT.

\subsubsection{Pre-Training}
We pre-train our model by MLM from scratch. For the masking strategy, we mask one of the tokens in a sample uniformly at random.
We train the model by AdamW, with the initial learning rate $2\times 10^{-5}$, weight decay $0.01$, the cosine scheduler. The other hyperparameters of AdamW are the same as the default of HuggingFace TrainerArguments. We pre-train the model for $100$ epochs. Per-device training batch size is $256$. The experiments are run on a server with Ubuntu. All the models are trained on two NVIDIA GeForce RTX 3090 GPUs.

\subsection{Real-World Relation Evaluation}\label{subappendix:rwre}

To extract a subgraph from ConceptNet, we first choose a source entity, query for the $k$ most related entities, and then repeat such a process for the returned entities. We adopt a breadth-first-search-like generation process to choose a subset of entities and construct a subgraph by considering $k$ most related entities within these chosen entities and the generation process is limited to some depth $d$ (the source entity are of depth $0$). 

In the real-world relation evaluation, we assess the LLMs' understanding of entity relationships by querying the $k$ most related entities within a specified set. We employ the prompt: 
``Consider the following concepts: [ENTITIES].
Suppose that these concepts are nodes of an undirected graph.
For each concept, consider [$k$] most related concepts.
According to the relations between these concepts, which edges should be included? Please answer with an edgelist.", where ``[ENTITIES]" and ``[$k$]" are placeholders for the actual entity set and the number of top-related entities, respectively. The LLMs will generate responses comprising edgelists, potentially accompanied by additional text, which are then utilized to construct relational graphs. See \cref{fig:rwre-example} for an illustration.
These graphs are compared with the corresponding subgraphs extracted from ConceptNet. 

\begin{figure*}[th]
    \centering
    \includegraphics[width=0.8\linewidth]{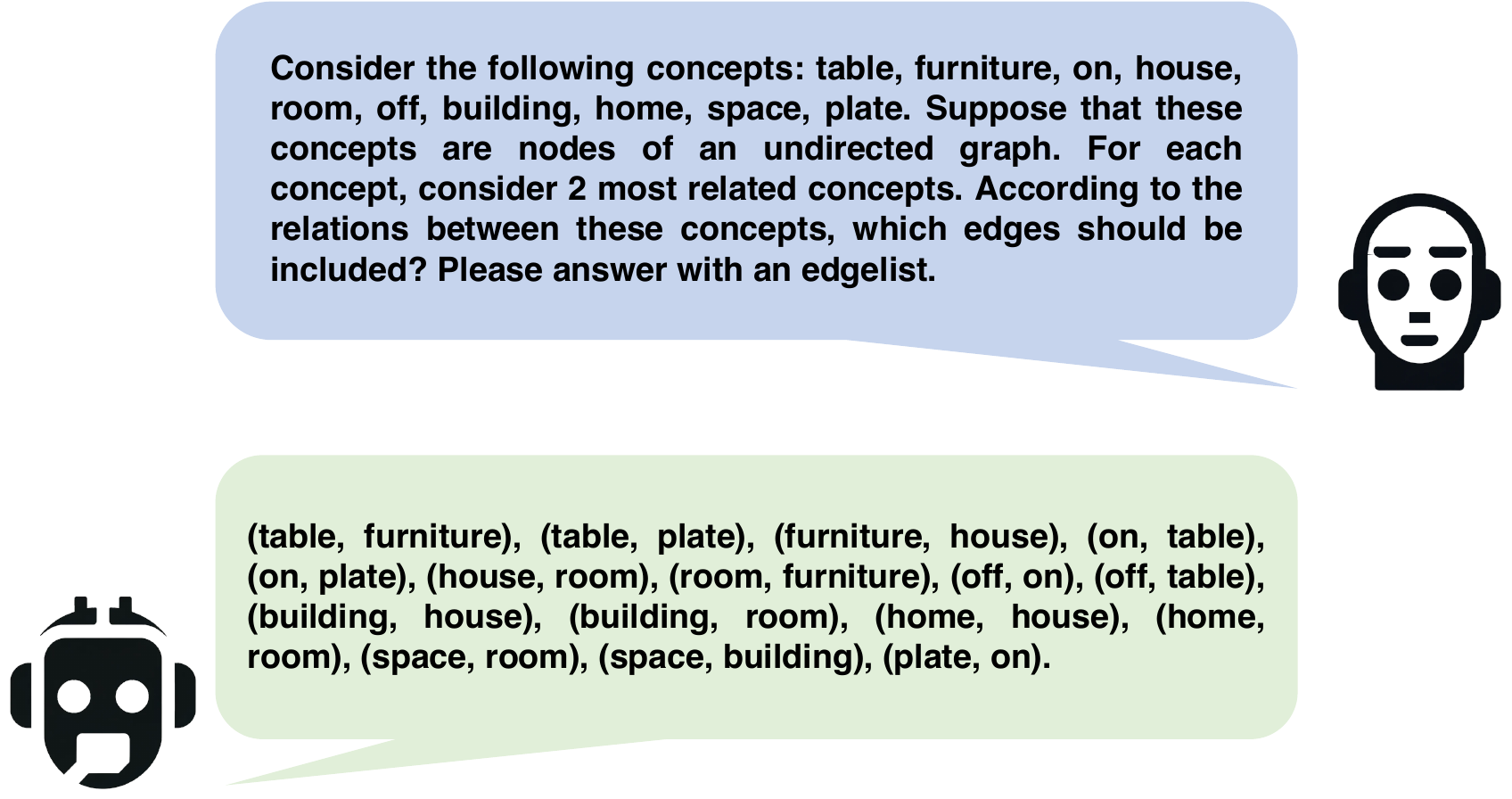}
    \caption{An example of real-world relation evaluation using GPT-4. Here, ``table" is the source entity, with $k=2$ indicating the two most closely related entities considered for generating the relational graph.}
    \label{fig:rwre-example}
\end{figure*}

Figures~\ref{fig:appx-cake}-\ref{fig:appx-zoo} are the evaluation results for all the source entities listed in \cref{tab:exp-res-k2d3}.

The correspondences between the entities and the letters used in the above figures are summarized in Tables~\ref{tab:entity-letter-1} and~\ref{tab:entity-letter-2}.

\begin{table}[th]
\caption{The correspondences between the entities and the letters for ConceptNet (Part 1).}
\label{tab:entity-letter-1}
\begin{center}
\begin{sc}
\begin{tabular}{lcccccc}
\toprule
 & A & B & C & D & E & F  \\
\midrule
cake & cake & birthday & dessert & celebration & lizard & party \\
dog & dog & bark & house & tree & building & home \\
fly & fly & insect & bug & flea & meadow & wiretap \\
human & human & school & home & learn & place & house \\
jacket & jacket & coat & shell & closet & material & husk \\
orange & orange & fruit & peel & eat & you & skin \\
paper & paper & write & sheet & pen & bed & closet \\
sea & sea & ocean & water & sail & lake & drink \\
table & table & furniture & on & house & room & off \\
zoo & zoo & animal & elephant & squirrel & circus & trunk \\
\bottomrule
\end{tabular}
\end{sc}
\end{center}
\end{table}

\begin{table}[th]
\caption{The correspondences between the entities and the letters for ConceptNet (Part 2).}
\label{tab:entity-letter-2}
\begin{center}
\begin{sc}
\begin{tabular}{lcccccc}
\toprule
 & G & H & I & J & K & L \\
\midrule
cake & garden & rock & - & - & - & - \\
dog & plant & grow & town & bank & place & - \\
fly & dog & wood & hayfield & investigation & tap & - \\
human & study & knowledge & location & bed & building & - \\
jacket & bedroom & clothes & wood & wool & chaff & - \\
orange & food & hunger & me & body & mole & - \\
paper & office & pocket & sleep & furniture & bedroom & clothes \\
sea & boat & wind & pond & liquid & beverage & - \\
table & building & home & space & plate & - & - \\
zoo & rodent & balloon & attic & car & - & - \\
\bottomrule
\end{tabular}
\end{sc}
\end{center}
\end{table}

\begin{figure*}[th]
    \centering
    \subfigure[Ground Truth]{
        \includegraphics[width=0.23\linewidth]{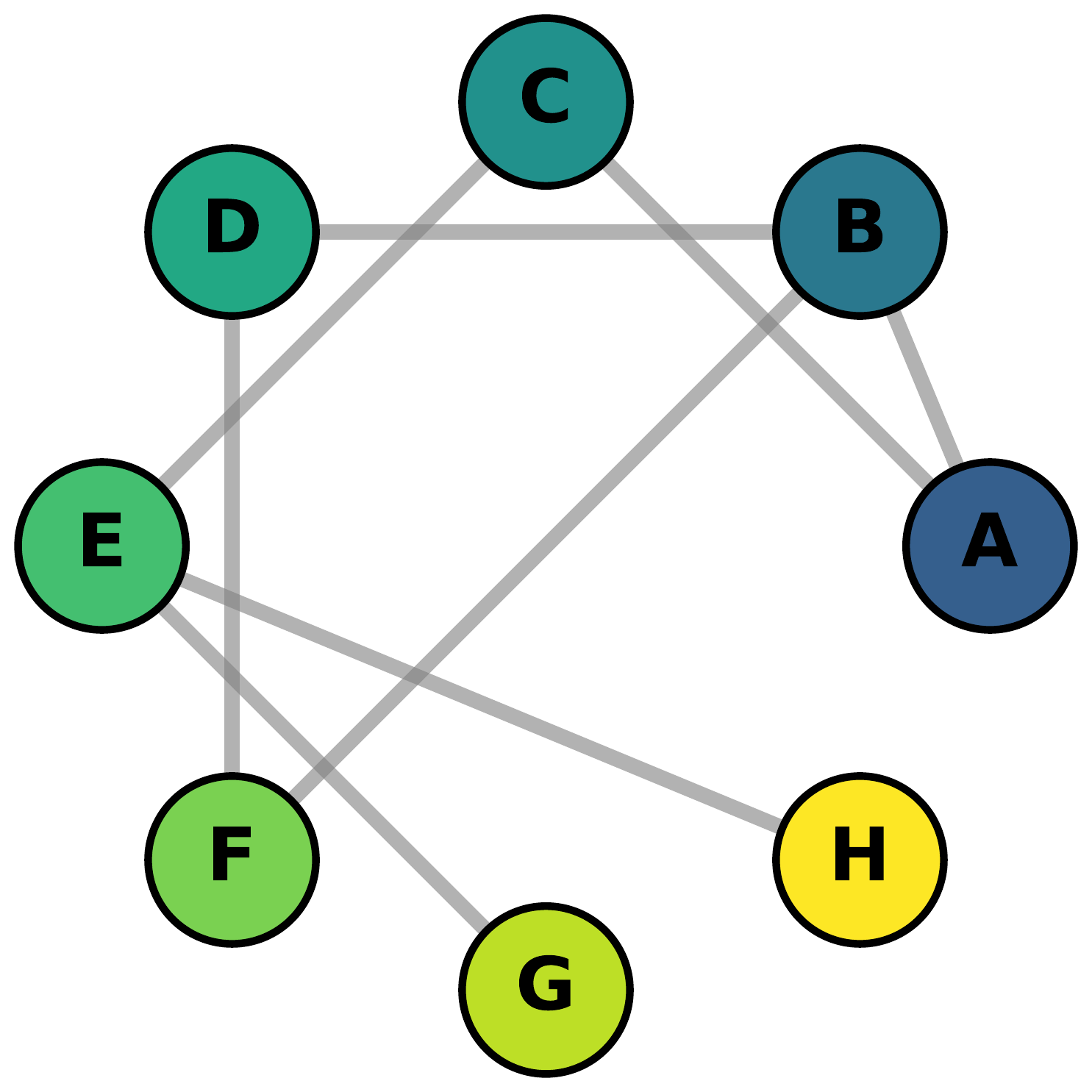}
    }
    \subfigure[LLAMA-2-70B]{
        \includegraphics[width=0.23\linewidth]{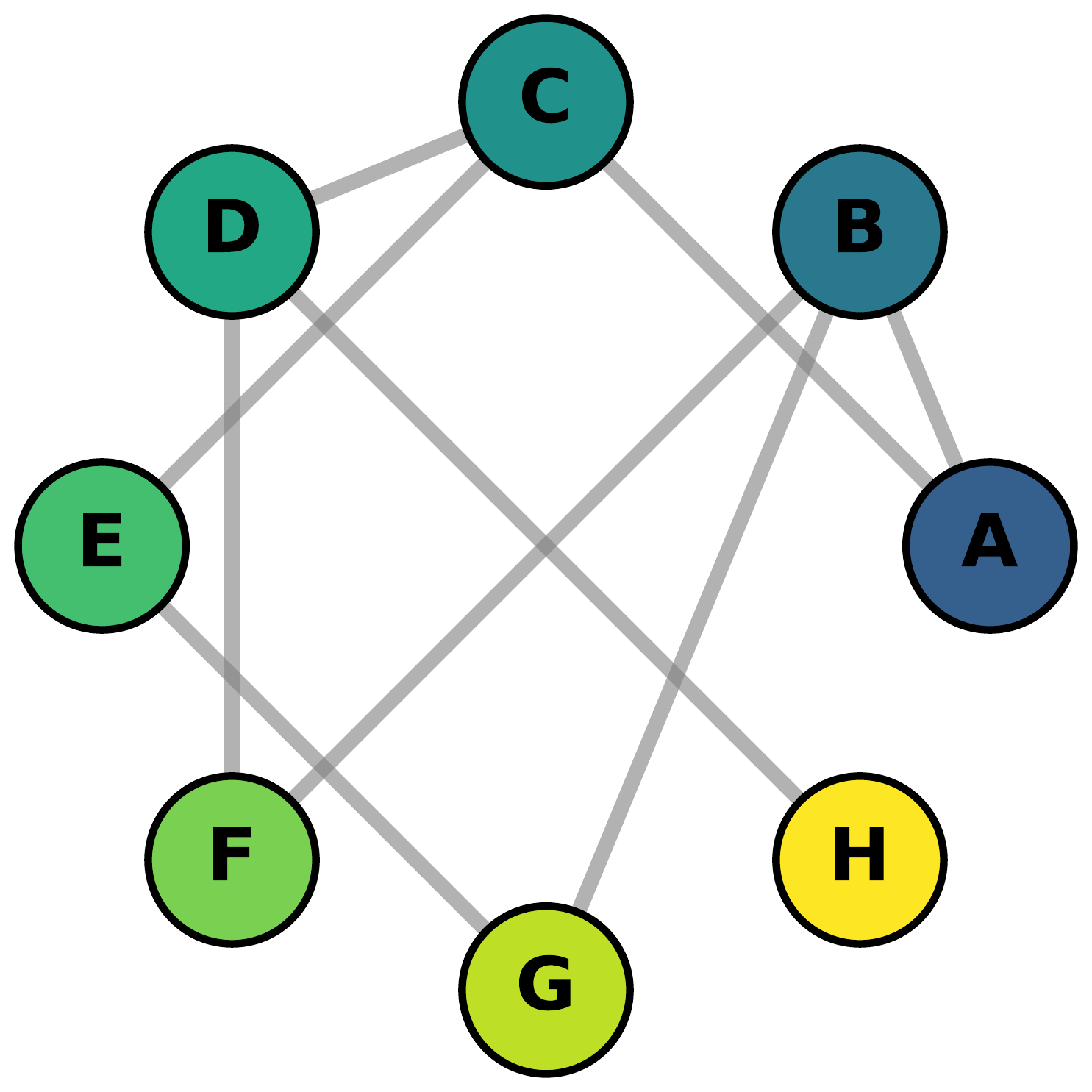}
    }
    \subfigure[GPT-3.5]{
        \includegraphics[width=0.23\linewidth]{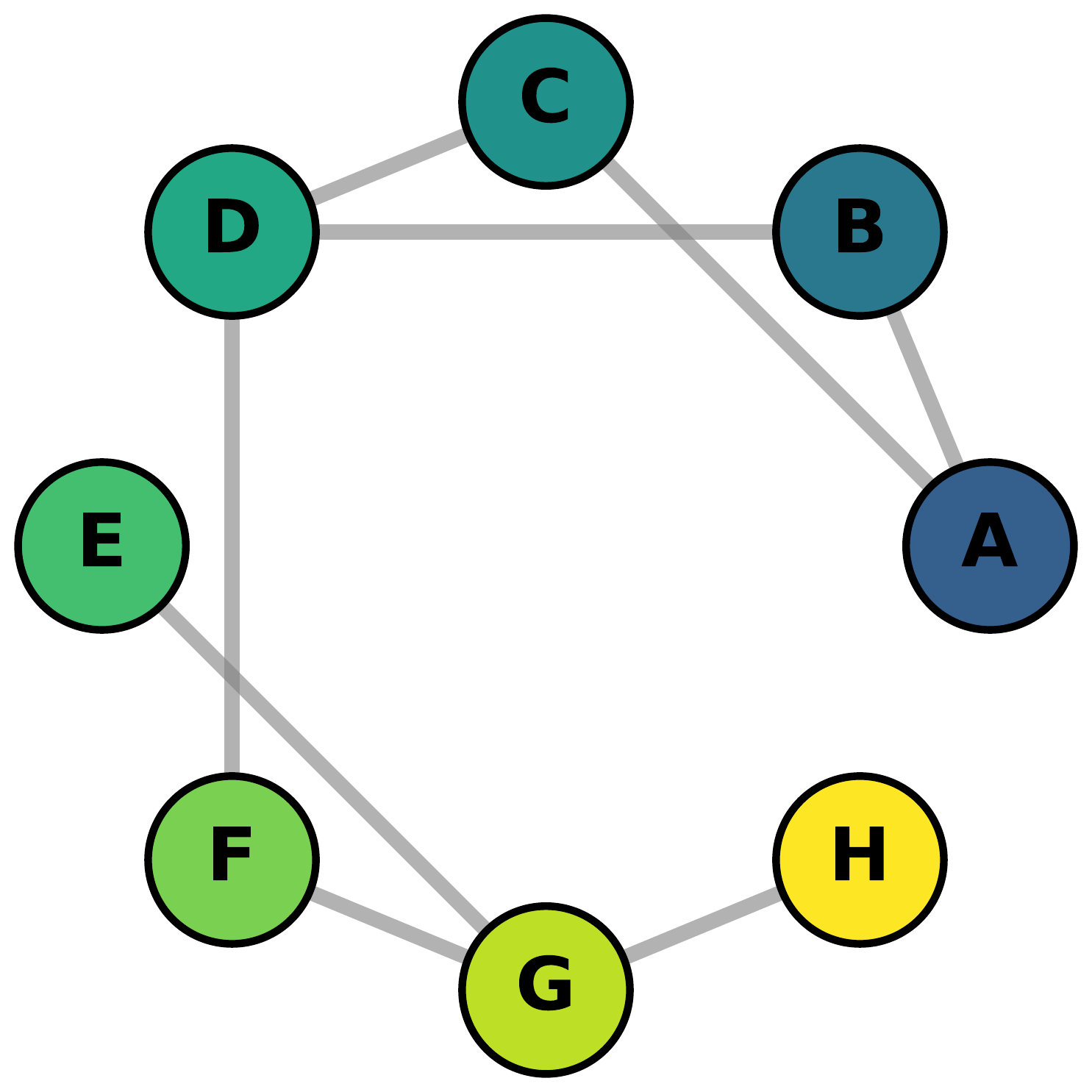}
    }
    \subfigure[GPT-4]{
        \includegraphics[width=0.23\linewidth]{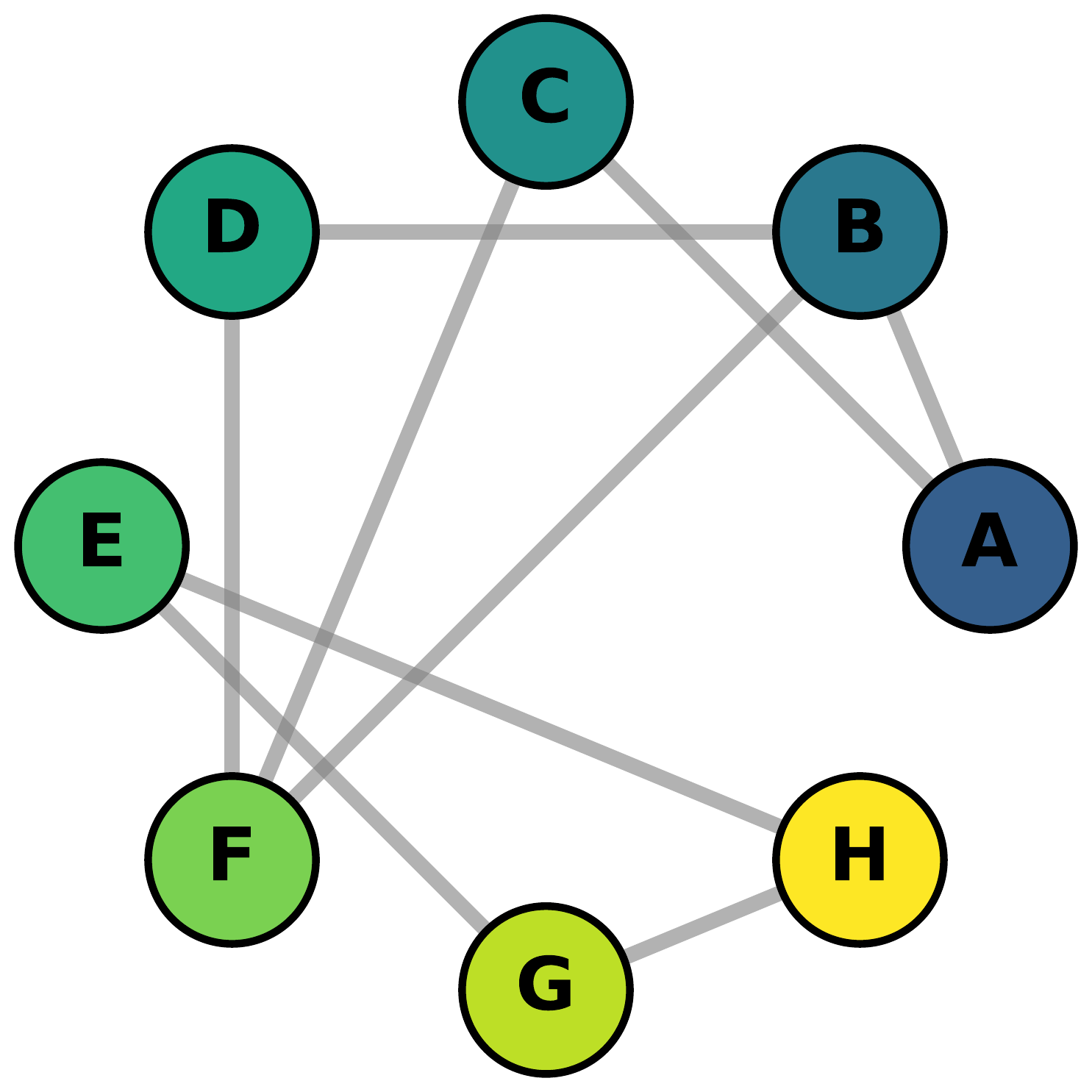}
    }
    \caption{Cake.}
    \label{fig:appx-cake}
\end{figure*}

\begin{figure*}[th]
    \centering
    \subfigure[Ground Truth]{
        \includegraphics[width=0.23\linewidth]{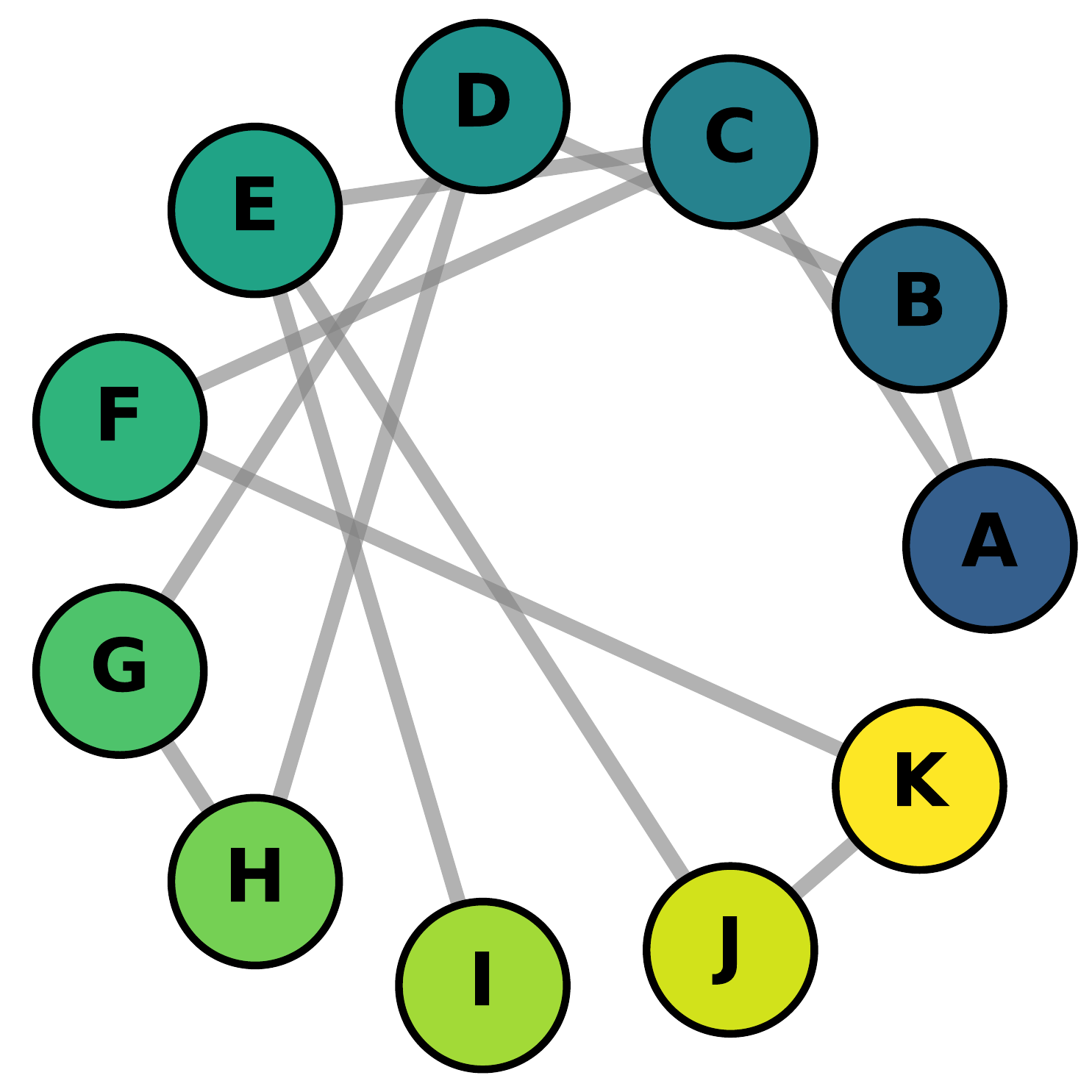}
    }
    \subfigure[LLAMA-2-70B]{
        \includegraphics[width=0.23\linewidth]{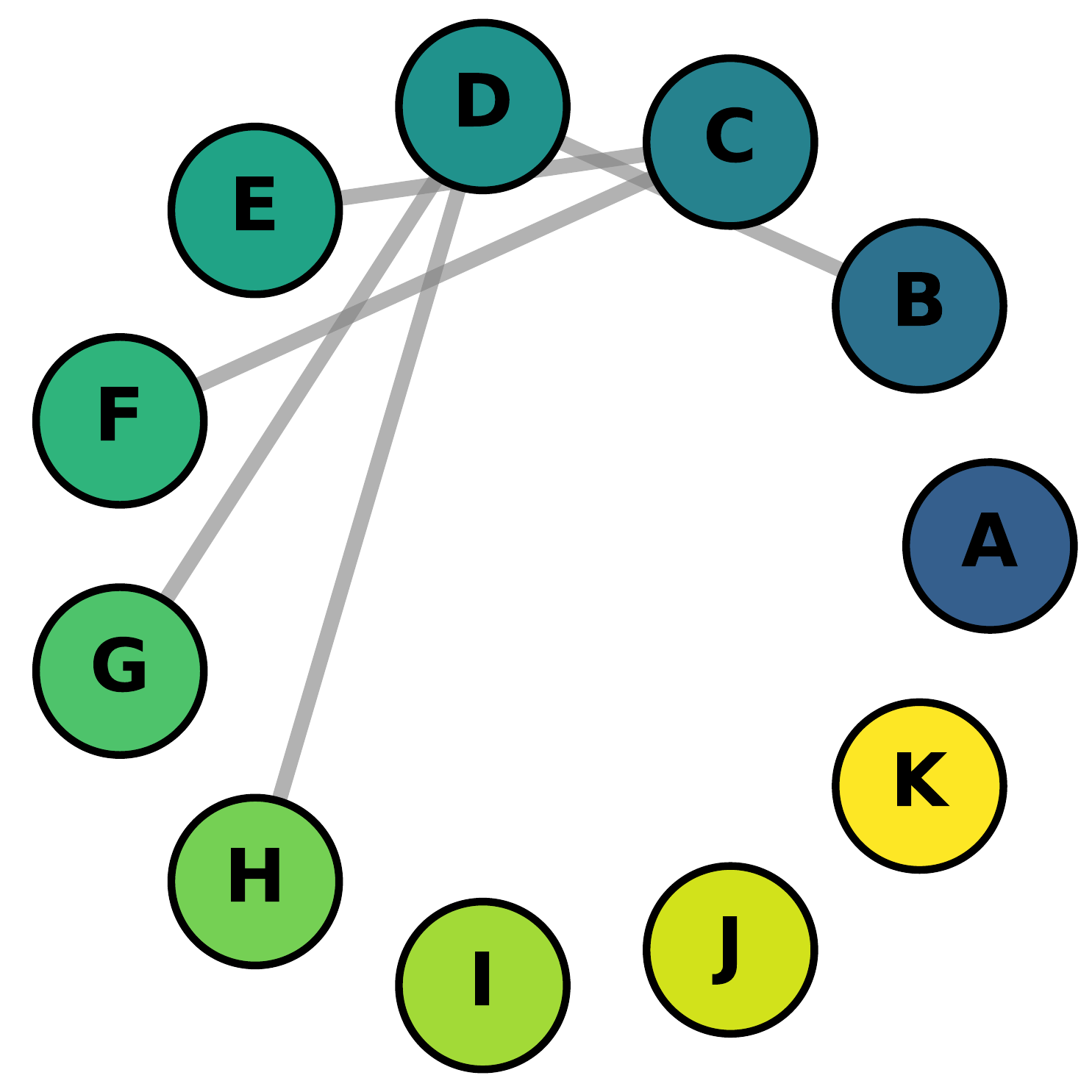}
    }
    \subfigure[GPT-3.5]{
        \includegraphics[width=0.23\linewidth]{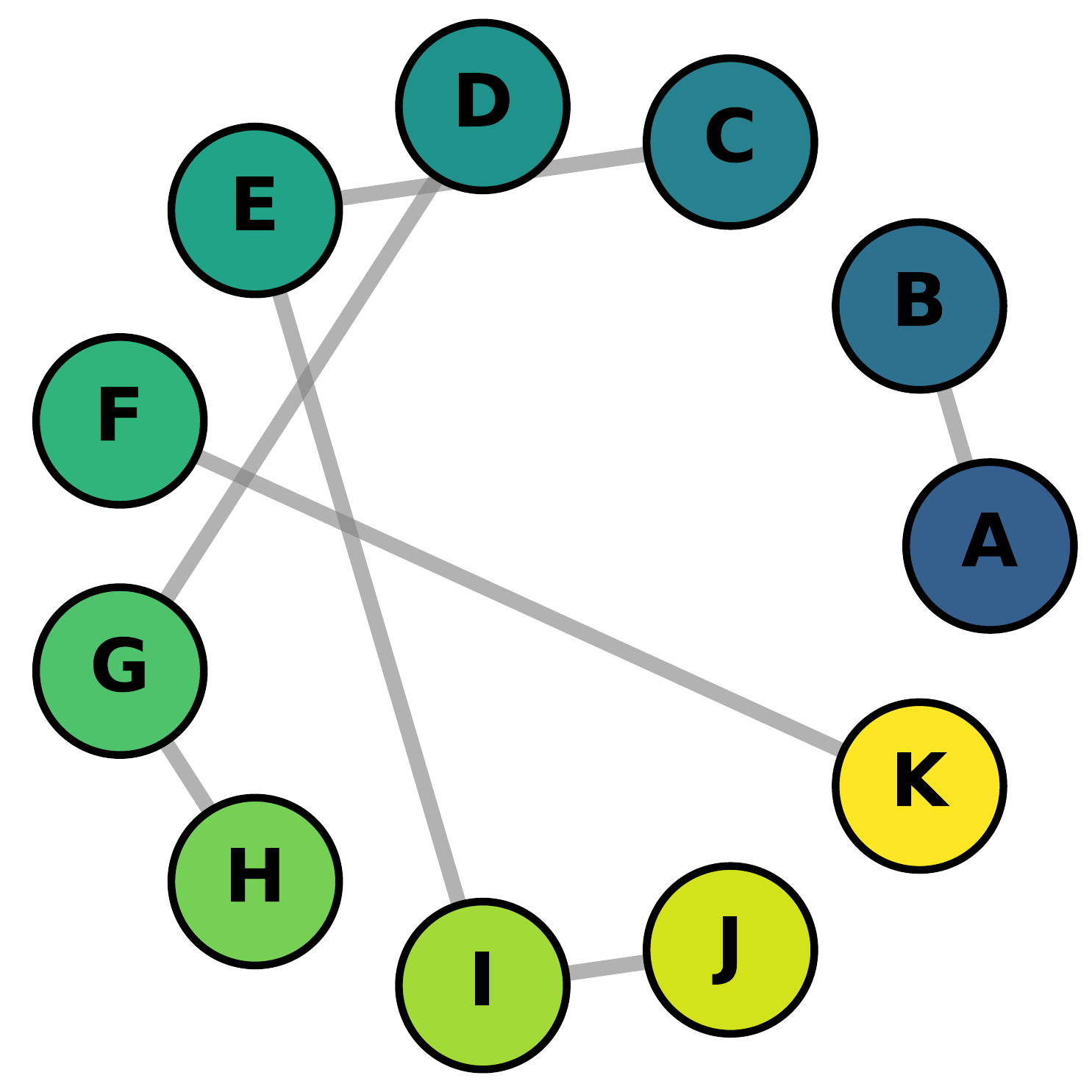}
    }
    \subfigure[GPT-4]{
        \includegraphics[width=0.23\linewidth]{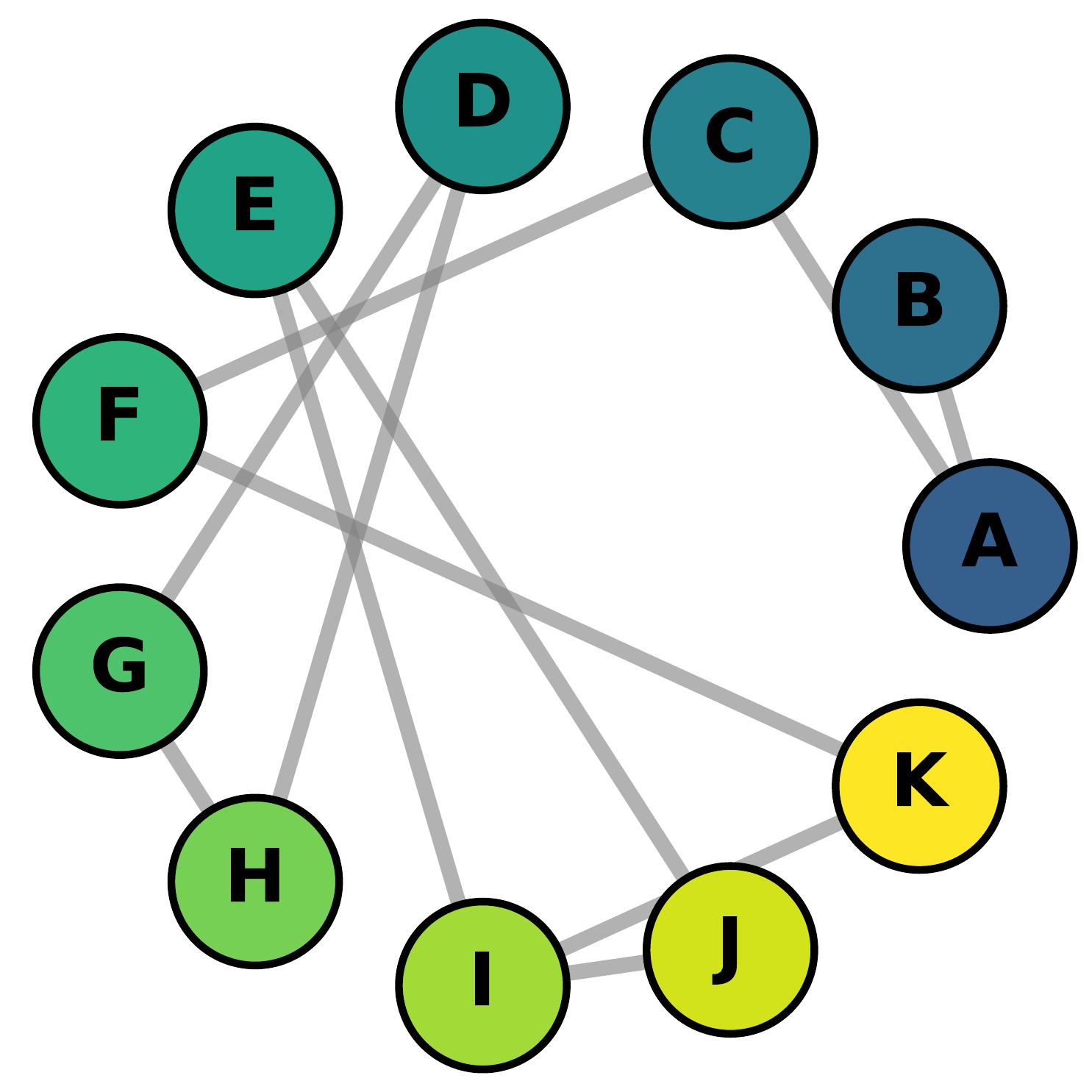}
    }
    \caption{Dog.}
    \label{fig:appx-dog}
\end{figure*}

\begin{figure*}[th]
    \centering
    \subfigure[Ground Truth]{
        \includegraphics[width=0.23\linewidth]{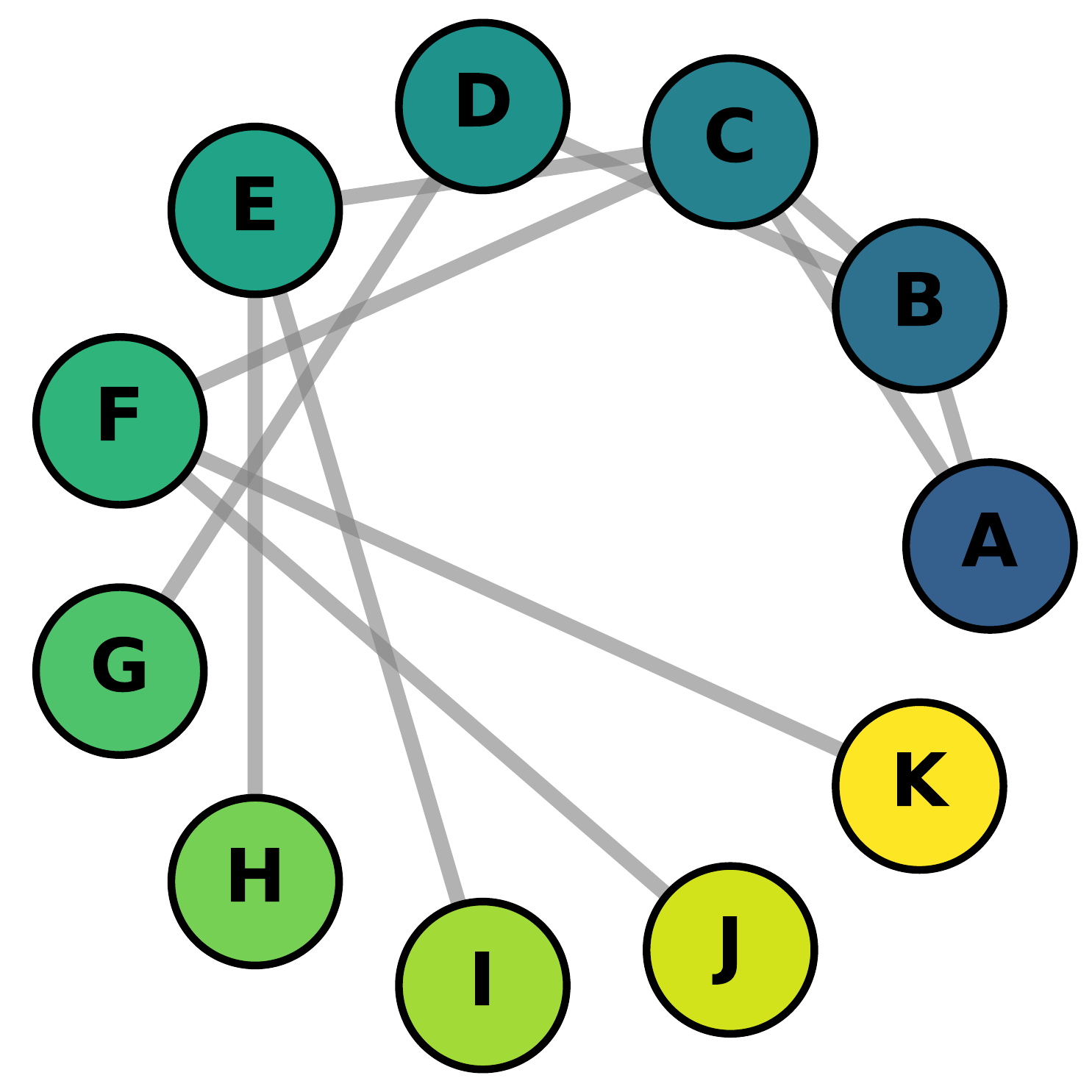}
    }
    \subfigure[LLAMA-2-70B]{
        \includegraphics[width=0.23\linewidth]{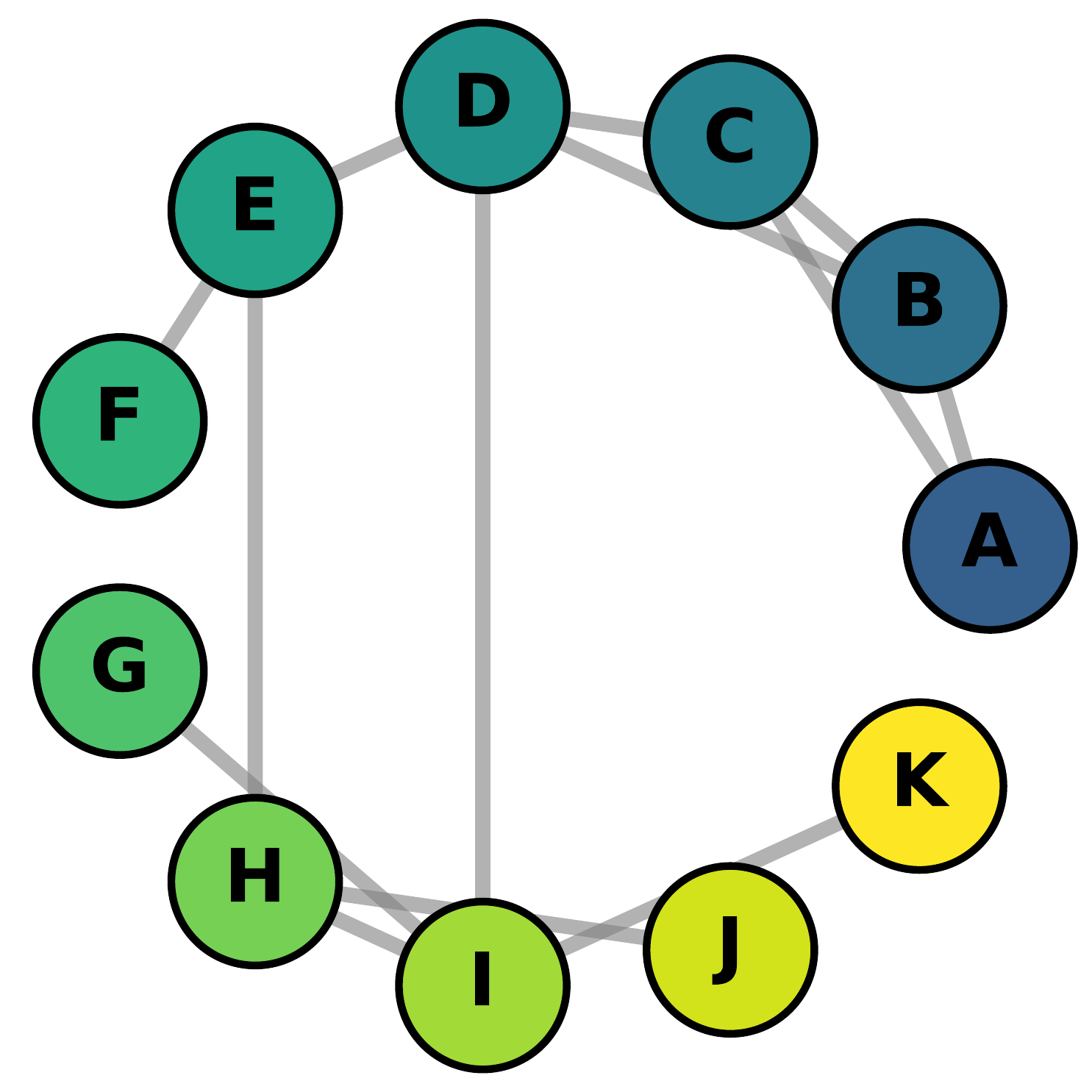}
    }
    \subfigure[GPT-3.5]{
        \includegraphics[width=0.23\linewidth]{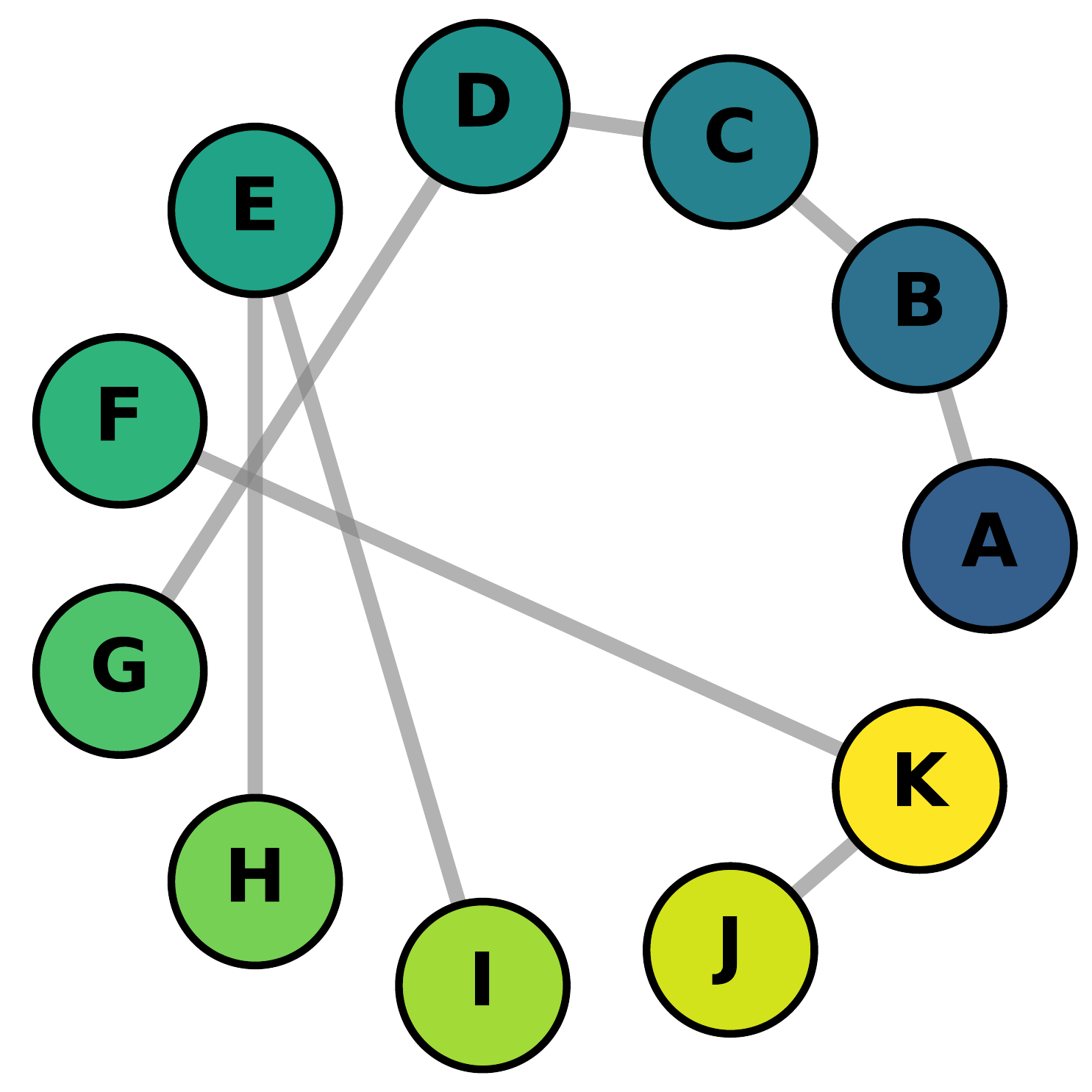}
    }
    \subfigure[GPT-4]{
        \includegraphics[width=0.23\linewidth]{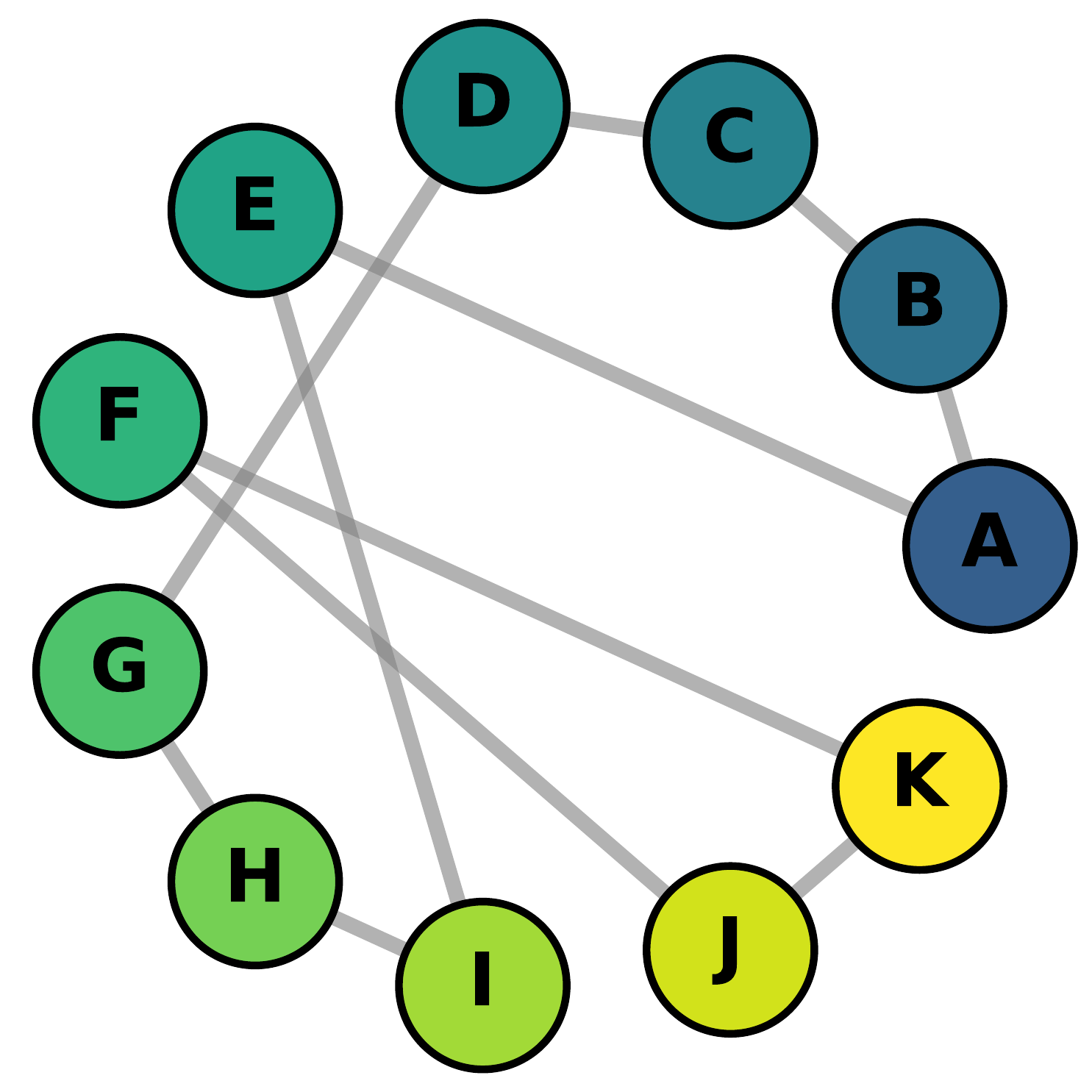}
    }
    \caption{Fly.}
    \label{fig:appx-fly}
\end{figure*}

\begin{figure*}[th]
    \centering
    \subfigure[Ground Truth]{
        \includegraphics[width=0.23\linewidth]{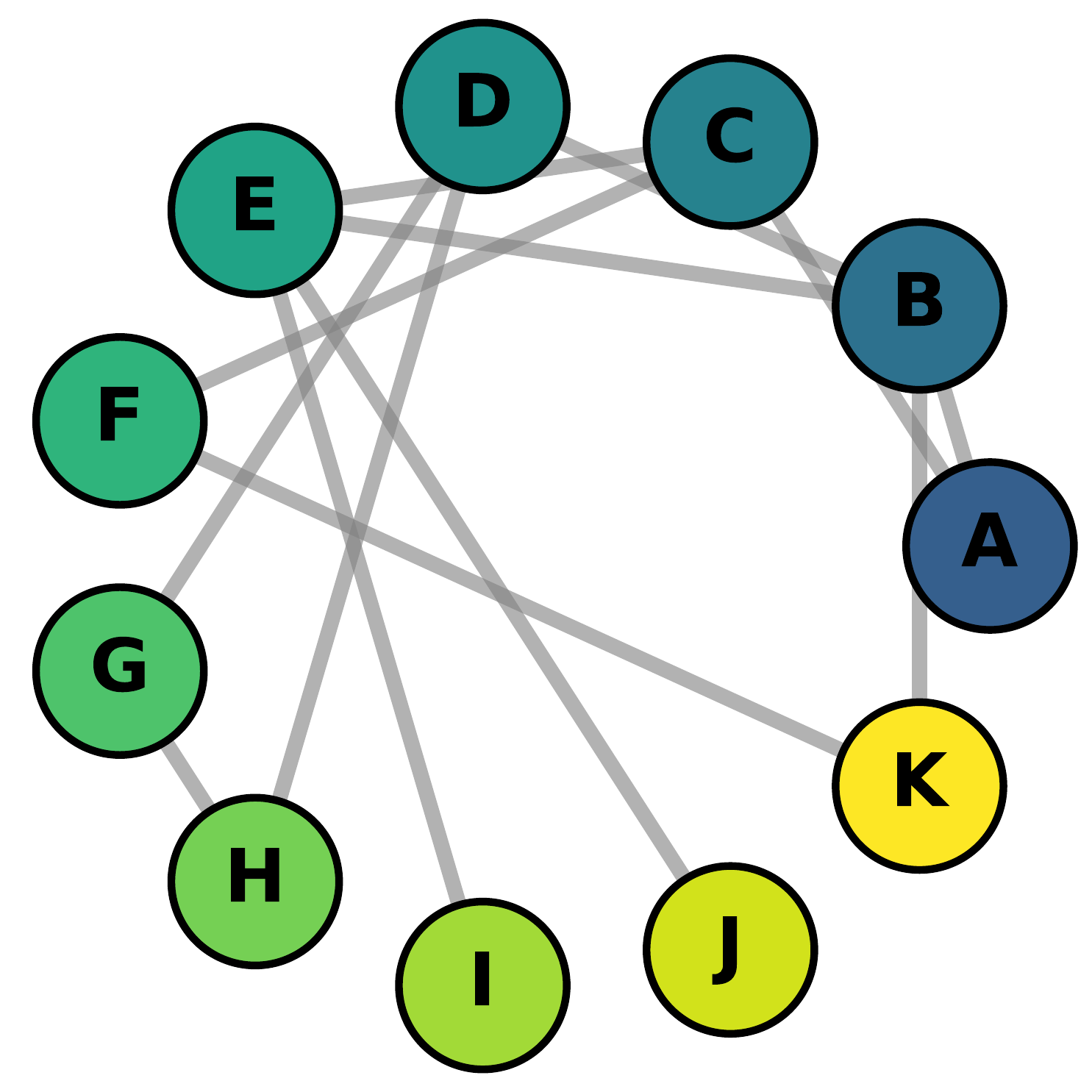}
    }
    \subfigure[LLAMA-2-70B]{
        \includegraphics[width=0.23\linewidth]{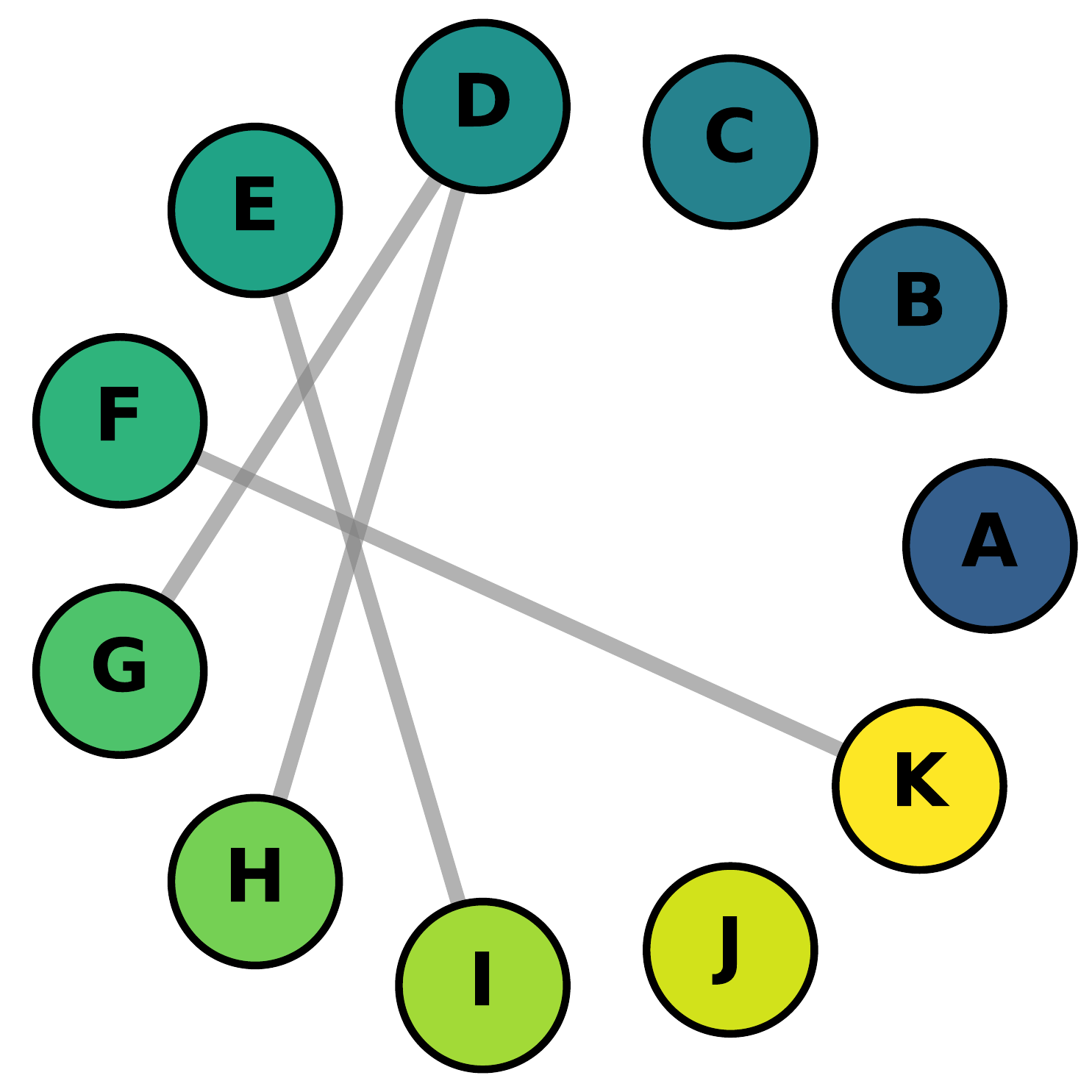}
    }
    \subfigure[GPT-3.5]{
        \includegraphics[width=0.23\linewidth]{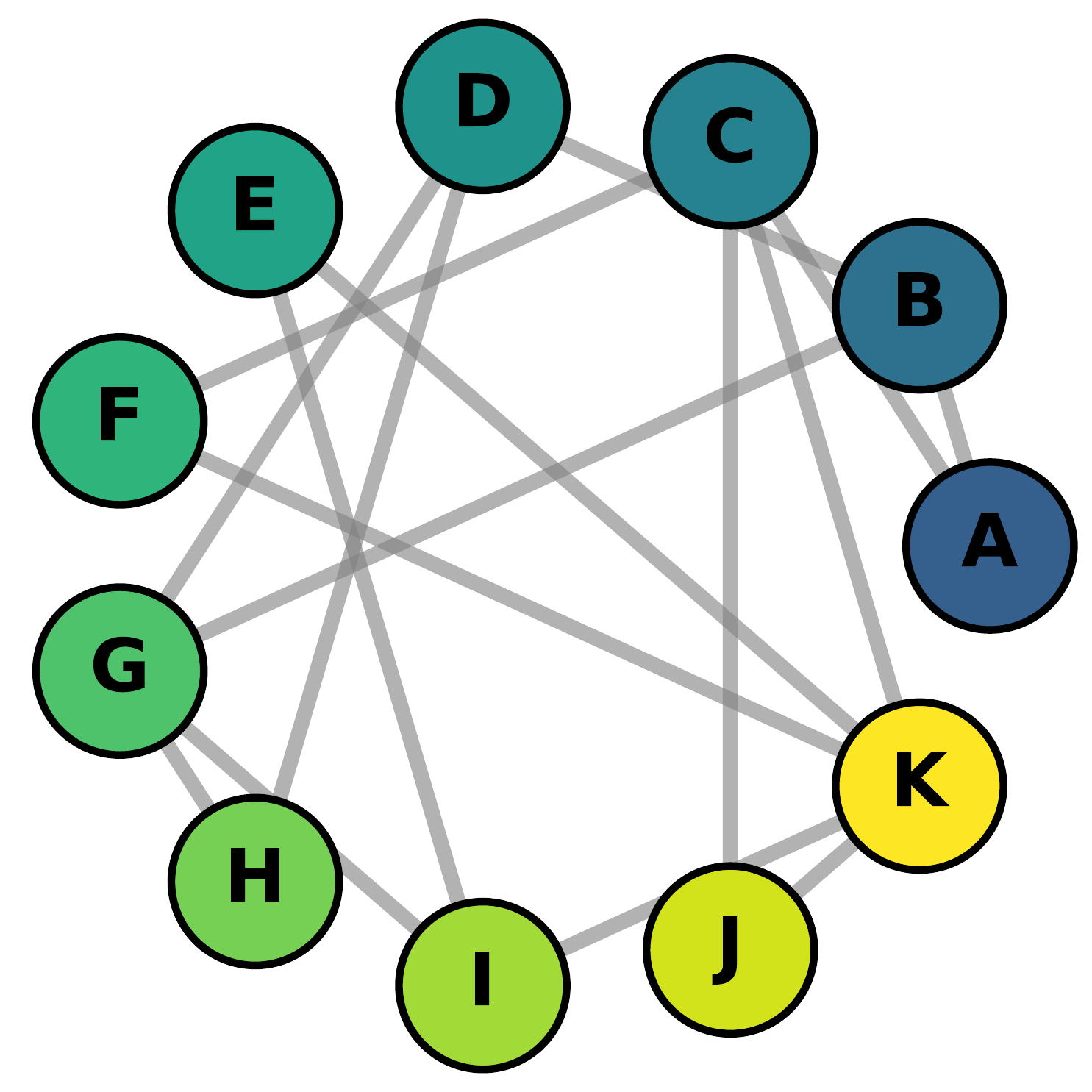}
    }
    \subfigure[GPT-4]{
        \includegraphics[width=0.23\linewidth]{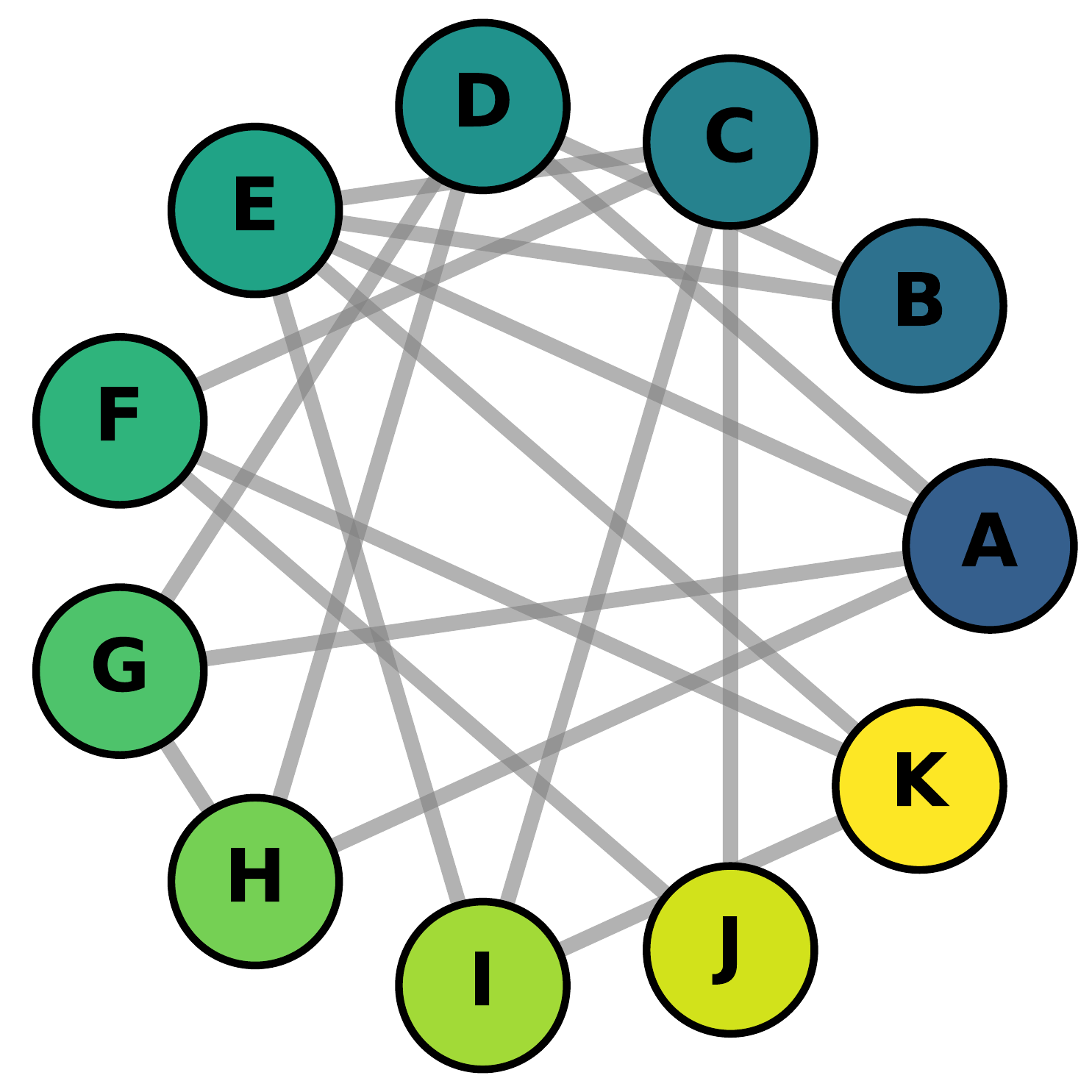}
    }
    \caption{Human.}
    \label{fig:appx-human}
\end{figure*}

\begin{figure*}[th]
    \centering
    \subfigure[Ground Truth]{
        \includegraphics[width=0.23\linewidth]{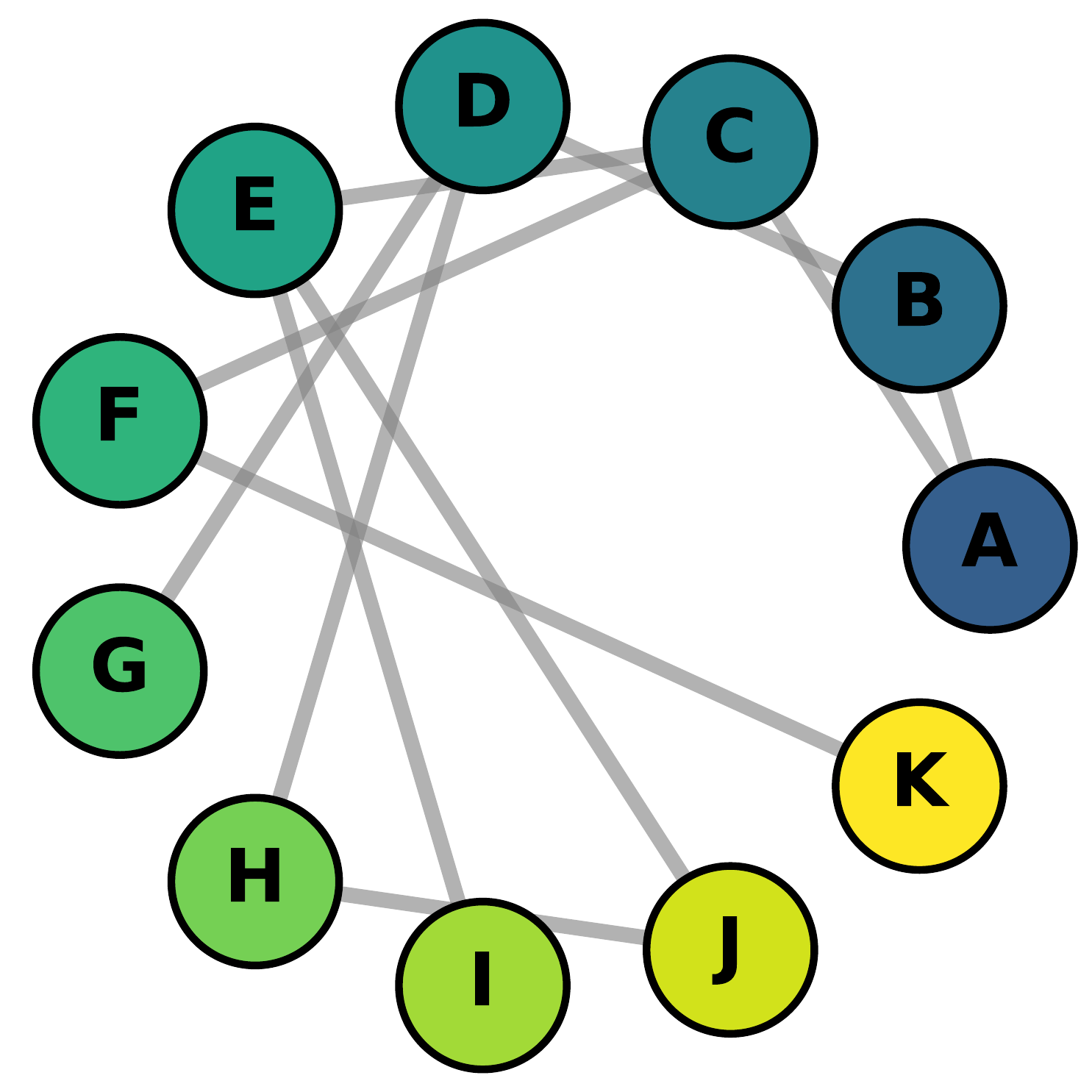}
    }
    \subfigure[LLAMA-2-70B]{
        \includegraphics[width=0.23\linewidth]{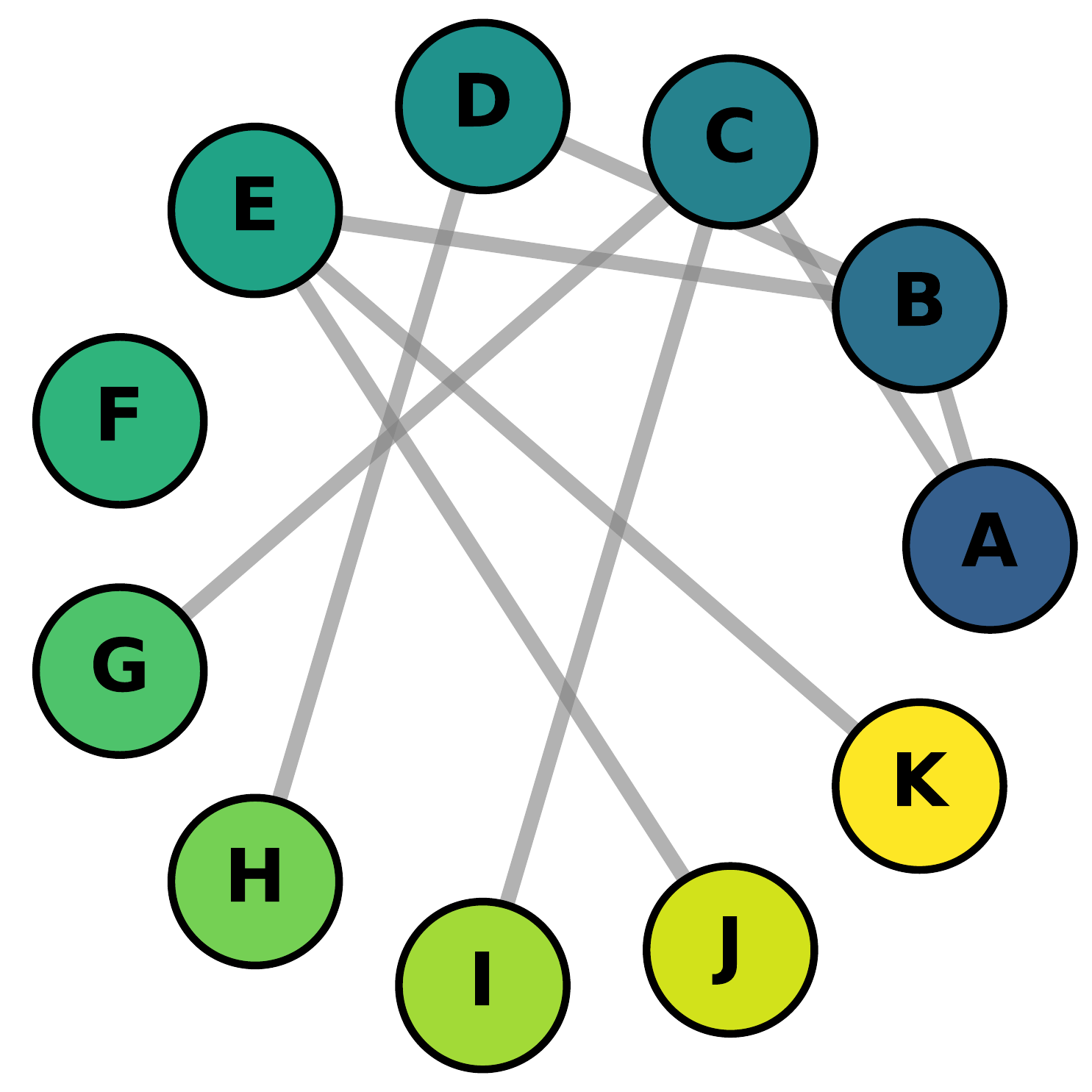}
    }
    \subfigure[GPT-3.5]{
        \includegraphics[width=0.23\linewidth]{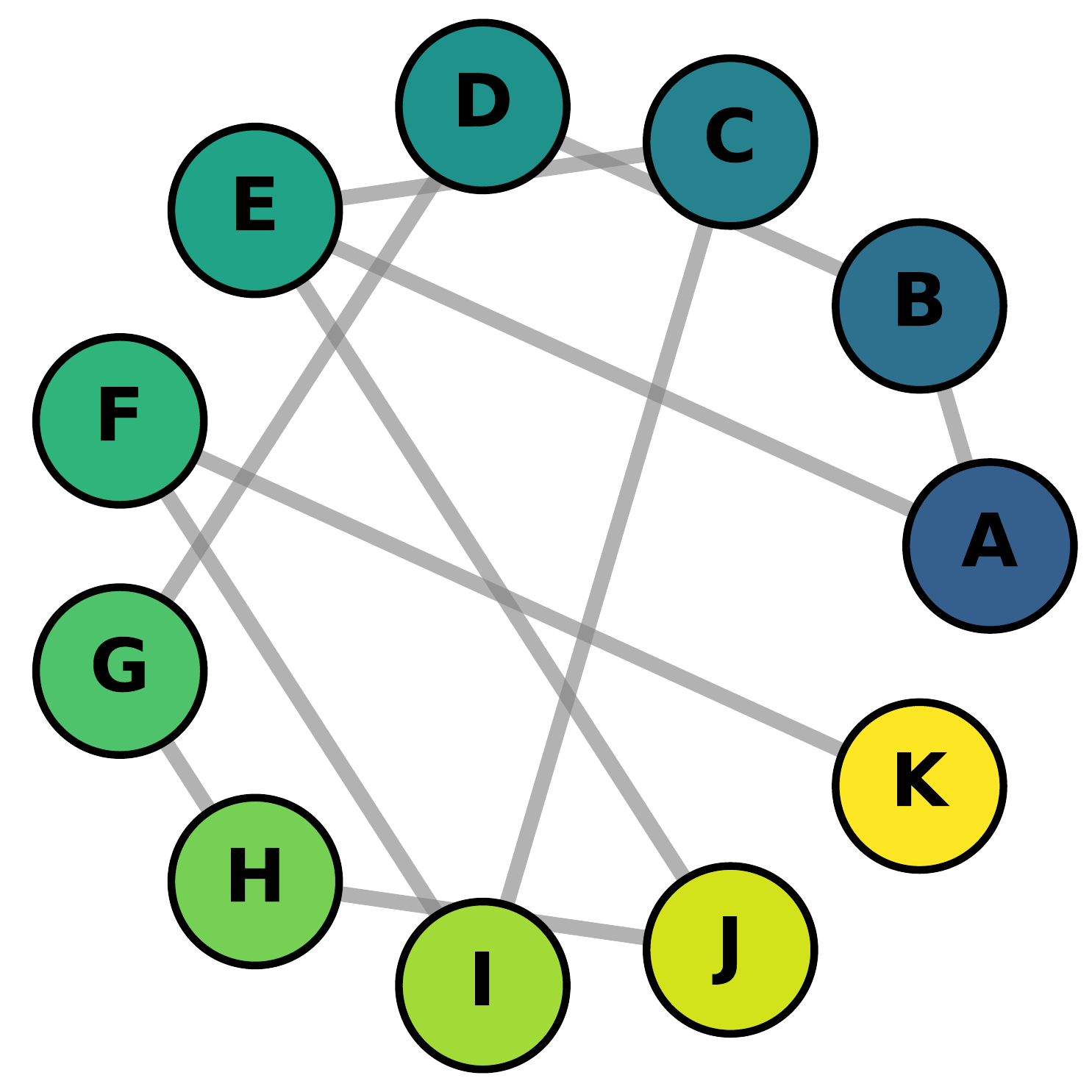}
    }
    \subfigure[GPT-4]{
        \includegraphics[width=0.23\linewidth]{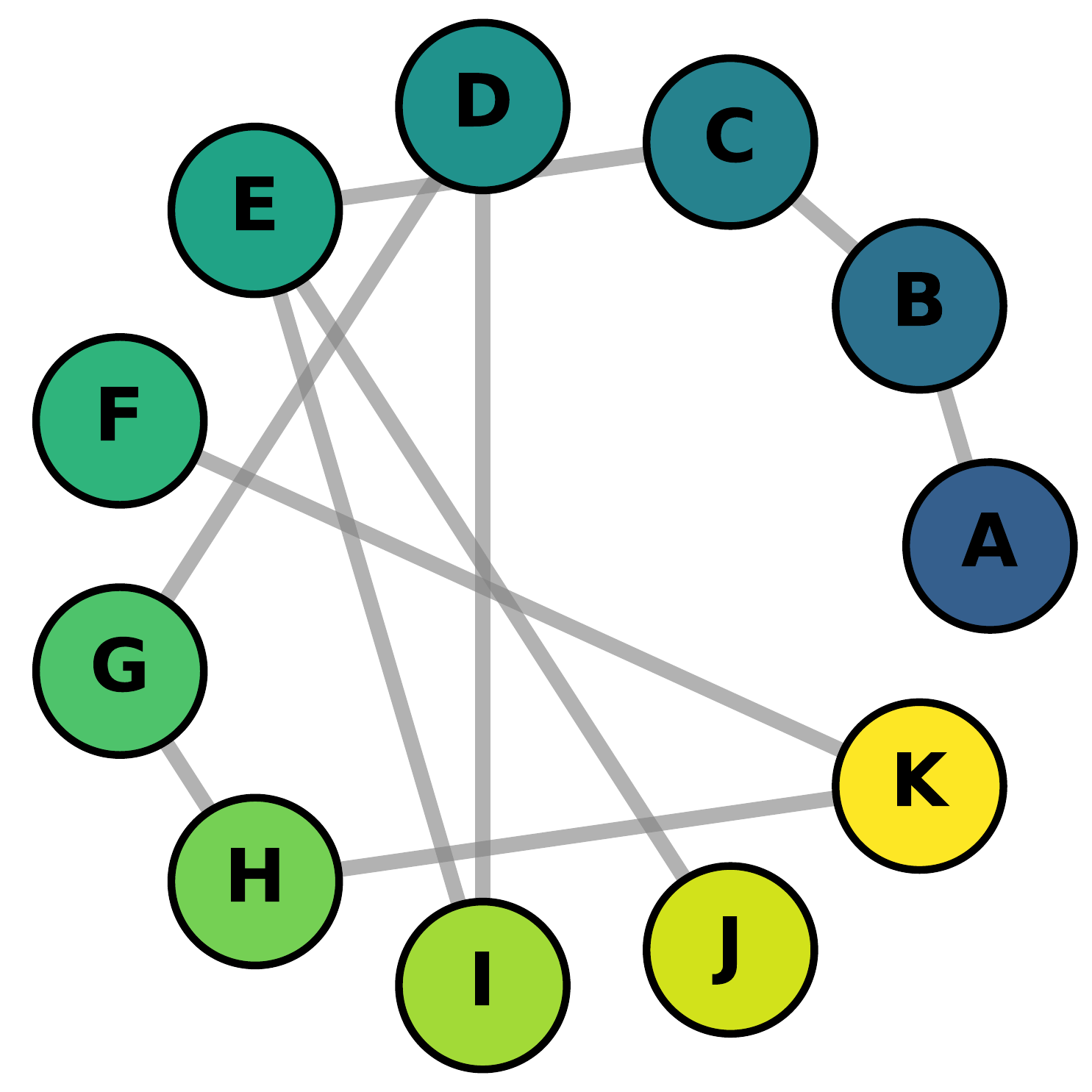}
    }
    \caption{Jacket.}
    \label{fig:appx-jacket}
\end{figure*}

\begin{figure*}[th]
    \centering
    \subfigure[Ground Truth]{
        \includegraphics[width=0.23\linewidth]{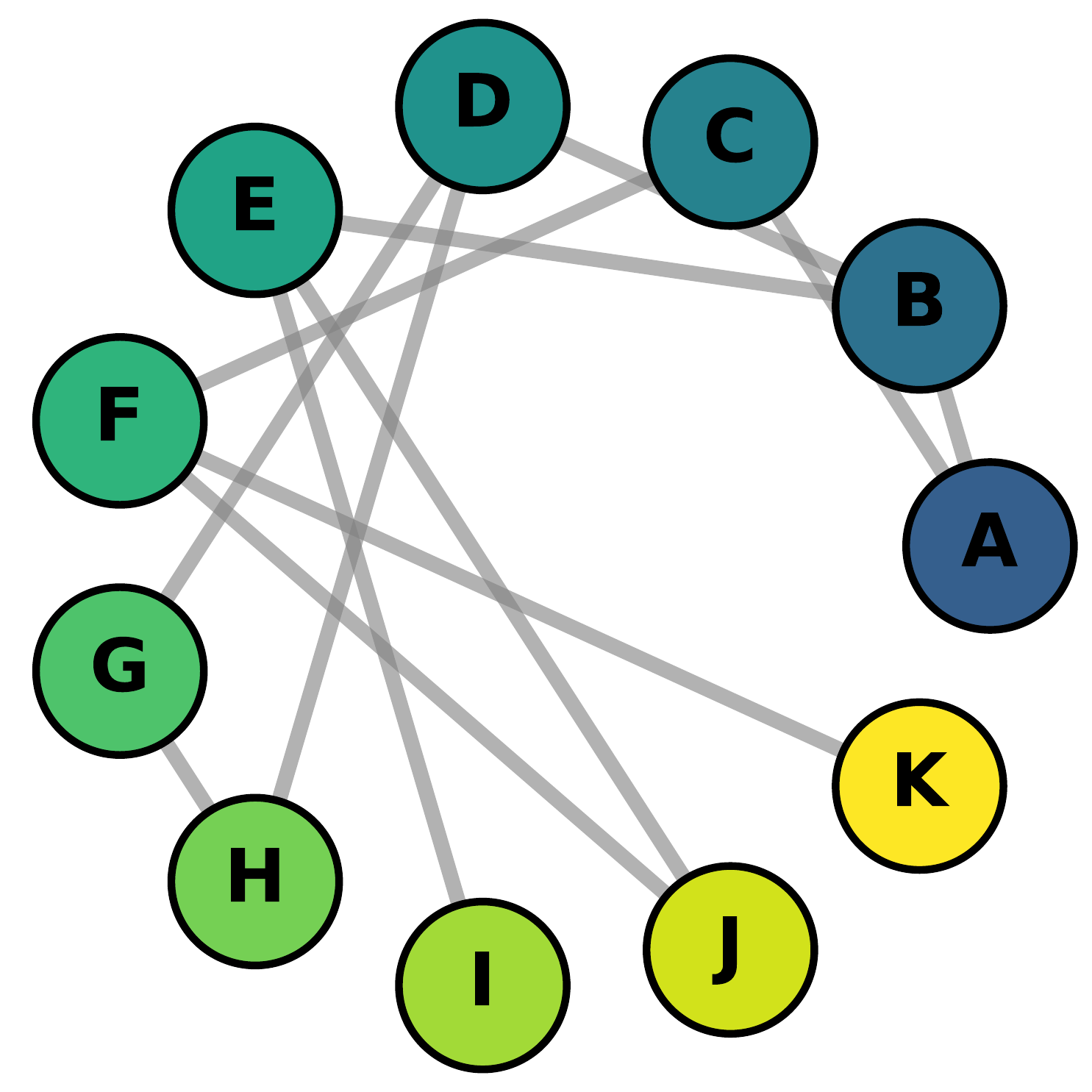}
    }
    \subfigure[LLAMA-2-70B]{
        \includegraphics[width=0.23\linewidth]{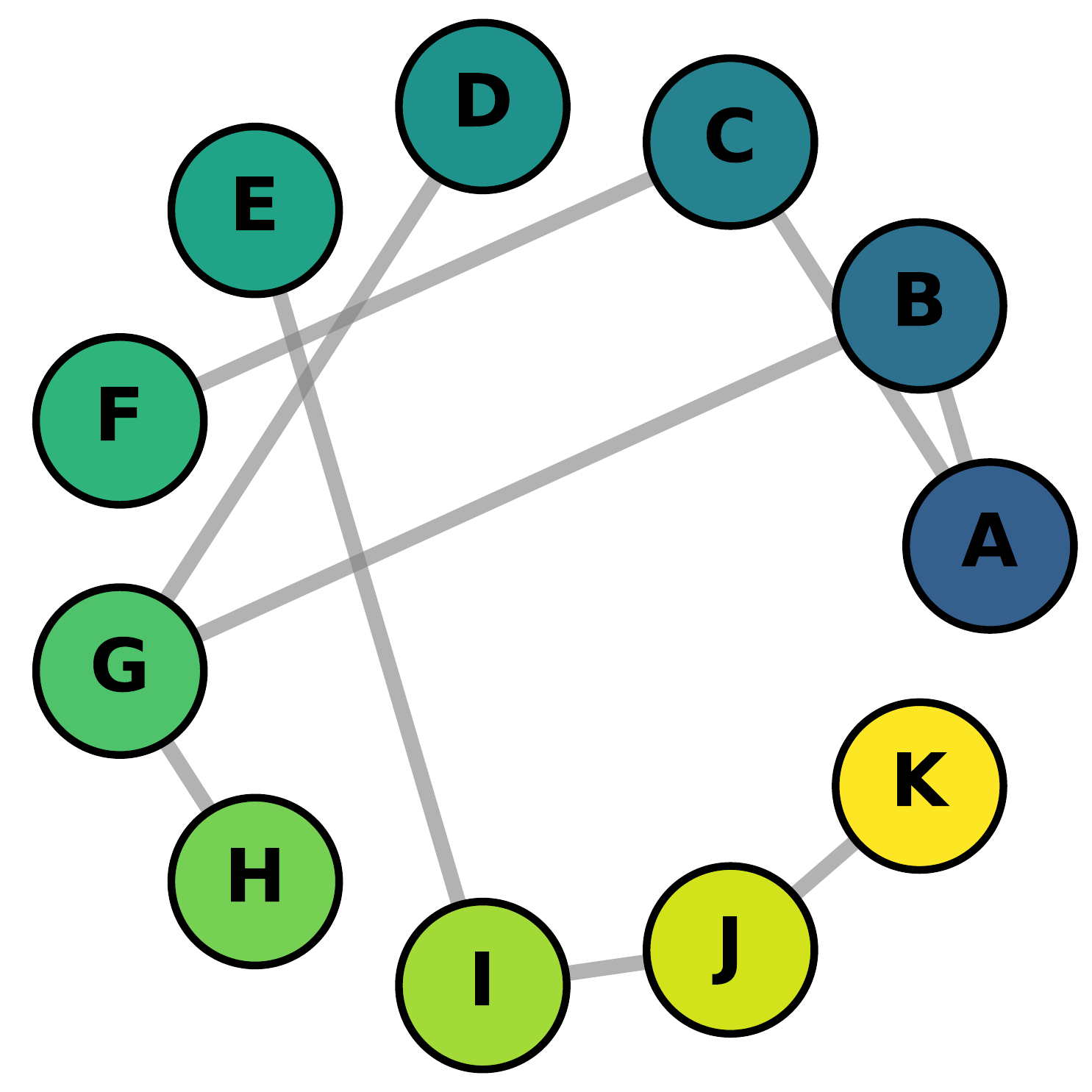}
    }
    \subfigure[GPT-3.5]{
        \includegraphics[width=0.23\linewidth]{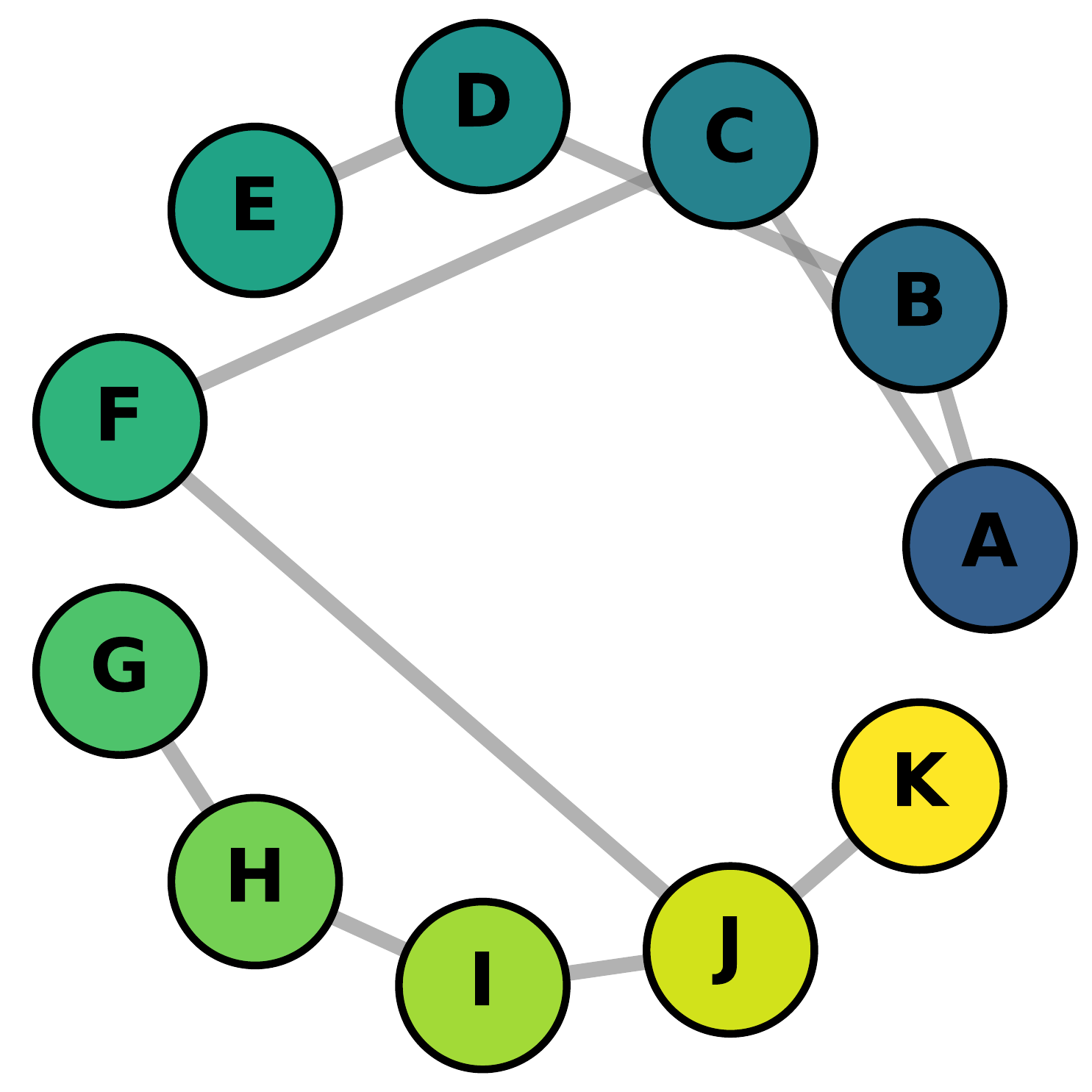}
    }
    \subfigure[GPT-4]{
        \includegraphics[width=0.23\linewidth]{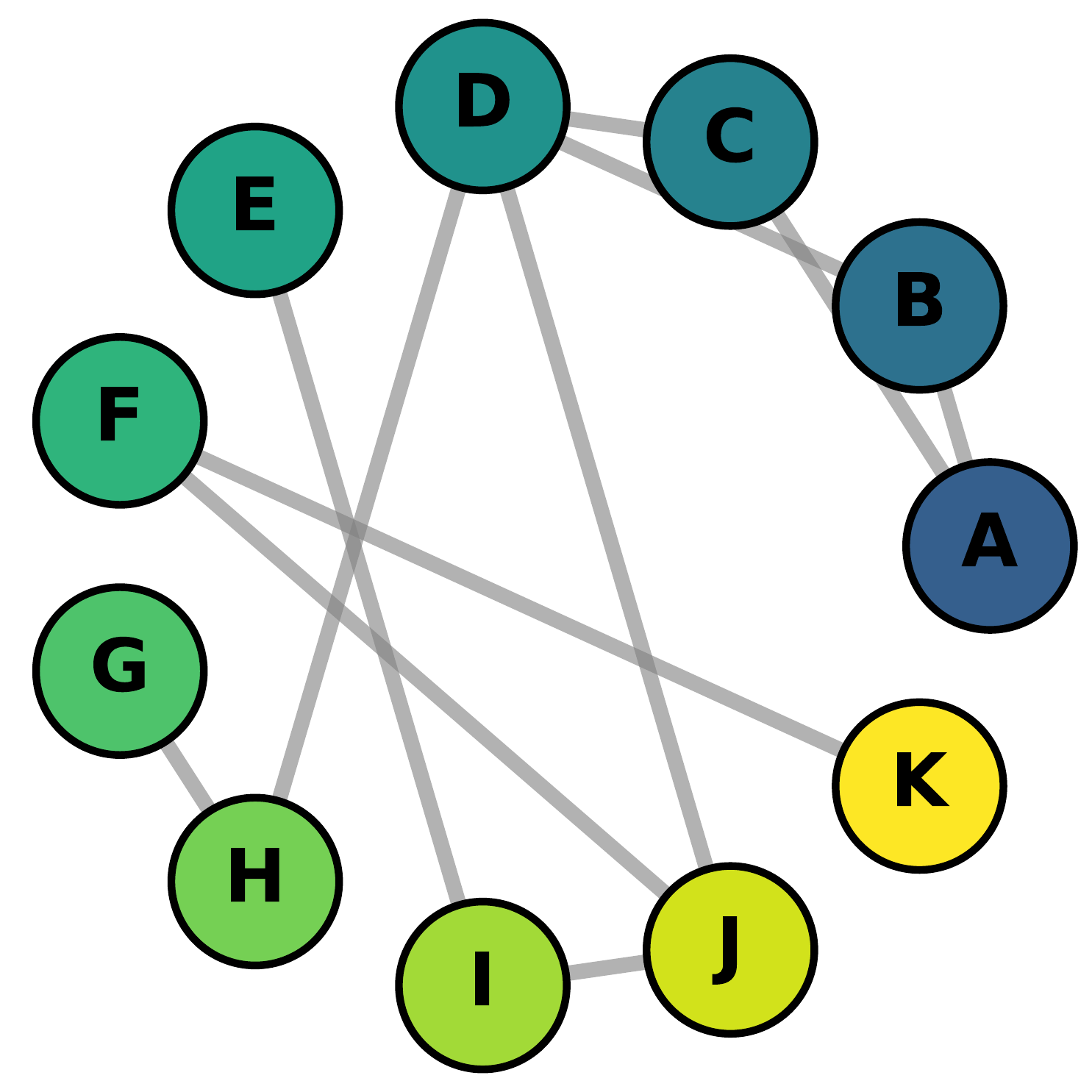}
    }
    \caption{Orange.}
    \label{fig:appx-orange}
\end{figure*}

\begin{figure*}[th]
    \centering
    \subfigure[Ground Truth]{
        \includegraphics[width=0.23\linewidth]{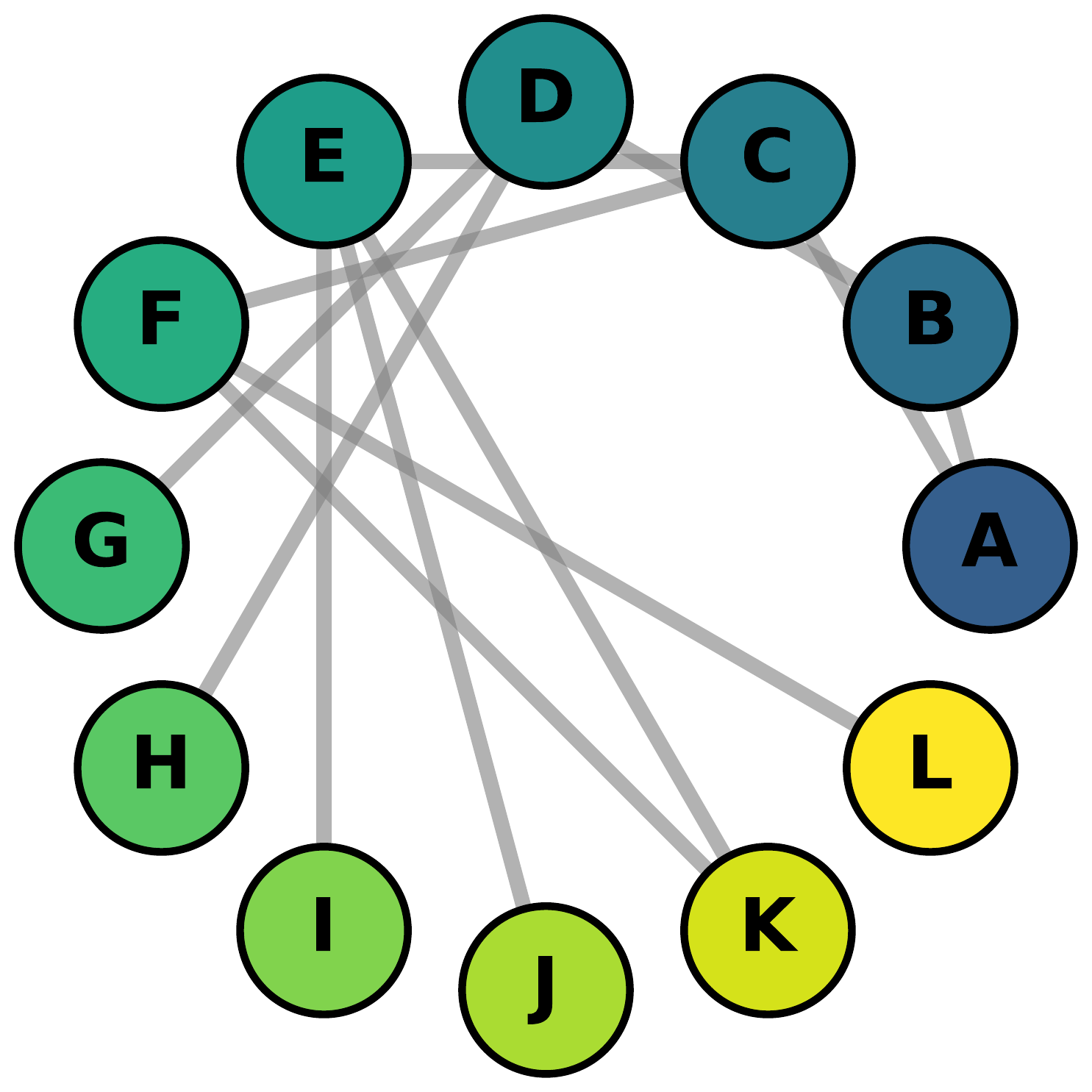}
    }
    \subfigure[LLAMA-2-70B]{
        \includegraphics[width=0.23\linewidth]{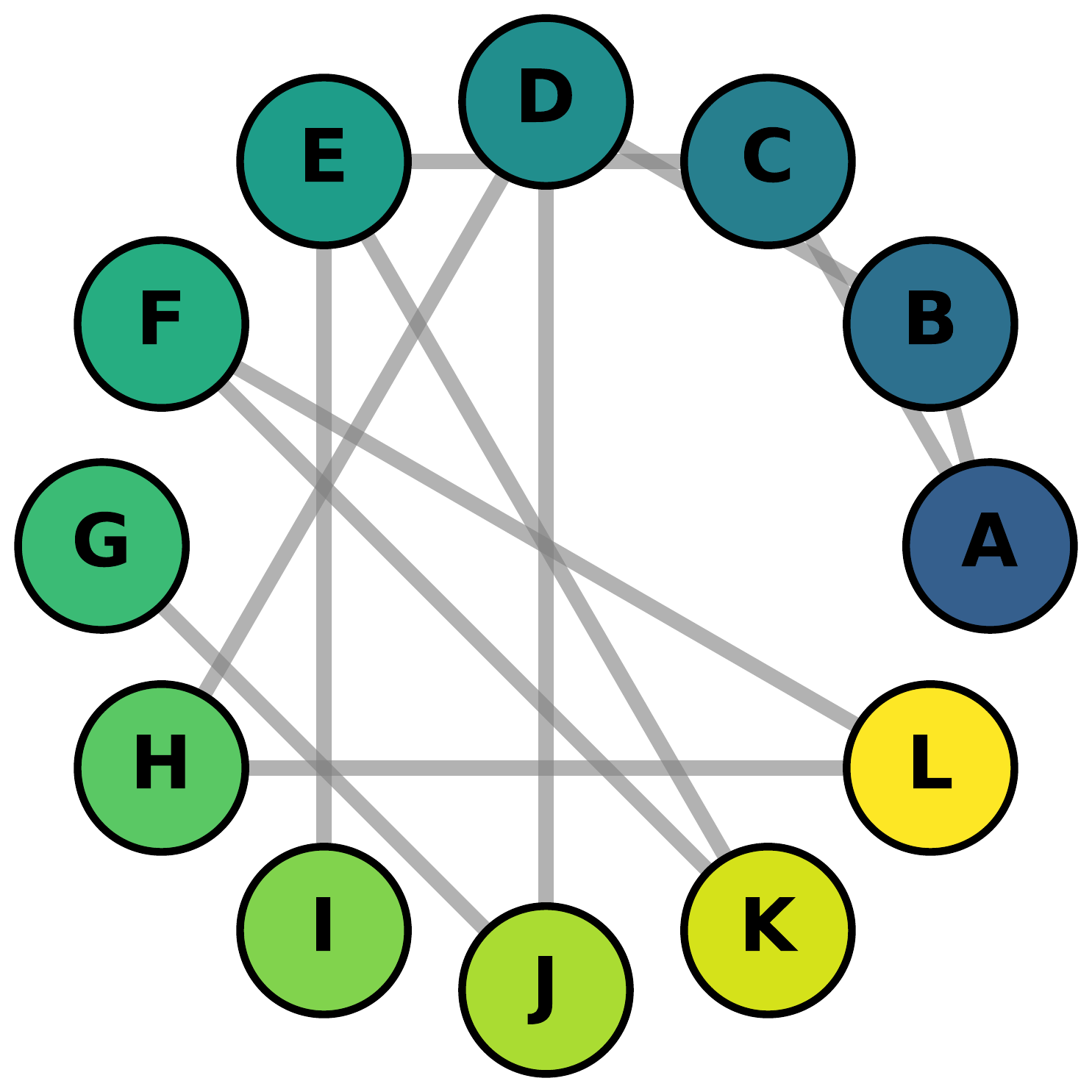}
    }
    \subfigure[GPT-3.5]{
        \includegraphics[width=0.23\linewidth]{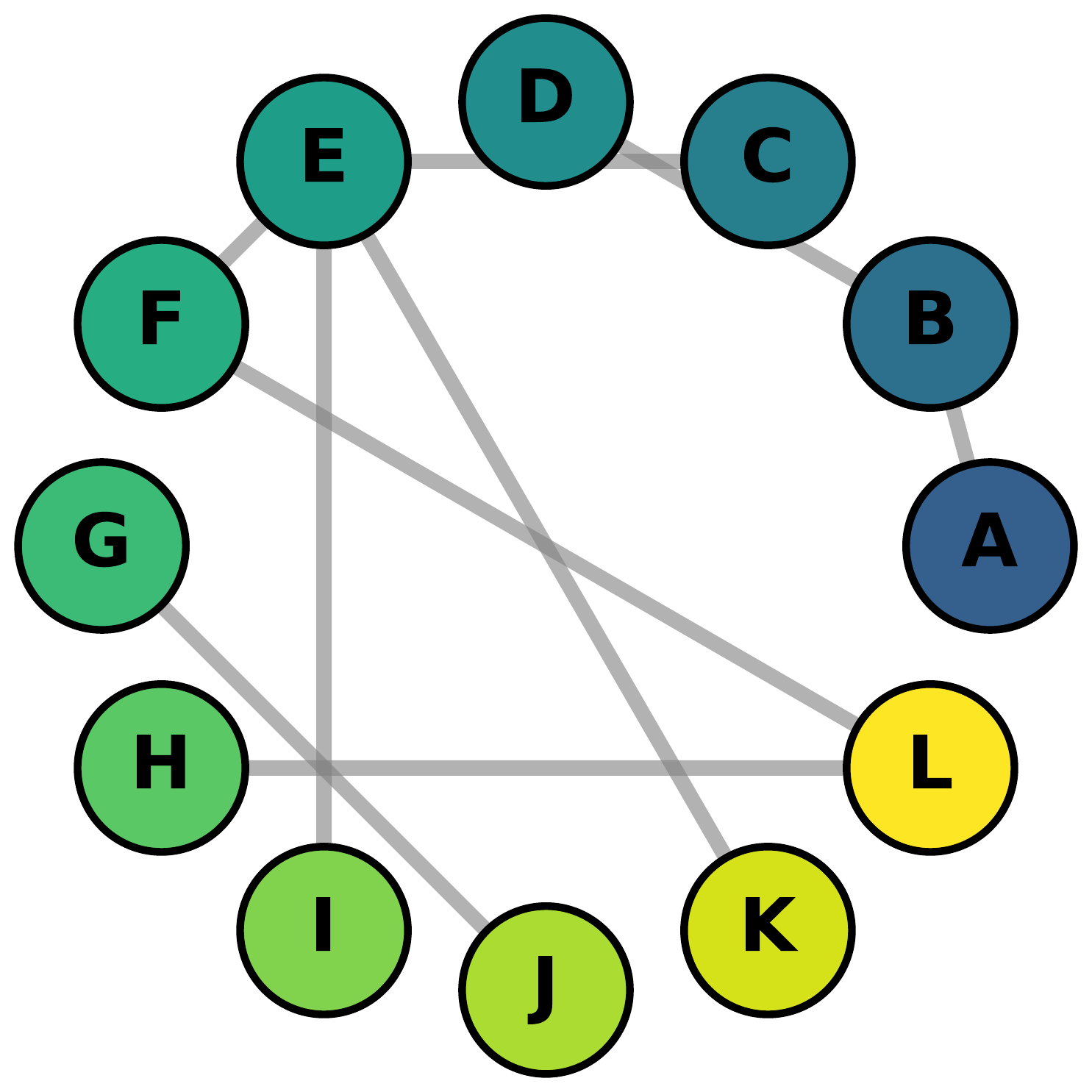}
    }
    \subfigure[GPT-4]{
        \includegraphics[width=0.23\linewidth]{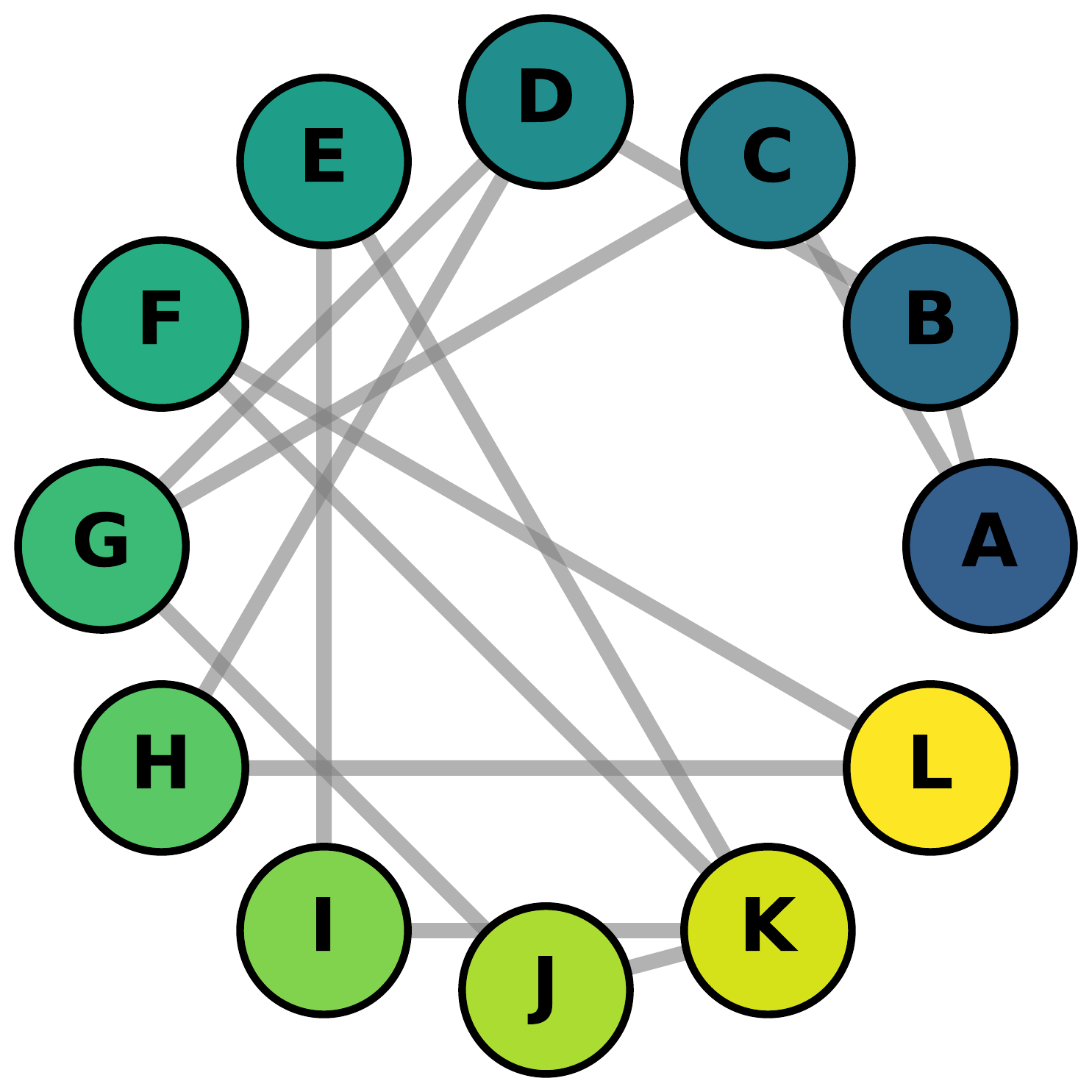}
    }
    \caption{Paper.}
    \label{fig:appx-paper}
\end{figure*}

\begin{figure*}[th]
    \centering
    \subfigure[Ground Truth]{
        \includegraphics[width=0.23\linewidth]{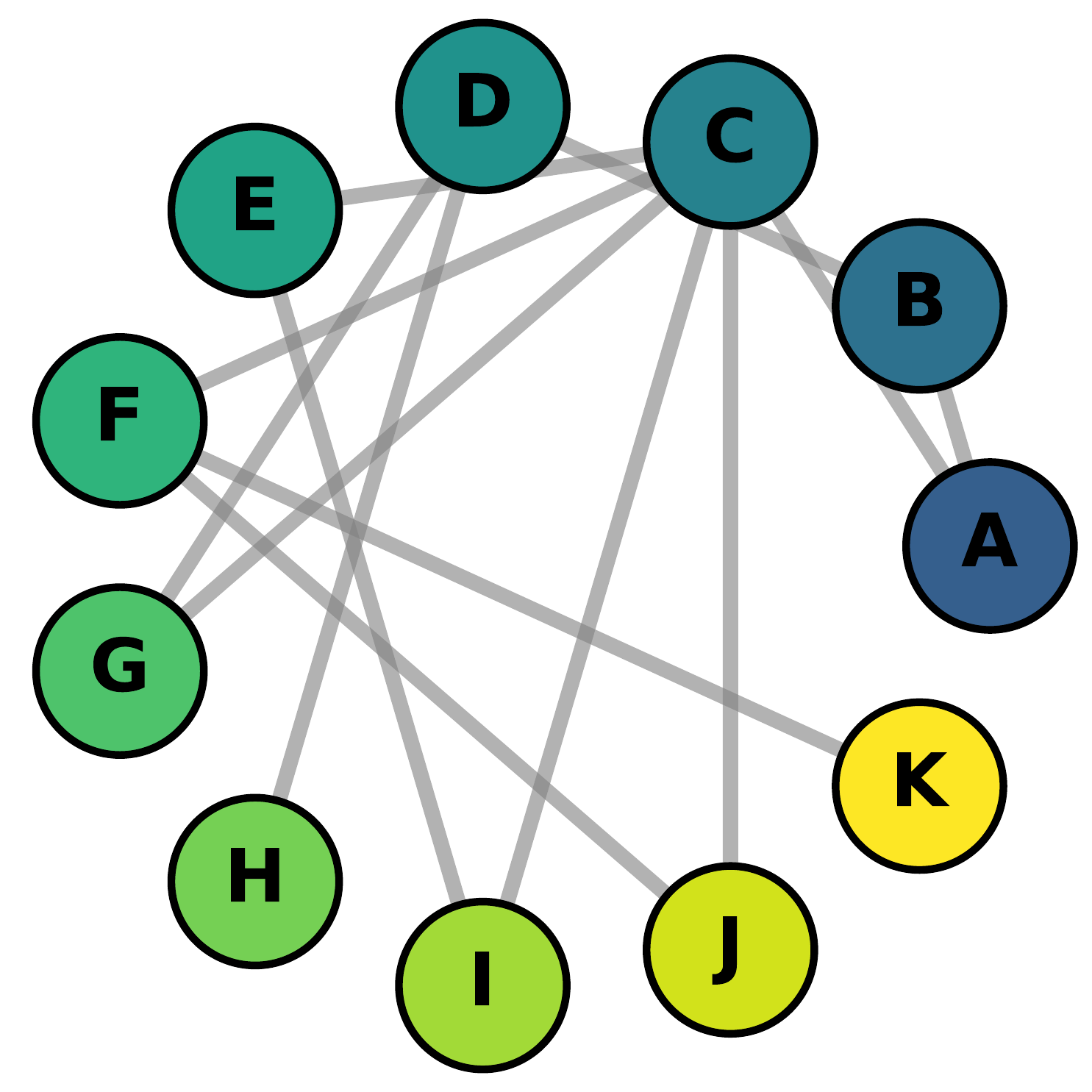}
    }
    \subfigure[LLAMA-2-70B]{
        \includegraphics[width=0.23\linewidth]{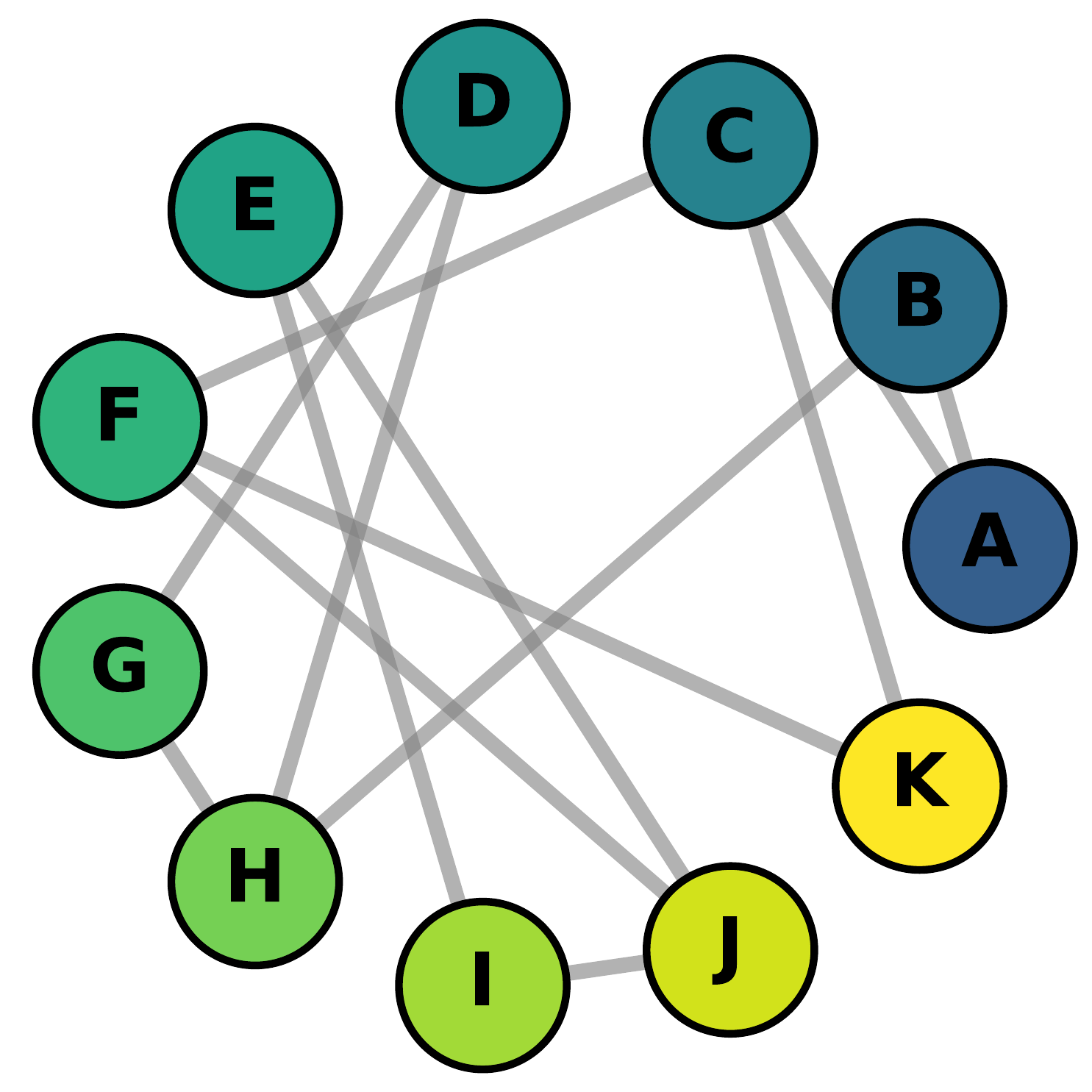}
    }
    \subfigure[GPT-3.5]{
        \includegraphics[width=0.23\linewidth]{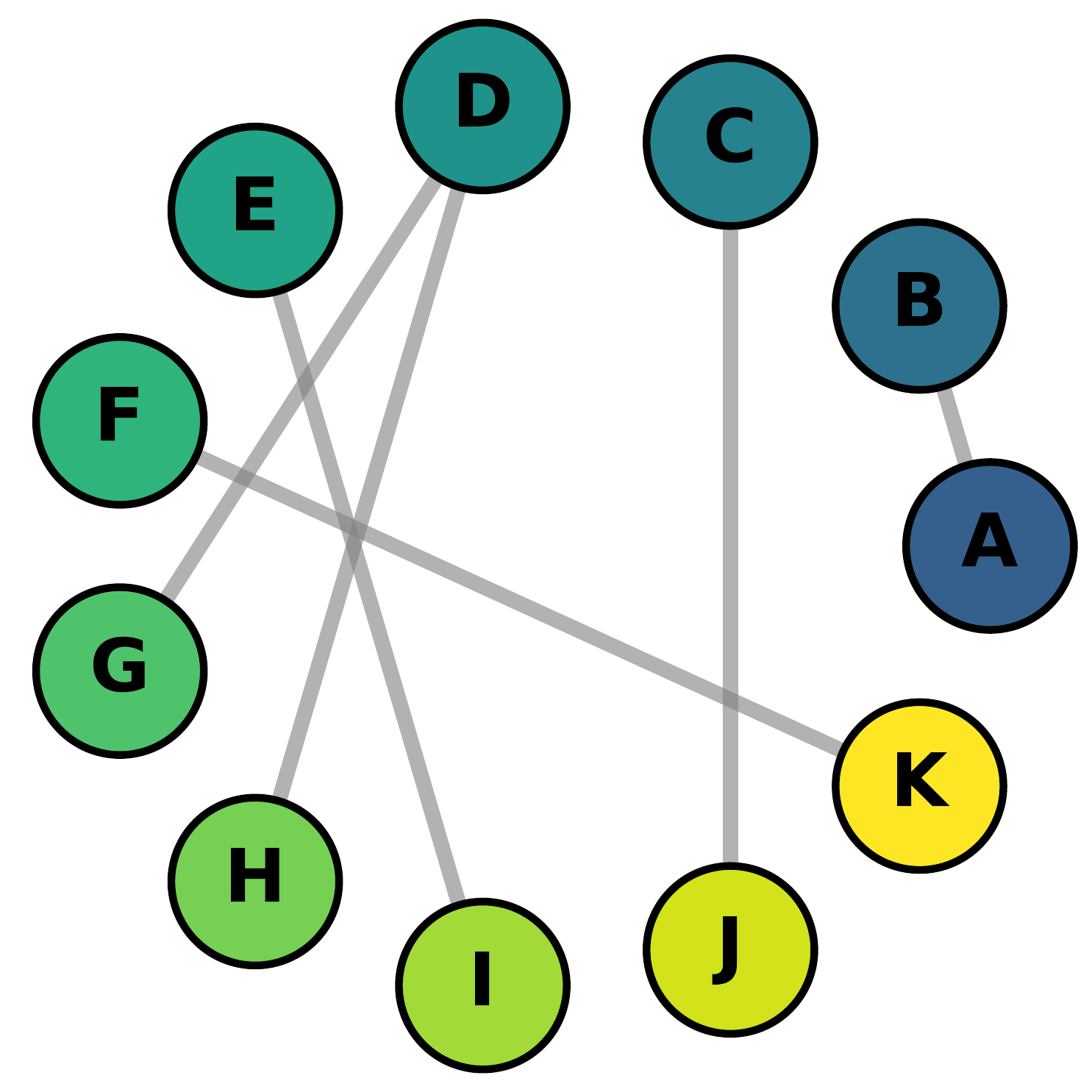}
    }
    \subfigure[GPT-4]{
        \includegraphics[width=0.23\linewidth]{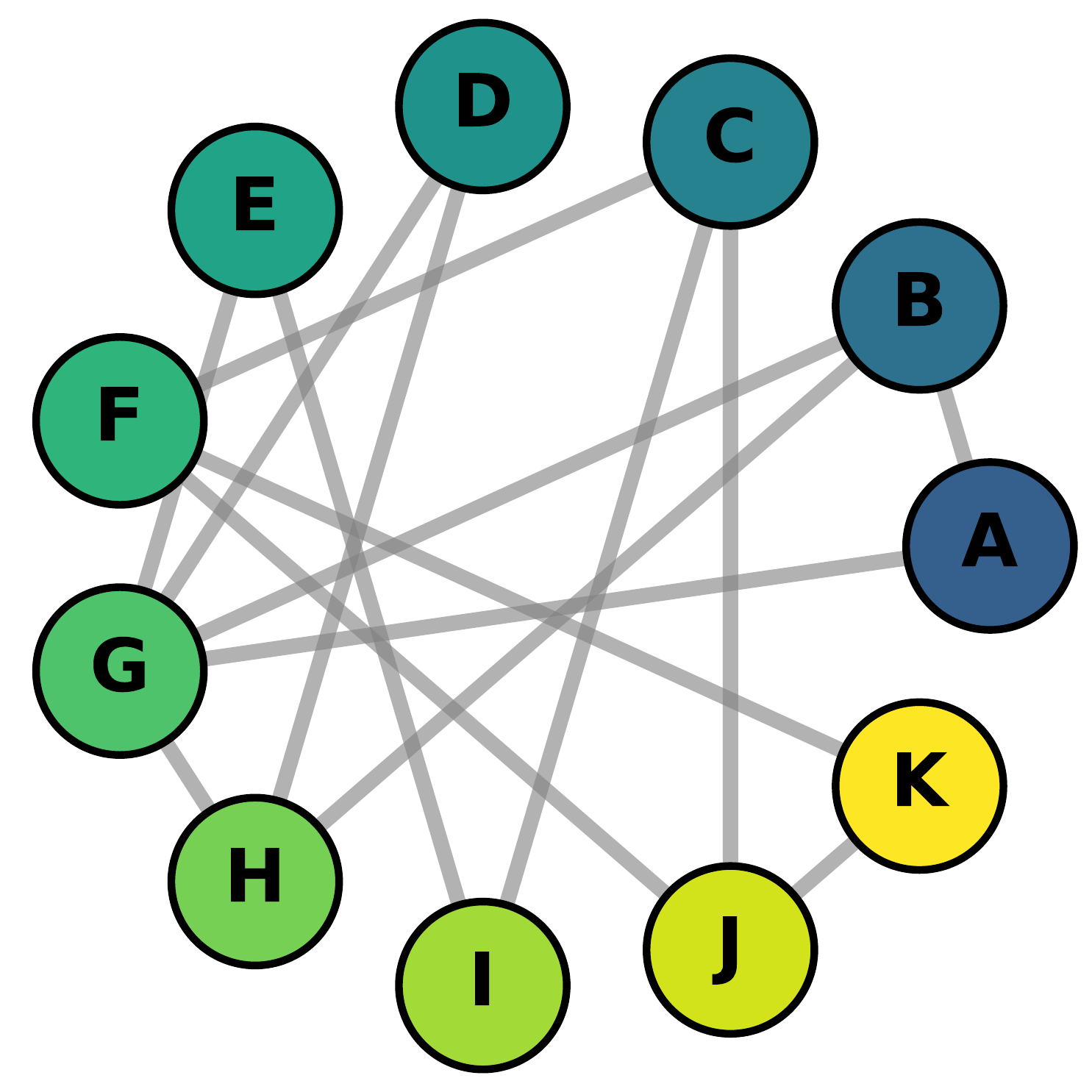}
    }
    \caption{Sea.}
    \label{fig:appx-sea}
\end{figure*}

\begin{figure*}[th]
    \centering
    \subfigure[Ground Truth]{
        \includegraphics[width=0.23\linewidth]{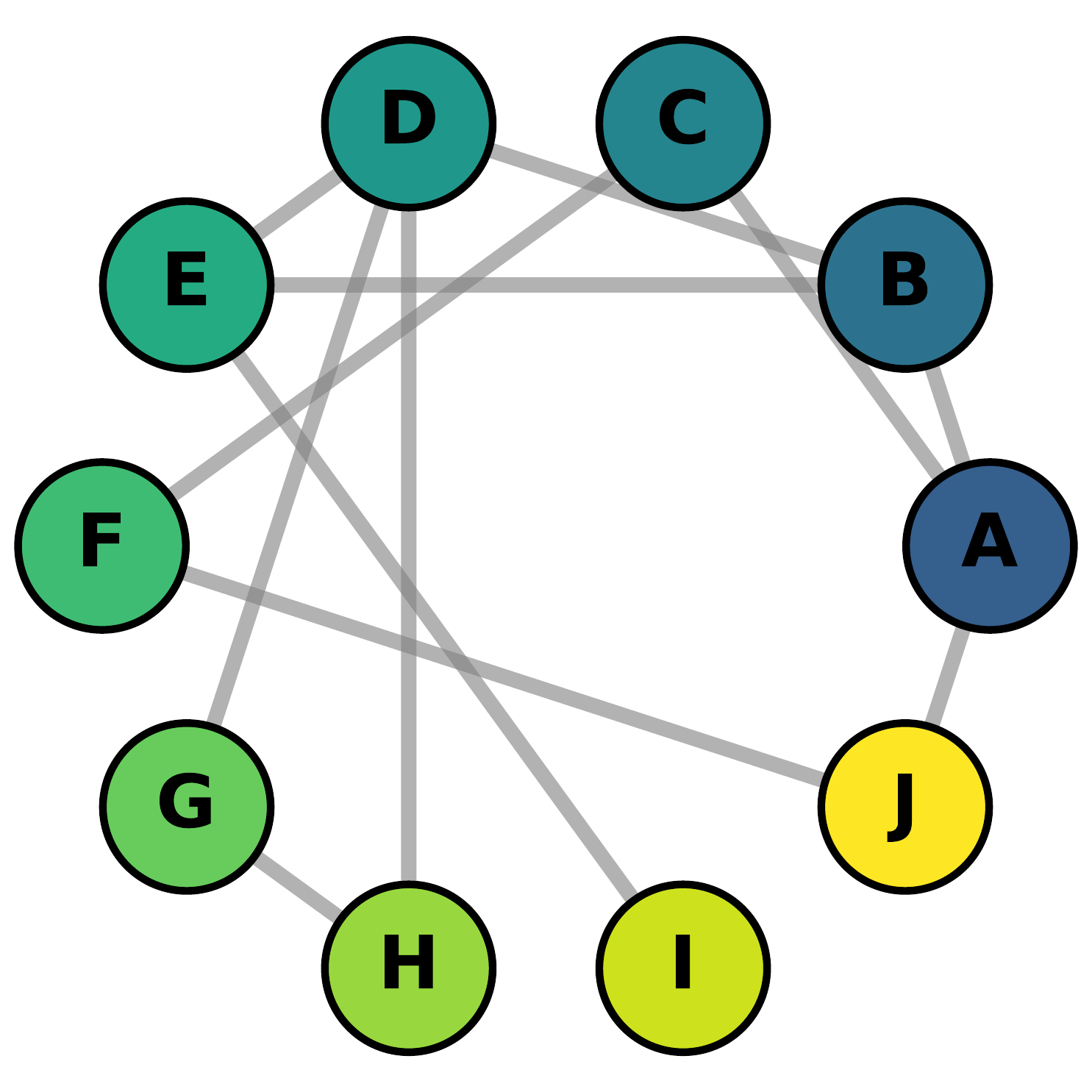}
    }
    \subfigure[LLAMA-2-70B]{
        \includegraphics[width=0.23\linewidth]{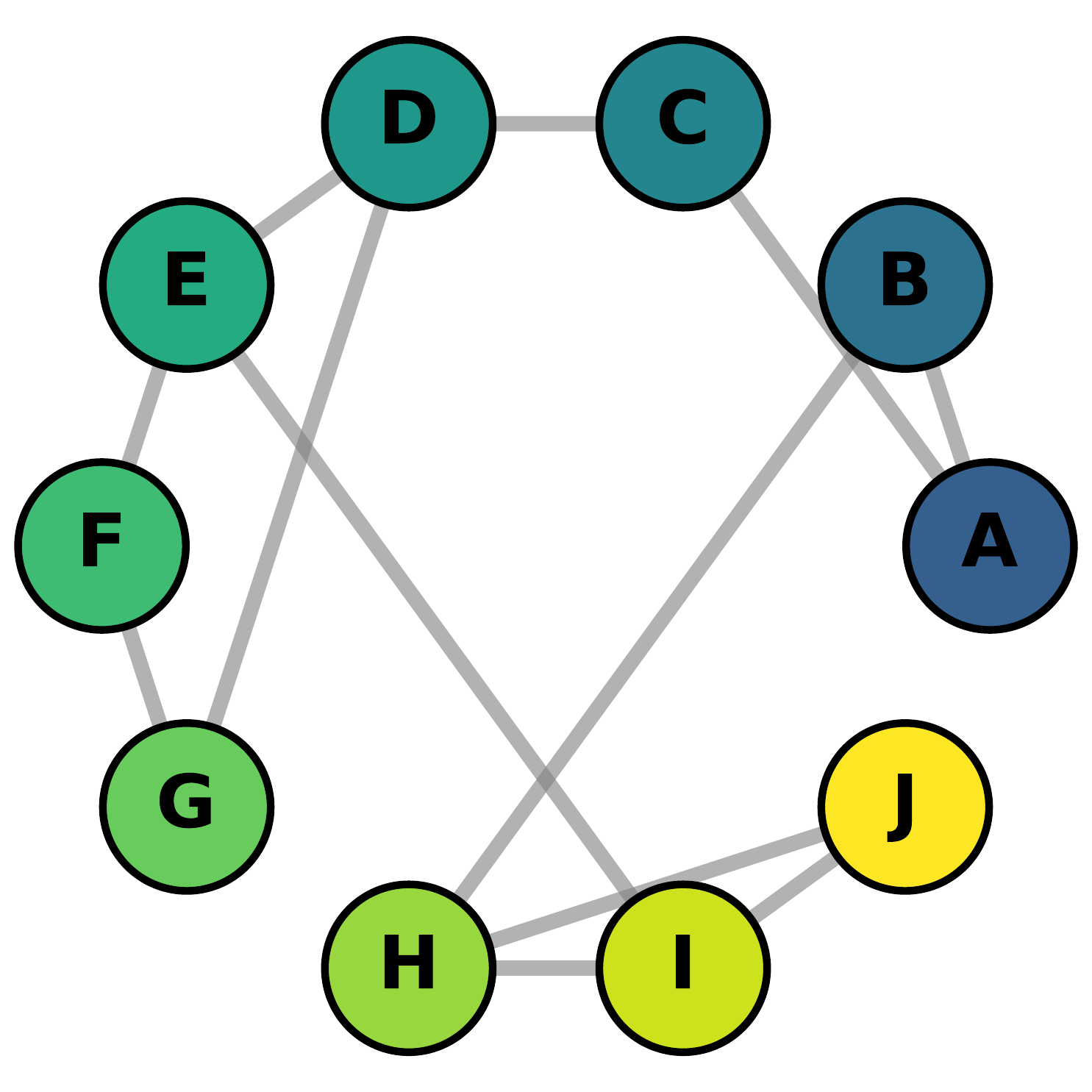}
    }
    \subfigure[GPT-3.5]{
        \includegraphics[width=0.23\linewidth]{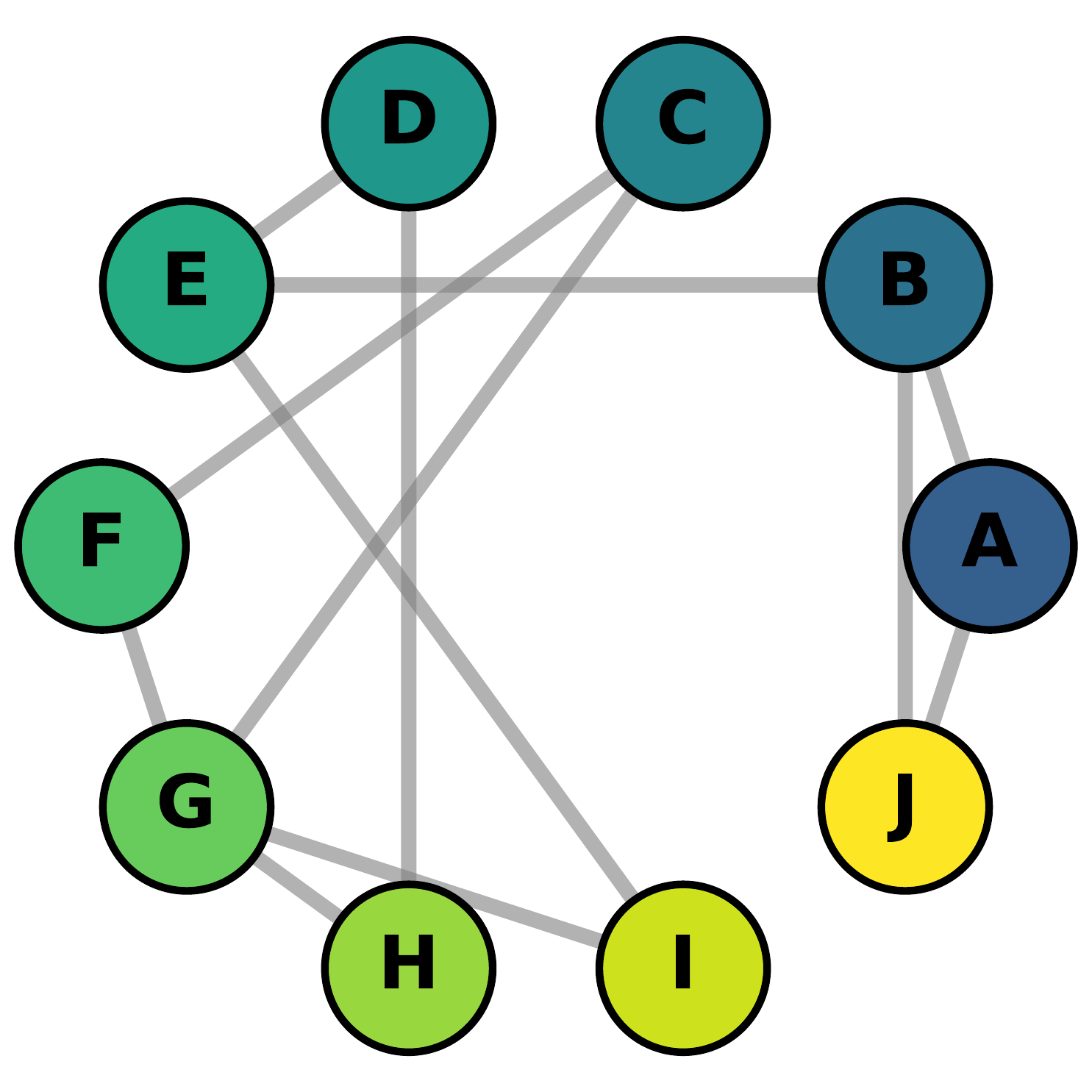}
    }
    \subfigure[GPT-4]{
        \includegraphics[width=0.23\linewidth]{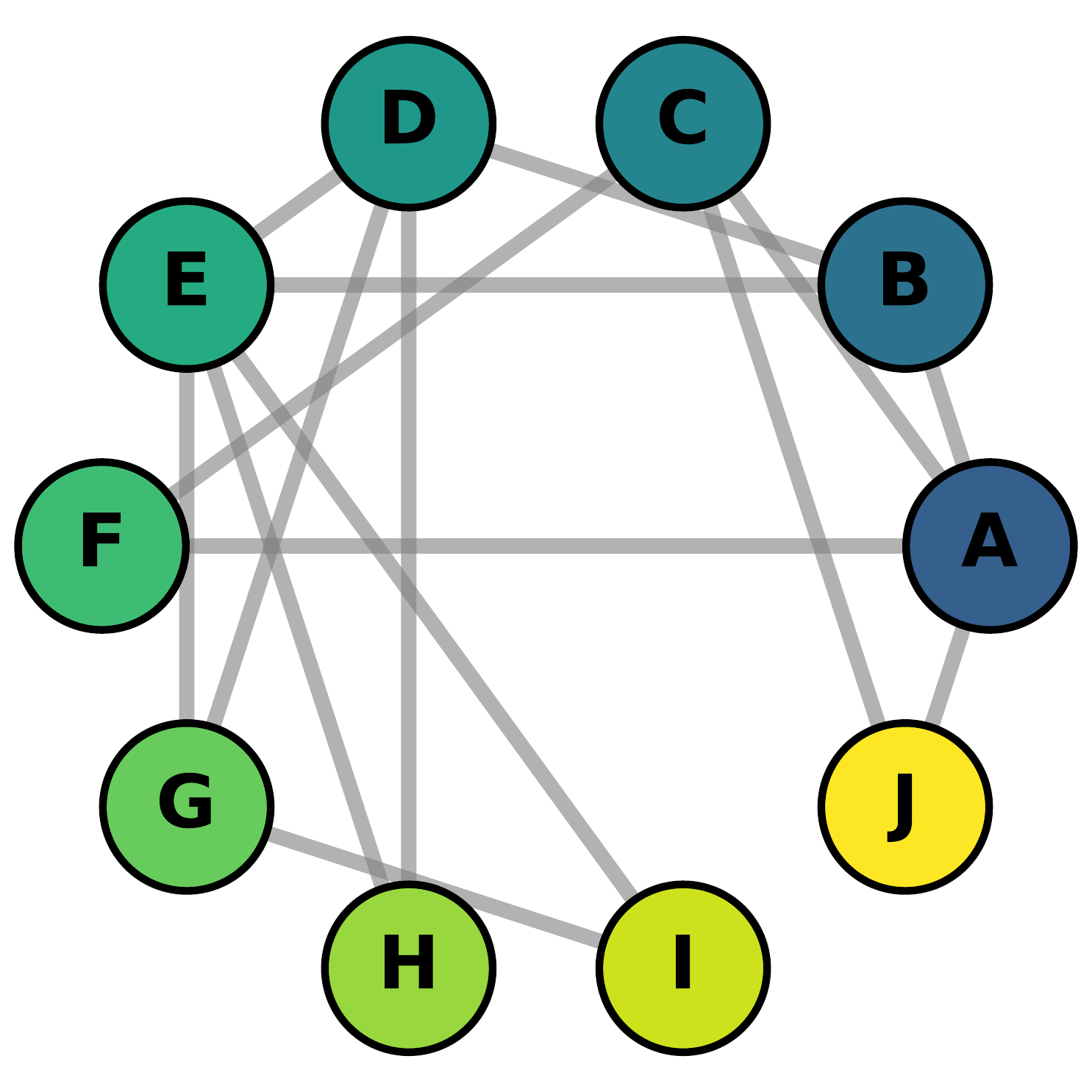}
    }
    \caption{Table.}
    \label{fig:appx-table}
\end{figure*}

\begin{figure*}[th]
    \centering
    \subfigure[Ground Truth]{
        \includegraphics[width=0.23\linewidth]{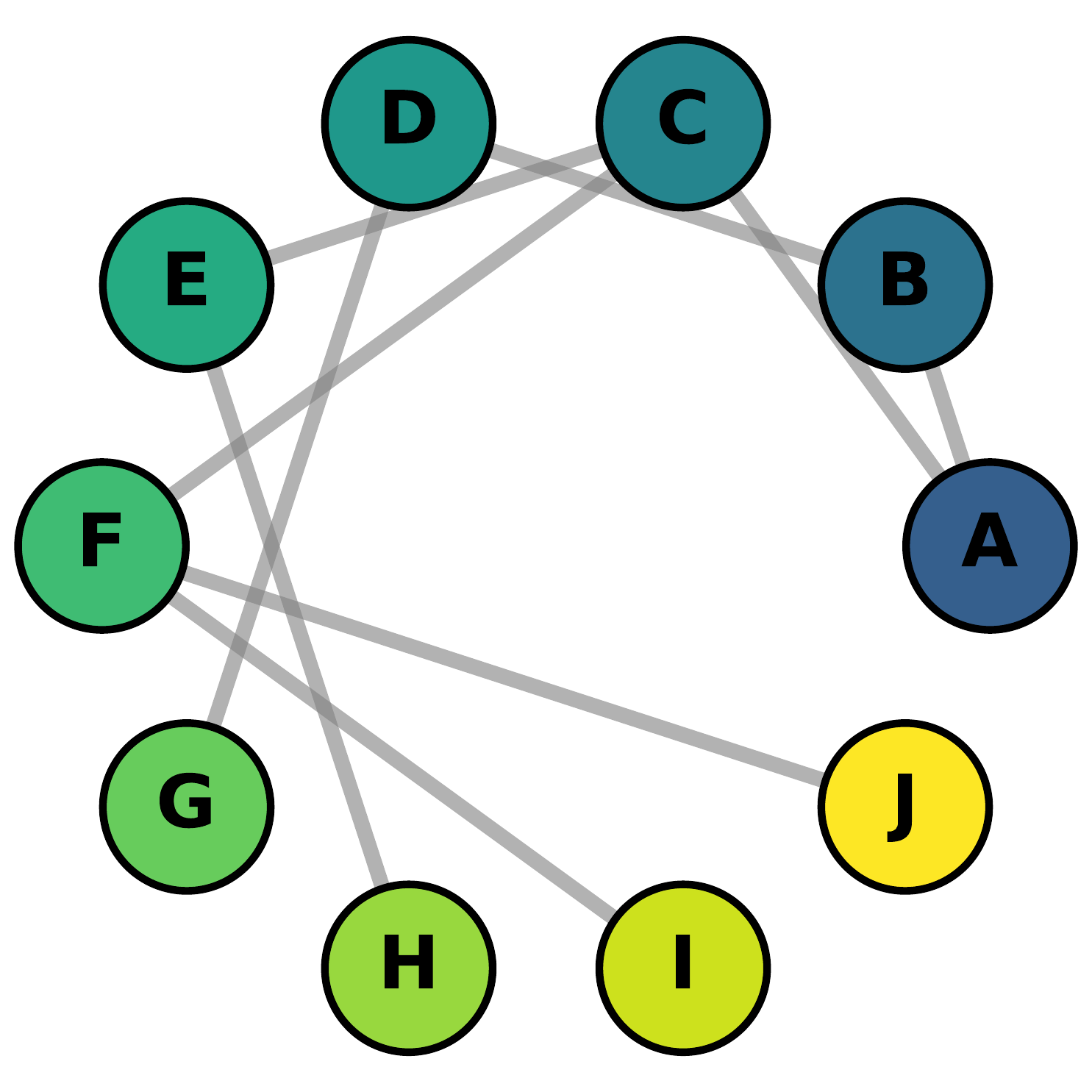}
    }
    \subfigure[LLAMA-2-70B]{
        \includegraphics[width=0.23\linewidth]{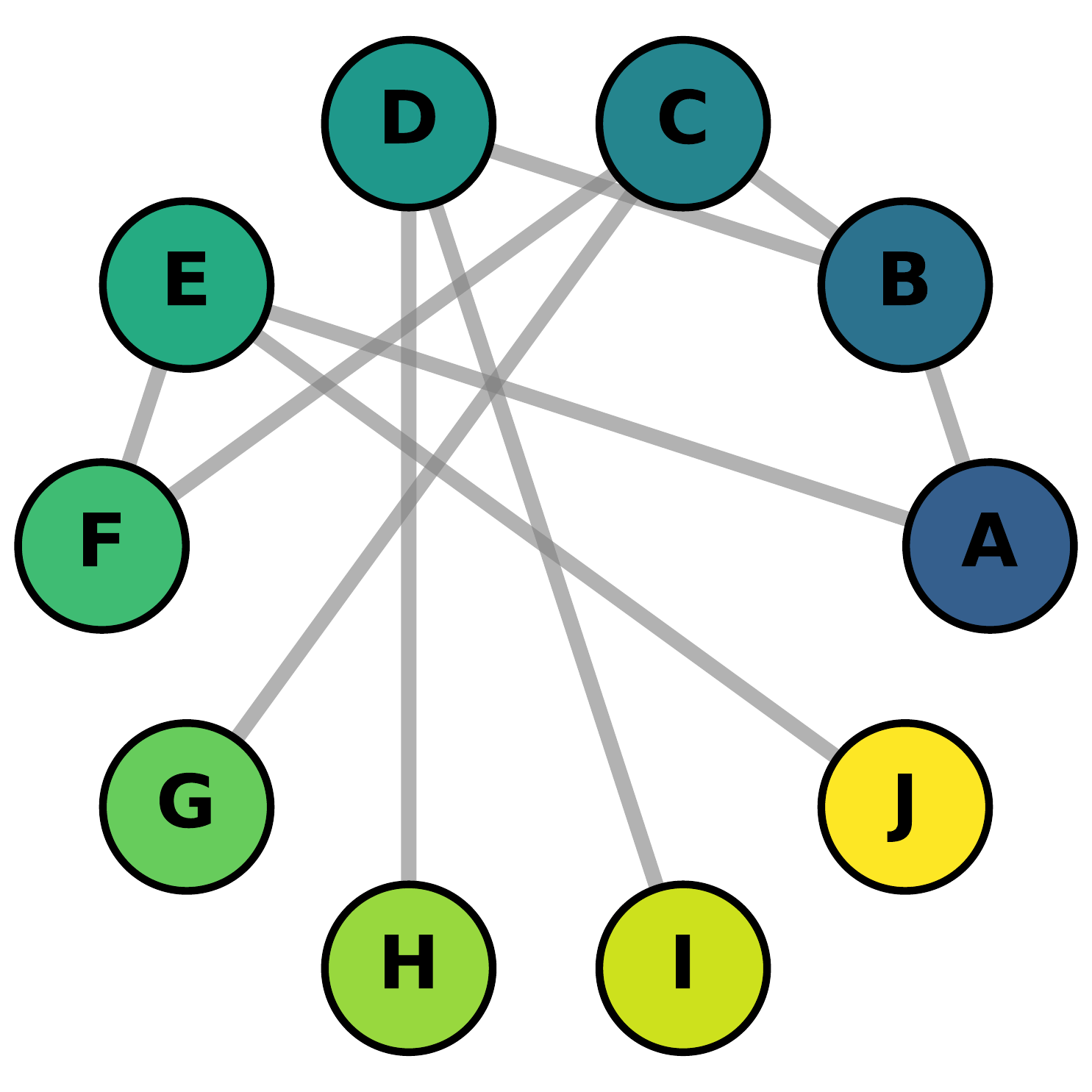}
    }
    \subfigure[GPT-3.5]{
        \includegraphics[width=0.23\linewidth]{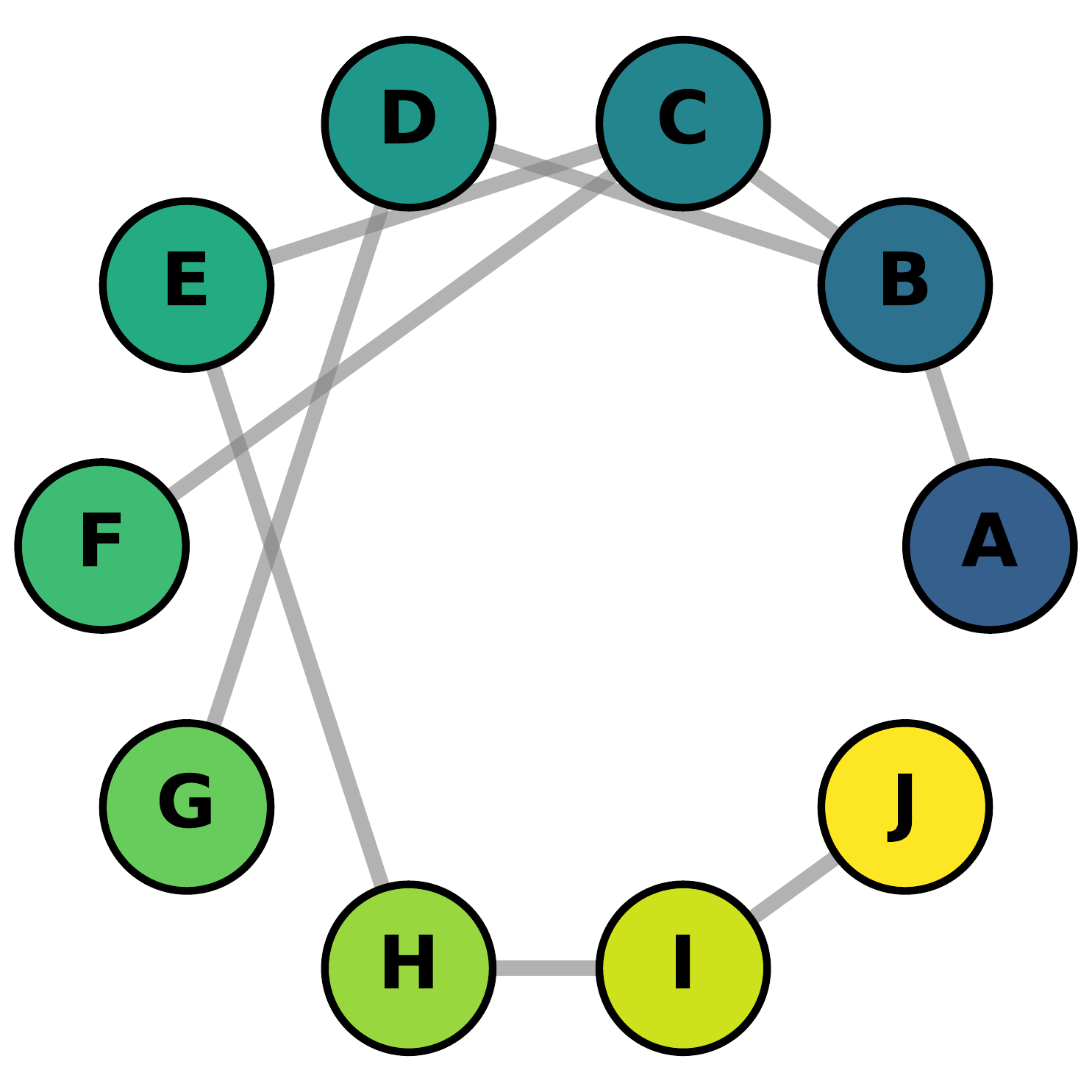}
    }
    \subfigure[GPT-4]{
        \includegraphics[width=0.23\linewidth]{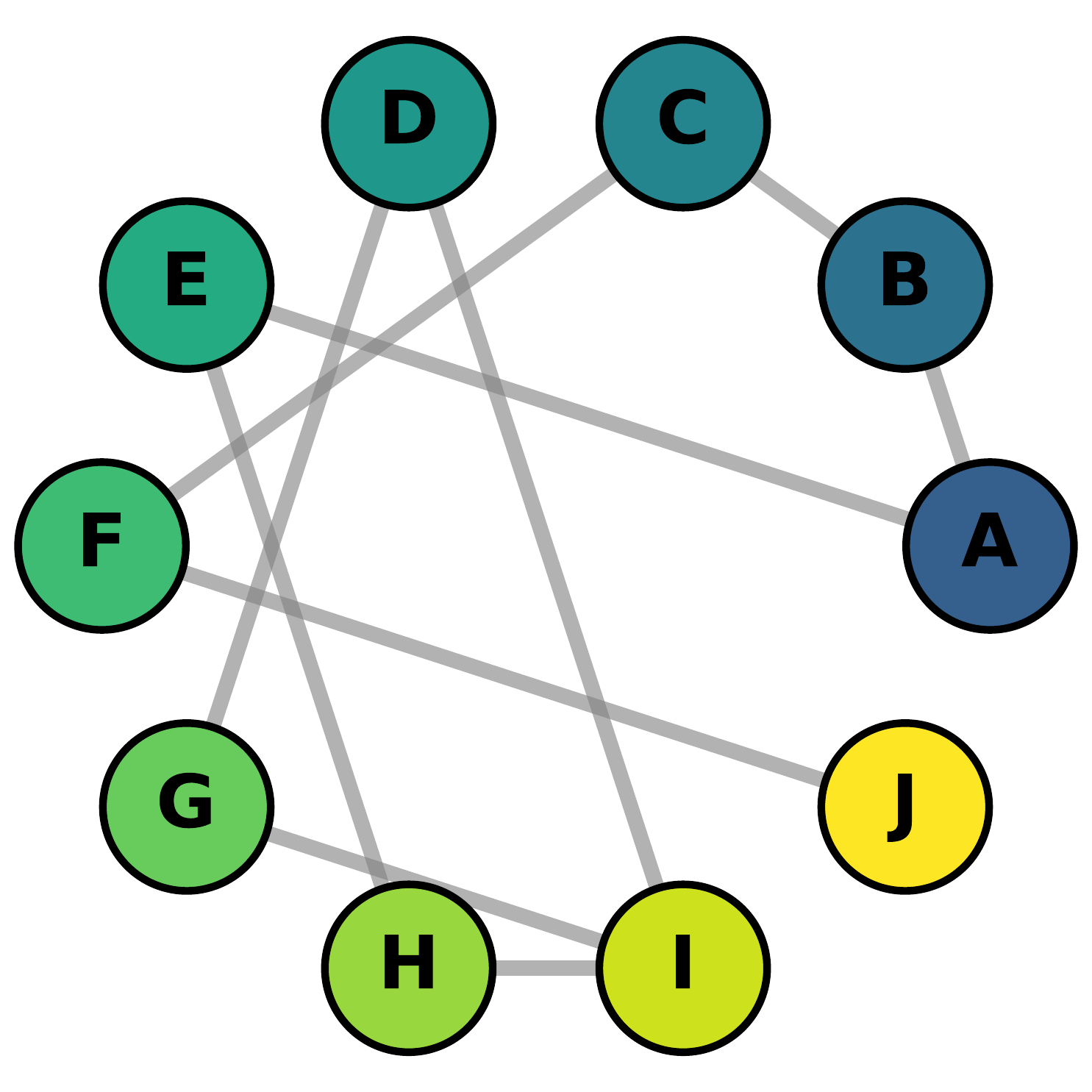}
    }
    \caption{Zoo.}
    \label{fig:appx-zoo}
\end{figure*}

\section{Additional Experiments}

\subsection{Additional Synthetic Relational Learning}

To show the applicability of our method to general graphs and its scalability to graphs of different magnitudes
we conduct additional synthetic relational learning in random graphs and subgraphs extracted from ConceptNet.

\subsubsection{Random Graphs}
We address synthetic relational learning tasks in random graphs of varying node counts. Specifically, we generate weighted connected random graphs (WCGNM) with  $n$ nodes and $m(n)=\frac{p n(n-1)}{2}$ edges, selected uniformly at random. We vary $n$ across five different magnitudes: $10, 20, 50, 100, and 200$, maintaining parameters $p=0.2$ and $\kappa=3.0$ for each. Each experimental setting is repeated $5$ times. The results are presented in \cref{fig:appx-srl-rg}.

\begin{figure*}[th]
    \centering
    \includegraphics[width=\textwidth]{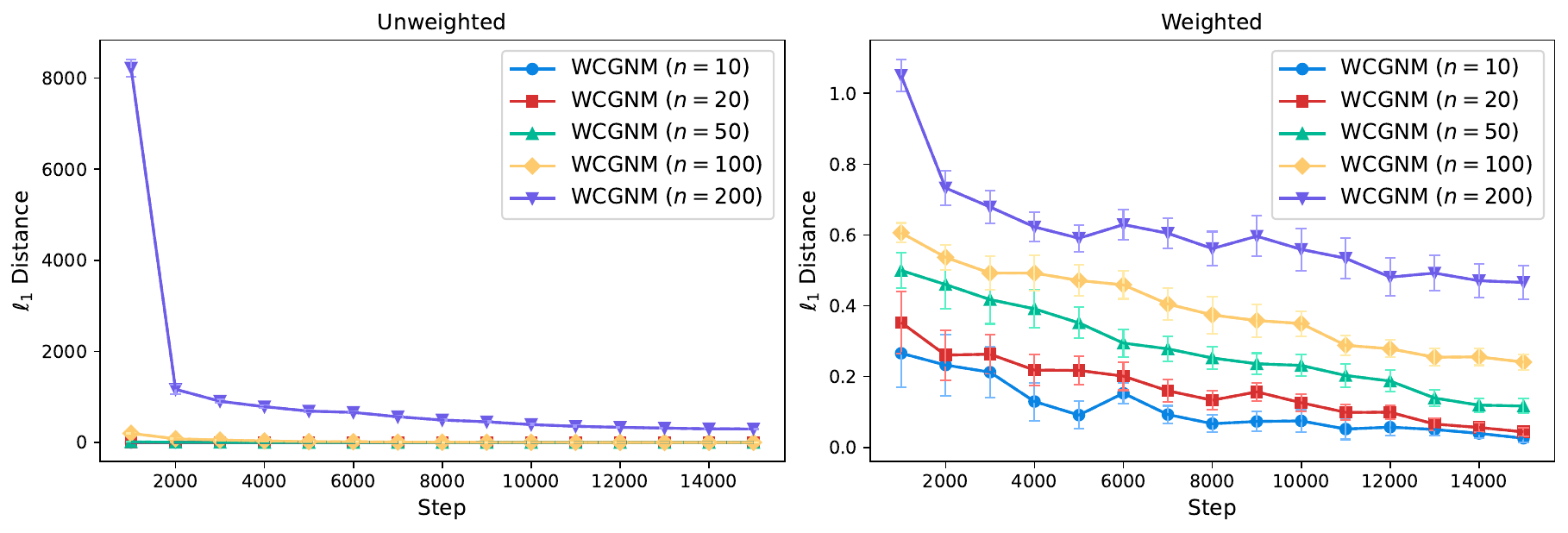}
    \caption{Synthetic relation learning in WCGNMs with different magnitudes.}
    \label{fig:appx-srl-rg}
\end{figure*}

\subsubsection{Subgraphs Extracted from ConceptNet}
We conduct synthetic relational learning tasks using relational graphs derived from ConceptNet, which represent more intricate real-world relational structures. These subgraphs are generated similarly to the real-world relation evaluation experiments in \cref{subsec:experiments-rwre} but include additional top-related pairs for each entity to increase complexity. Specifically, we focus on the three most related pairs of each entity. Each resulting subgraph comprises approximately 50 nodes, making them more complex than the specific structured graphs used in the experiments of \cref{subsec:experiments-srl}. The results are shown in \cref{fig:appx-srl-conceptnet}.

\begin{figure*}[th]
    \centering
    \includegraphics[width=\textwidth]{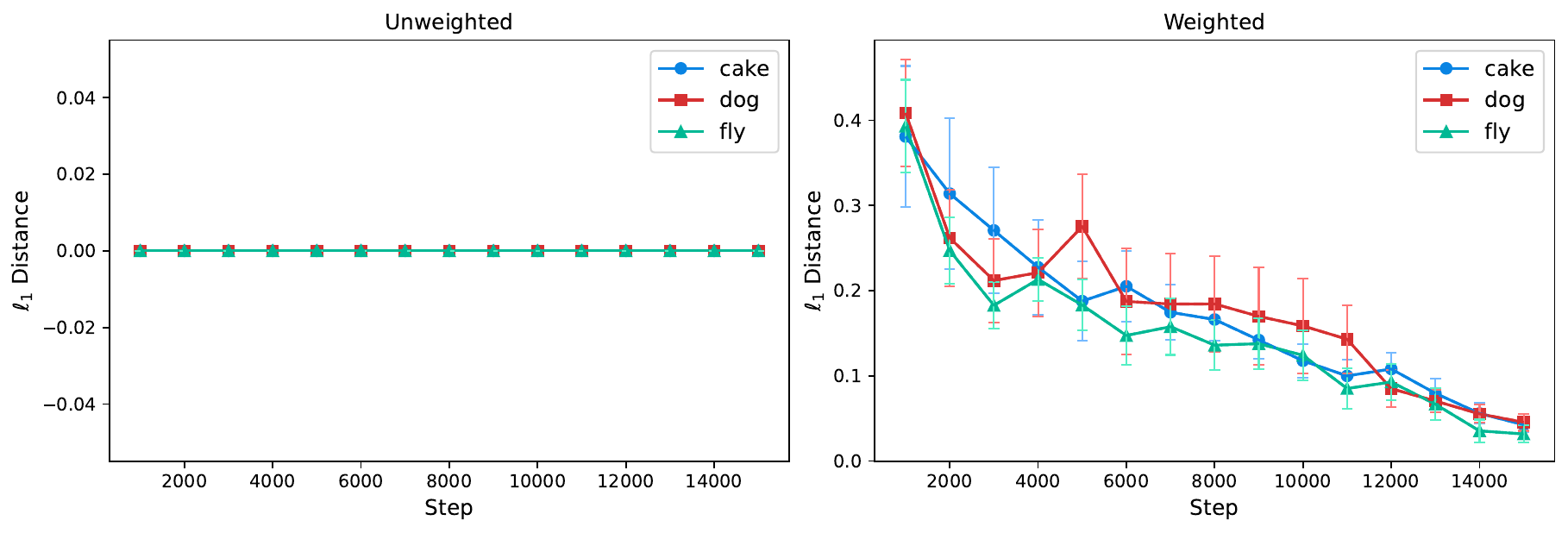}
    \caption{Synthetic relation learning in the subgraphs extracted from ConceptNet.}
    \label{fig:appx-srl-conceptnet}
\end{figure*}

\subsection{Additional Real-World Relation Evaluation}

\subsubsection{Relation Evaluation in WordNet}
We perform similar experiments to that in \cref{subappendix:rwre} in WordNet \citep{miller1995wordnet} to show that our method can be applied to relational learning scenarios beyond ConceptNet. Figures~\ref{fig:appx-rwre-wordnet-cake} -~\ref{fig:appx-rwre-wordnet-zoo} are the evaluation results, which are summarized in \cref{tab:appx-rwre-wordnet}. The correspondences between the entities and the letters used in the figures are summarized in Tables~\ref{tab:wordnet-entity-letter-1} and~\ref{tab:wordnet-entity-letter-2}.

\begin{table*}[th]
\caption{The correspondences between the entities and the letters for WordNet (Part 1).}
\label{tab:wordnet-entity-letter-1}
\begin{center}
\begin{sc}
\begin{tabular}{lcccccc}
\toprule
 & A & B & C & D & E & F  \\
\midrule
cake & cake & bar & patty & barroom & saloon & dish \\
dog & dog & frump & cad & bounder & blackguard & - \\
fly & fly & tent-fly & rainfly & - & - & - \\
paper & paper & newspaper & composition & newsprint & composing & constitution \\
zoo & zoo & menagerie & facility & collection & installation & adeptness \\
\bottomrule
\end{tabular}
\end{sc}
\end{center}
\end{table*}

\begin{table*}[th]
\caption{The correspondences between the entities and the letters for WordNet (Part 2).}
\label{tab:wordnet-entity-letter-2}
\begin{center}
\begin{sc}
\begin{tabular}{lcccccc}
\toprule
 & G & H & I & J & K & L \\
\midrule
cake & dishful & smasher & - & - & - & - \\
dog & - & - & - & - & - & - \\
fly & - & - & - & - & - & - \\
paper & placement & establishment & formation & - & - & - \\
zoo & aggregation & accumulation & installing & installment & adroitness & deftness \\
\bottomrule
\end{tabular}
\end{sc}
\end{center}
\end{table*}

\begin{table*}[th]
\caption{Some relation evaluation results in WordNet. Similarly, the subgraphs are generated from different source entities with $k=2$ and $d=3$. The dissimilarity measure is the same as that in ConceptNet.}
\label{tab:appx-rwre-wordnet}
\begin{center}
\begin{sc}
\begin{tabular}{lcccccccccc}
\toprule
 & cake & dog & fly & paper & zoo \\
\midrule
GPT-3.5     & $1.33$ & $1.00$ & $0.00$ & $0.75$ & $1.33$ \\
GPT-4       & $1.33$ & $1.00$ & $0.00$ & $1.00$ & $1.00$ \\
\bottomrule
\end{tabular}
\end{sc}
\end{center}
\end{table*}

\begin{figure*}[th]
    \centering
    \subfigure[Ground Truth]{
        \includegraphics[width=0.3\linewidth]{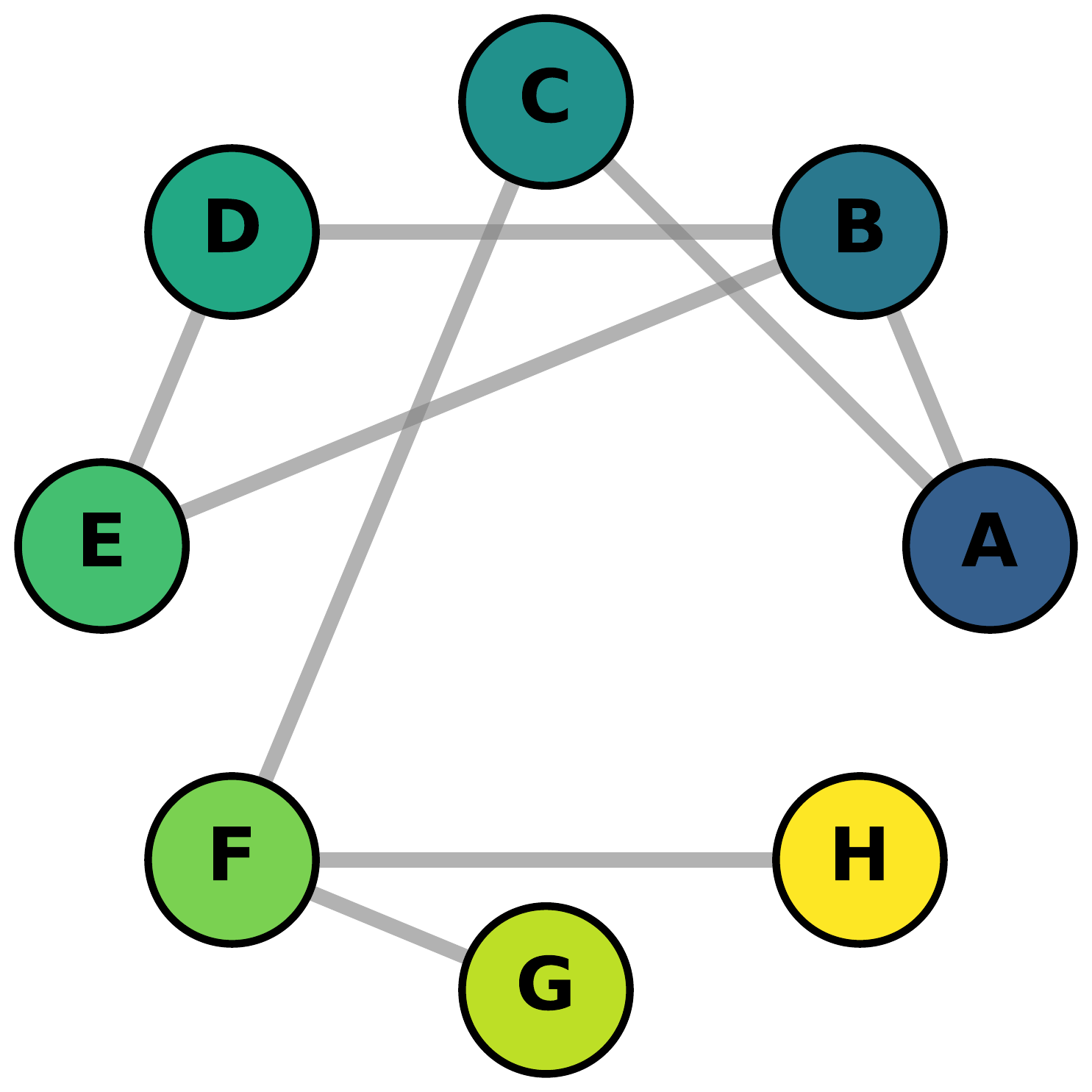}
    }
    \hfill
    \subfigure[GPT-3.5]{
        \includegraphics[width=0.3\linewidth]{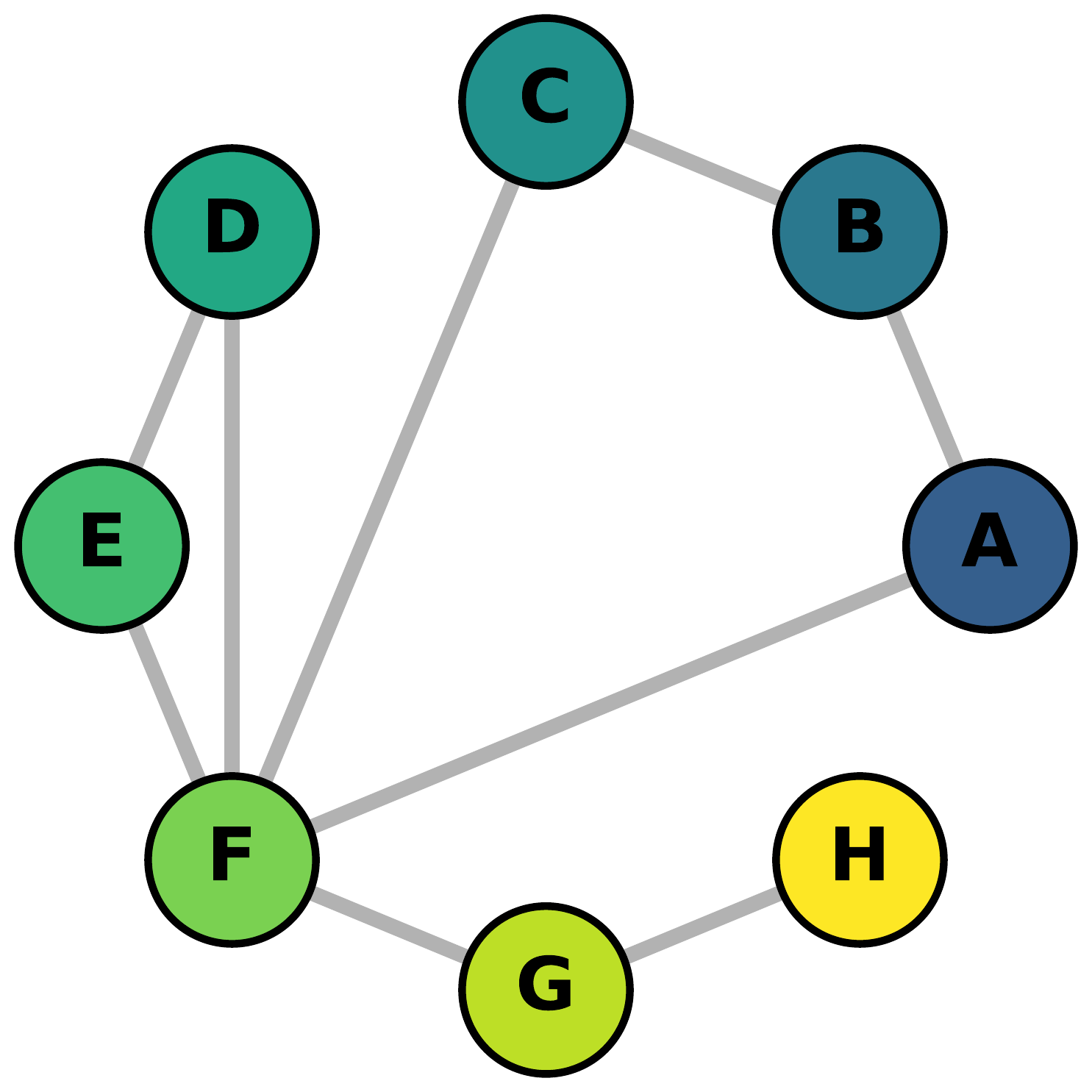}
    }
    \hfill
    \subfigure[GPT-4]{
        \includegraphics[width=0.3\linewidth]{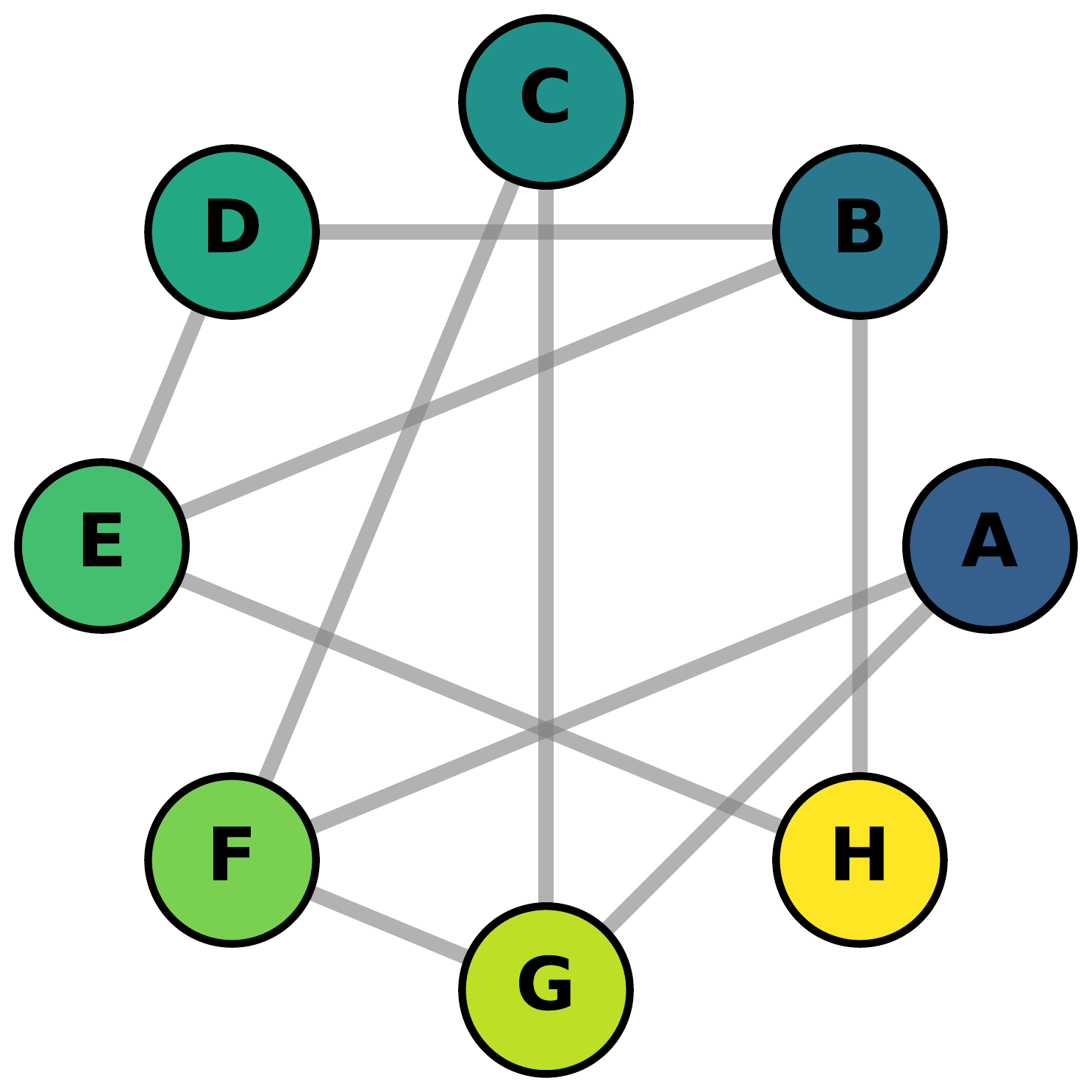}
    }
    \caption{Cake (WordNet).}
    \label{fig:appx-rwre-wordnet-cake}
\end{figure*}

\begin{figure*}[th]
    \centering
    \subfigure[Ground Truth]{
        \includegraphics[width=0.3\linewidth]{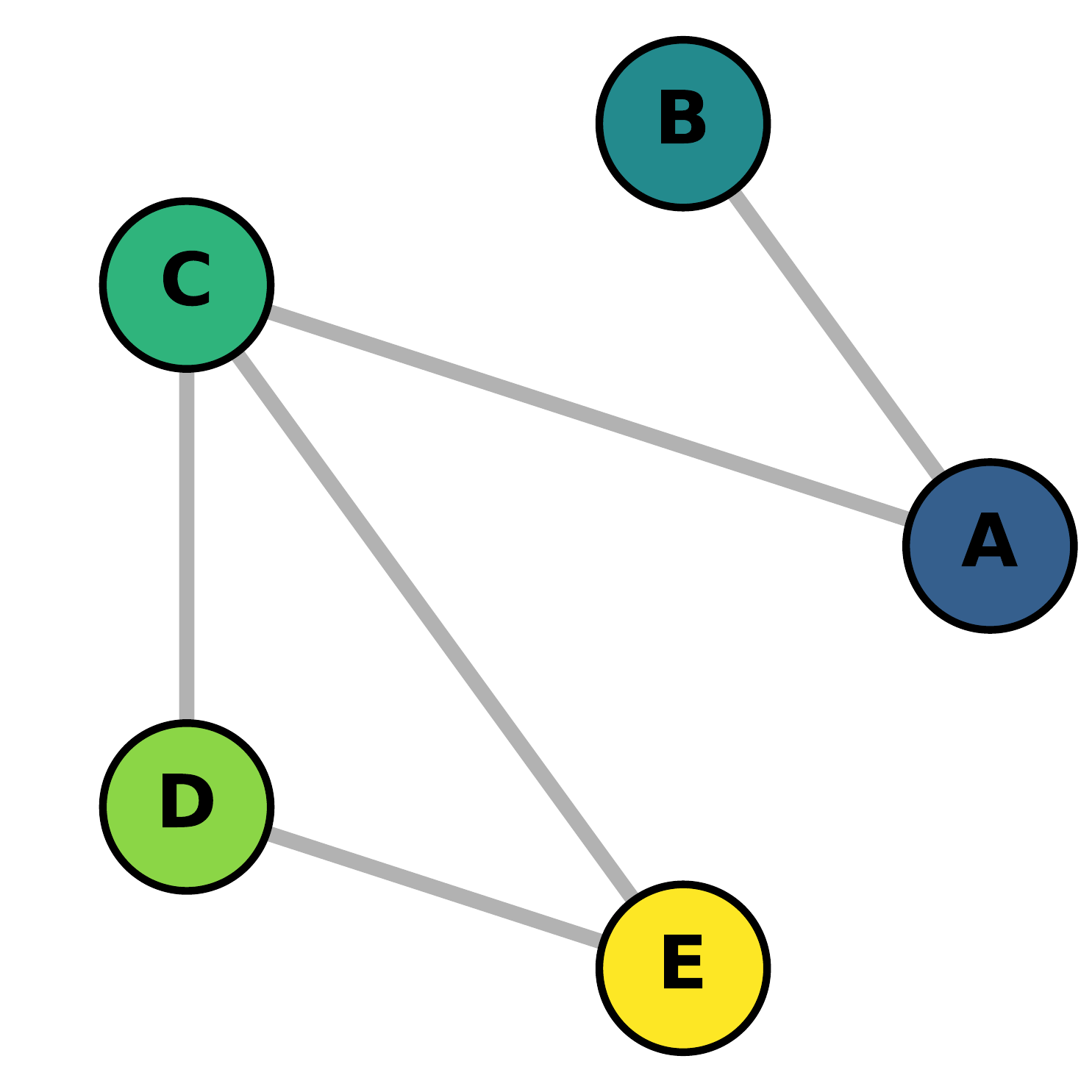}
    }
    \hfill
    \subfigure[GPT-3.5]{
        \includegraphics[width=0.3\linewidth]{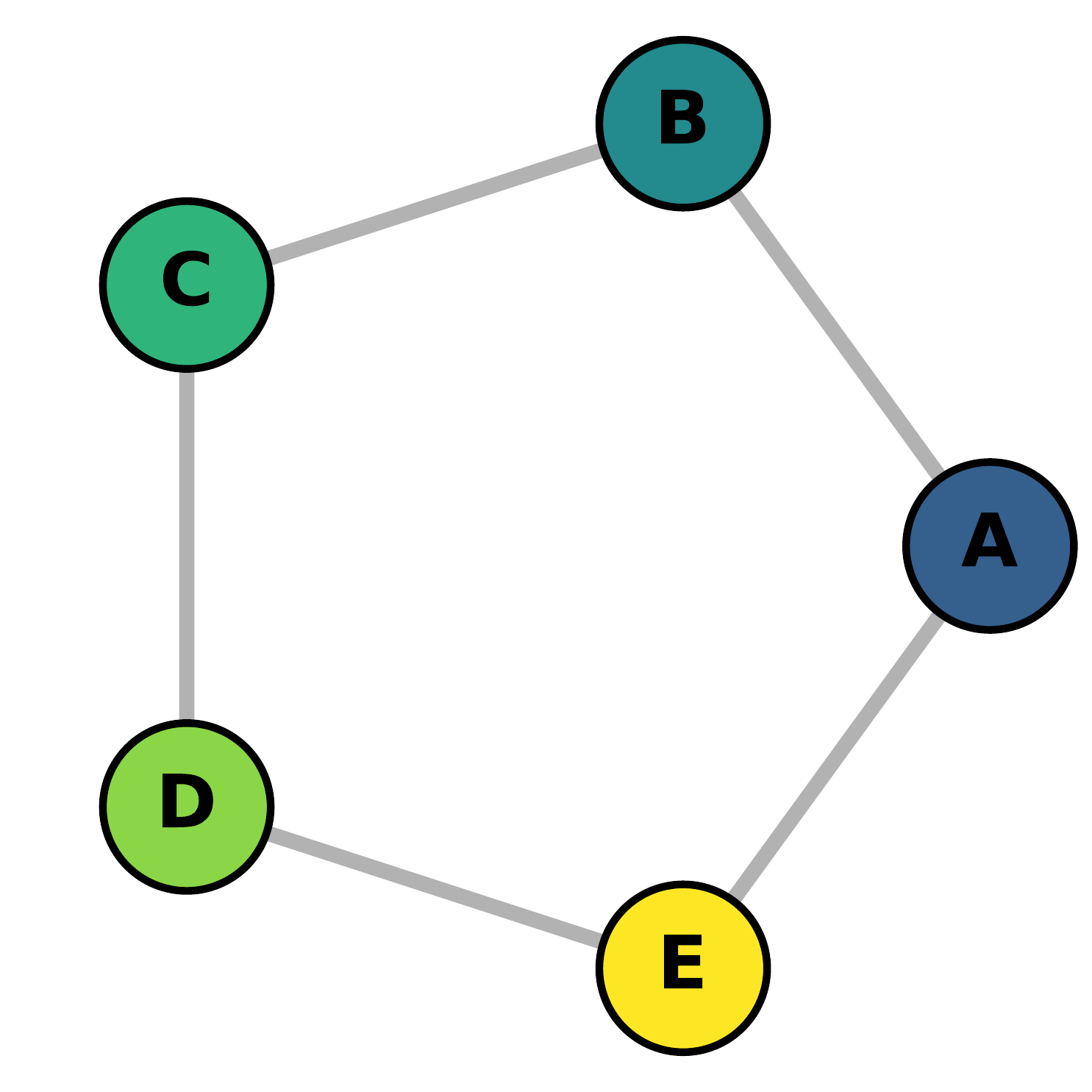}
    }
    \hfill
    \subfigure[GPT-4]{
        \includegraphics[width=0.3\linewidth]{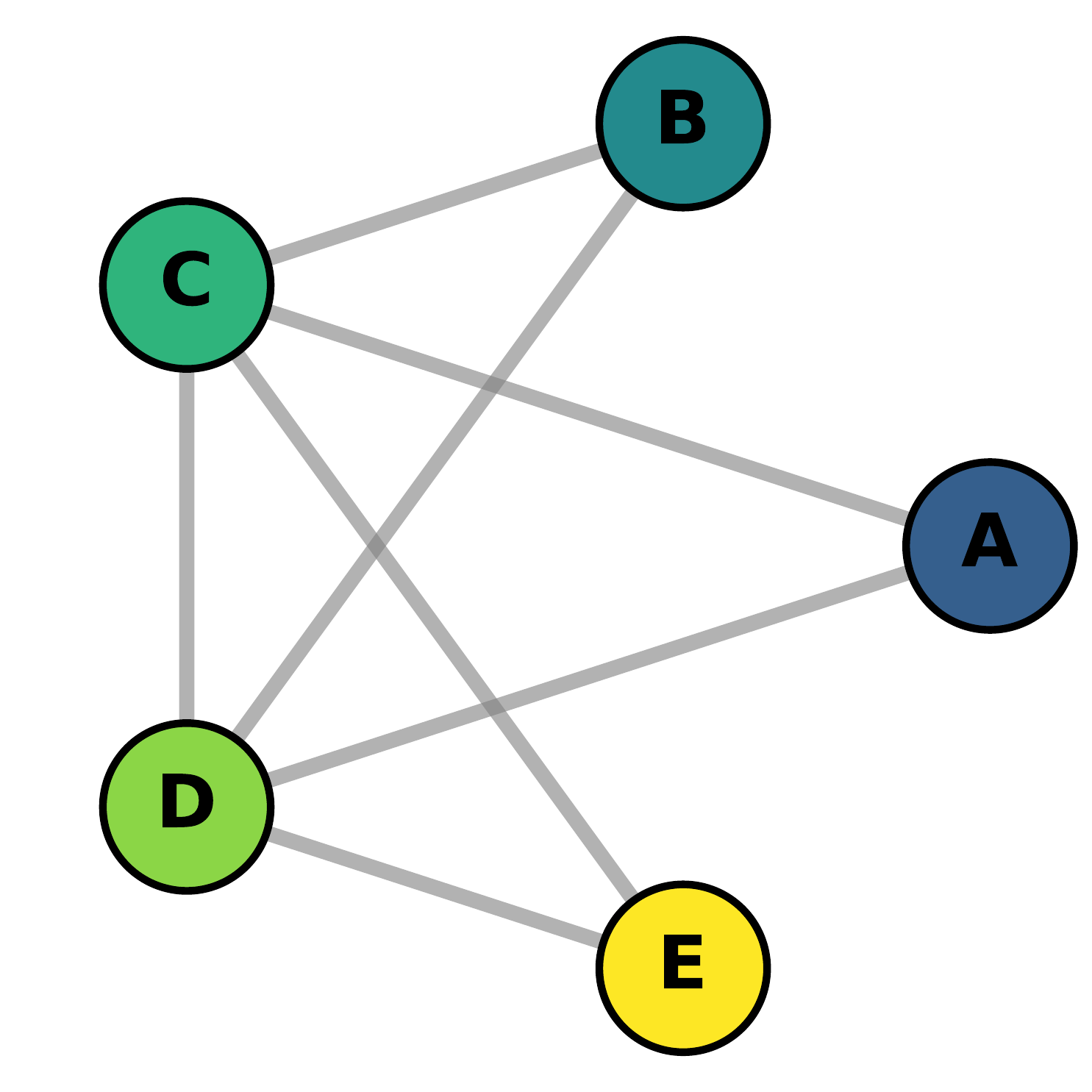}
    }
    \caption{Dog (WordNet).}
    \label{fig:appx-rwre-wordnet-dog}
\end{figure*}

\begin{figure*}[th]
    \centering
    \subfigure[Ground Truth]{
        \includegraphics[width=0.3\linewidth]{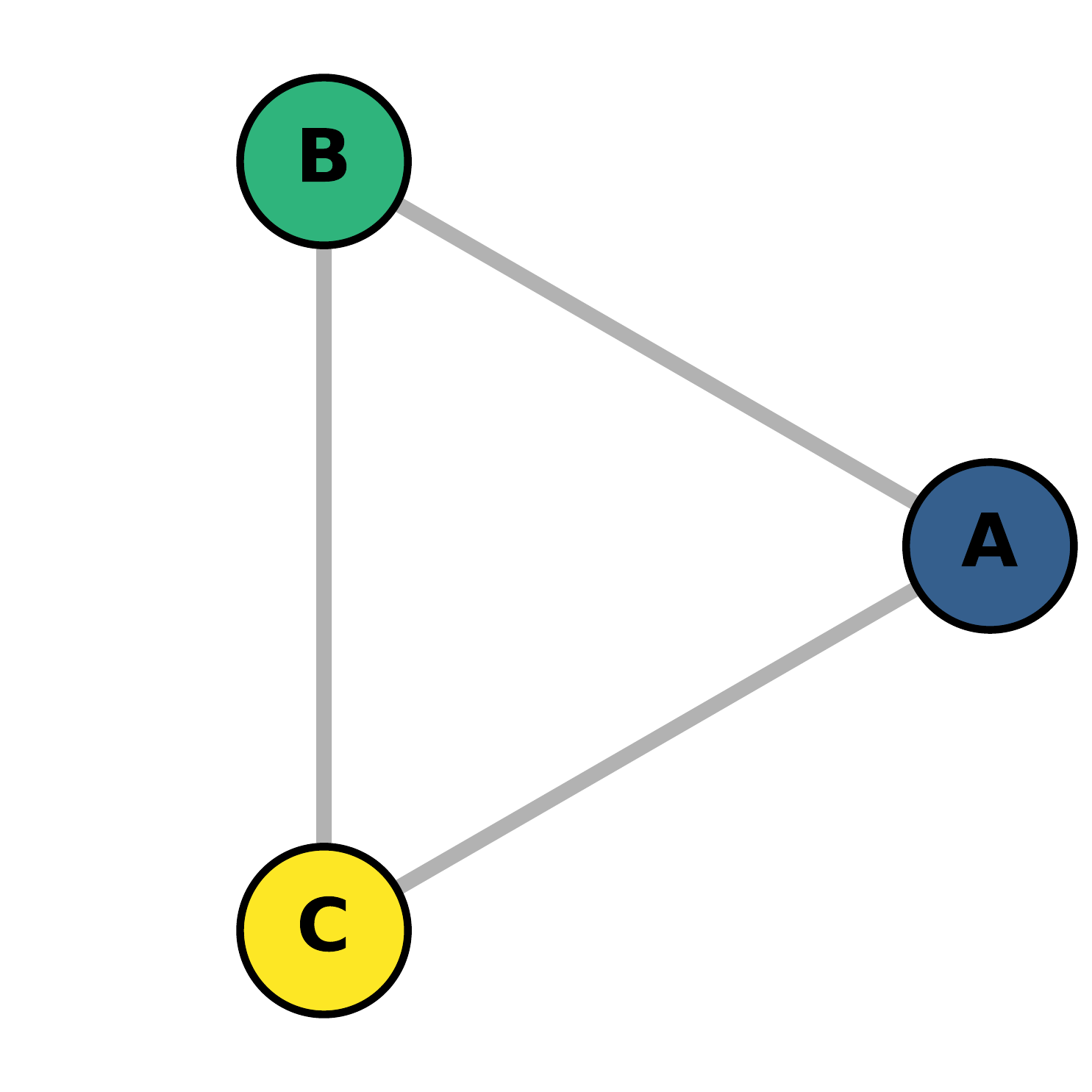}
    }
    \hfill
    \subfigure[GPT-3.5]{
        \includegraphics[width=0.3\linewidth]{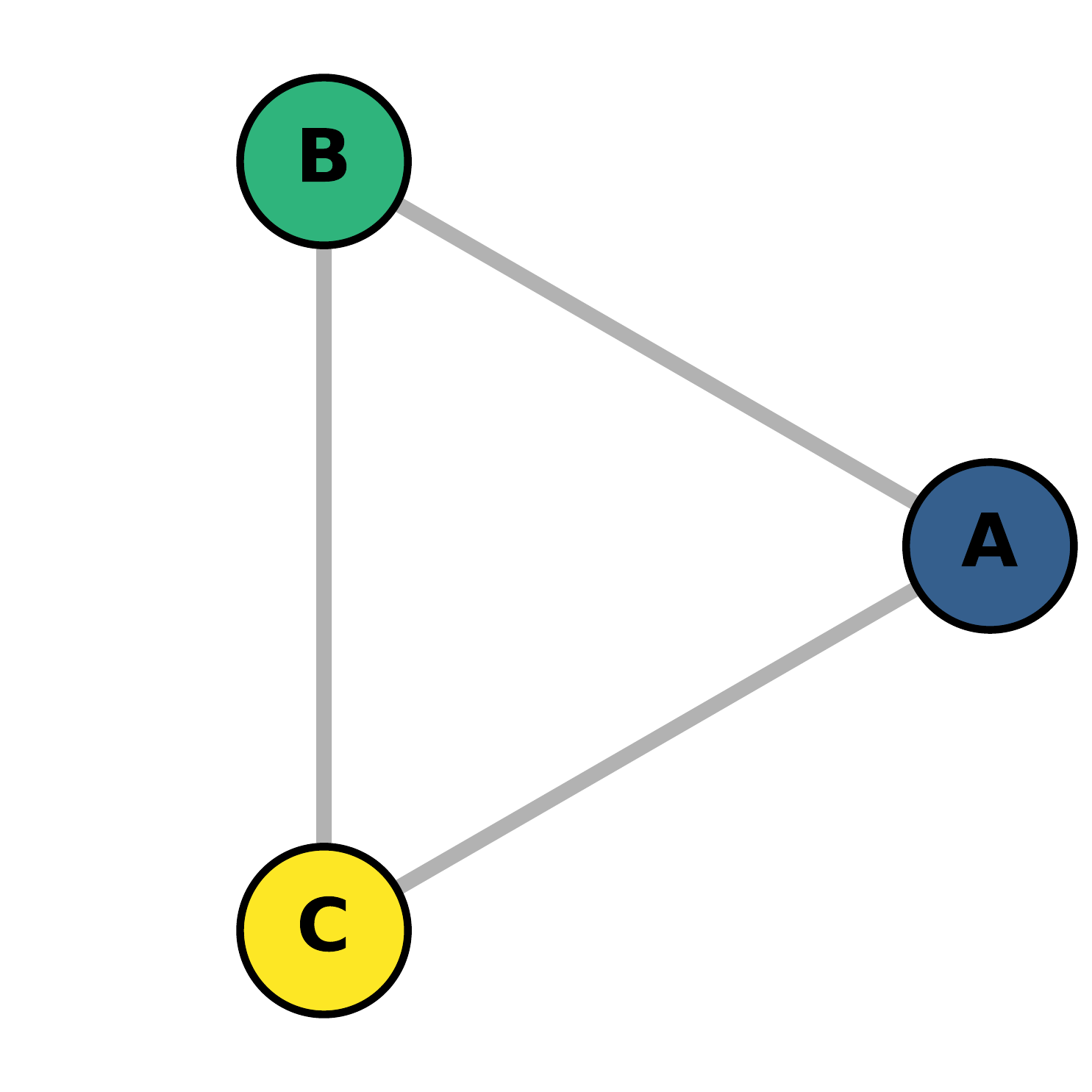}
    }
    \hfill
    \subfigure[GPT-4]{
        \includegraphics[width=0.3\linewidth]{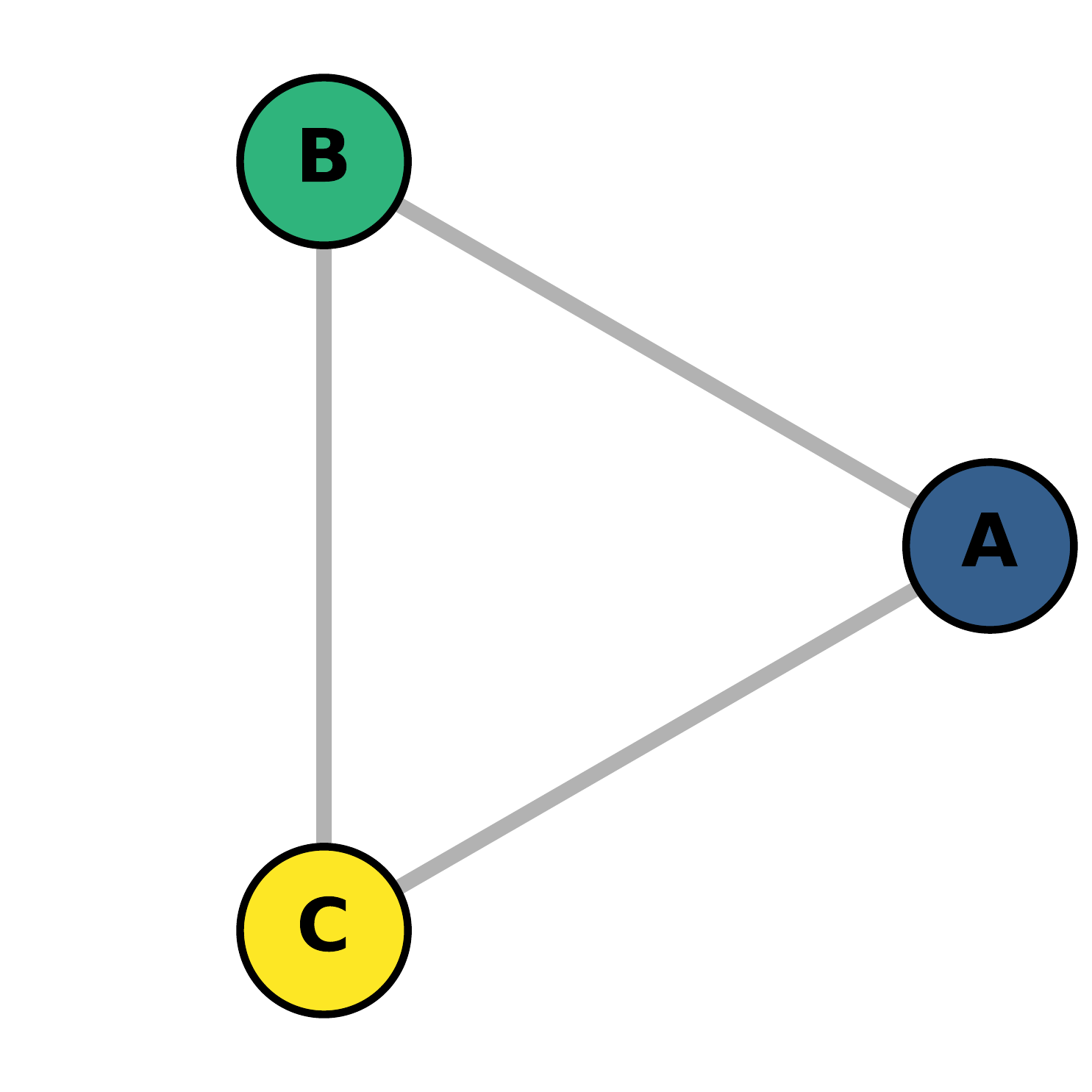}
    }
    \caption{Fly (WordNet).}
    \label{fig:appx-rwre-wordnet-fly}
\end{figure*}

\begin{figure*}[th]
    \centering
    \subfigure[Ground Truth]{
        \includegraphics[width=0.3\linewidth]{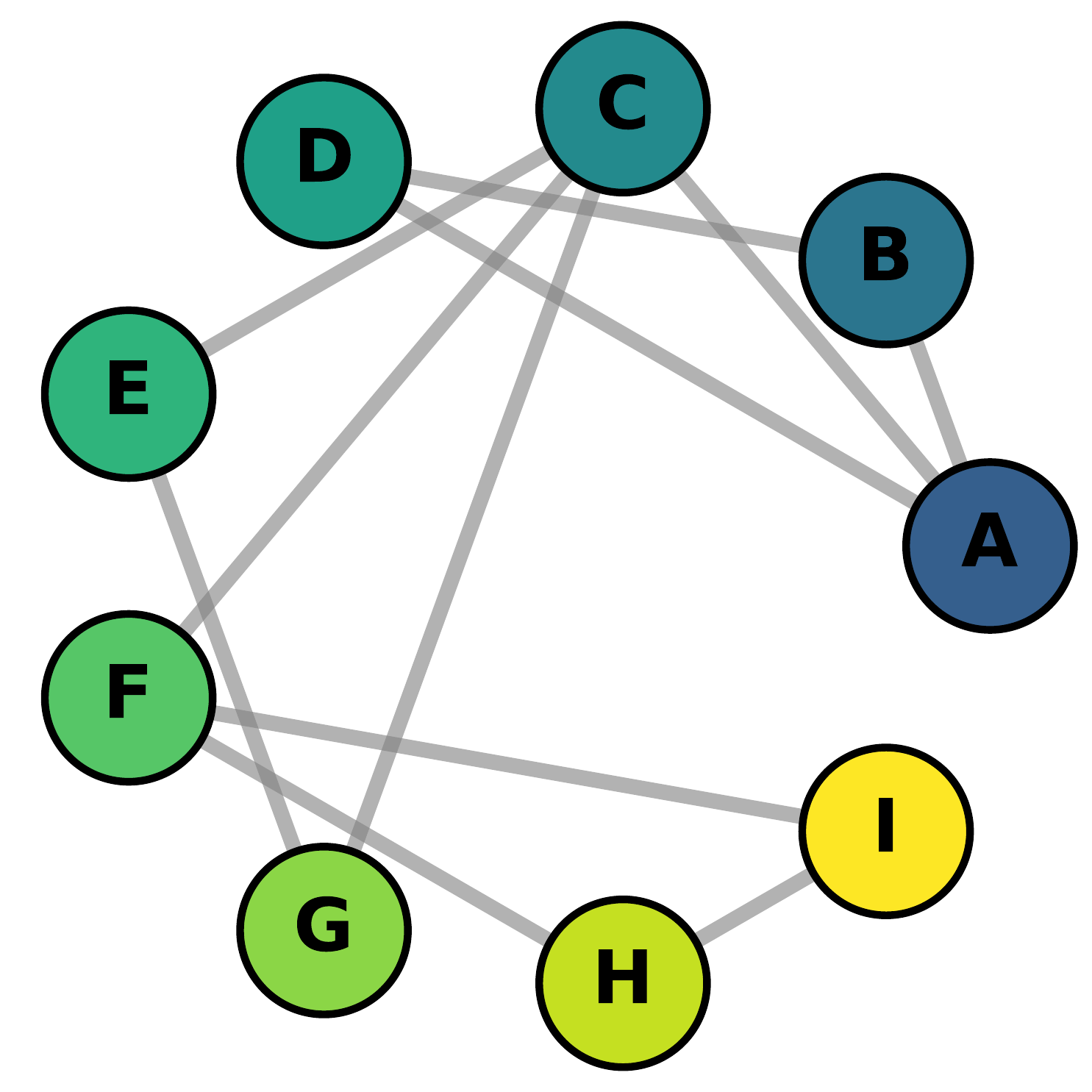}
    }
    \hfill
    \subfigure[GPT-3.5]{
        \includegraphics[width=0.3\linewidth]{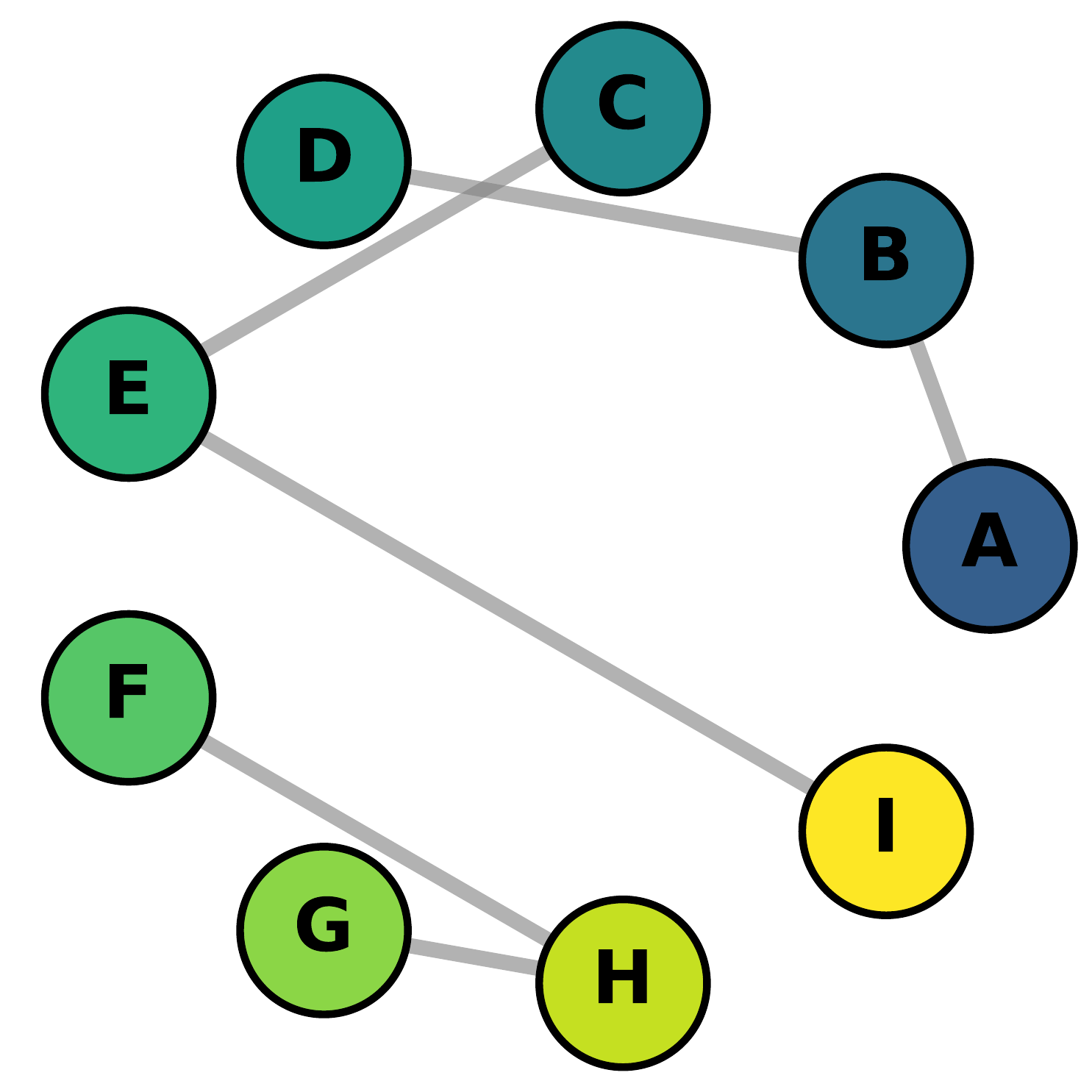}
    }
    \hfill
    \subfigure[GPT-4]{
        \includegraphics[width=0.3\linewidth]{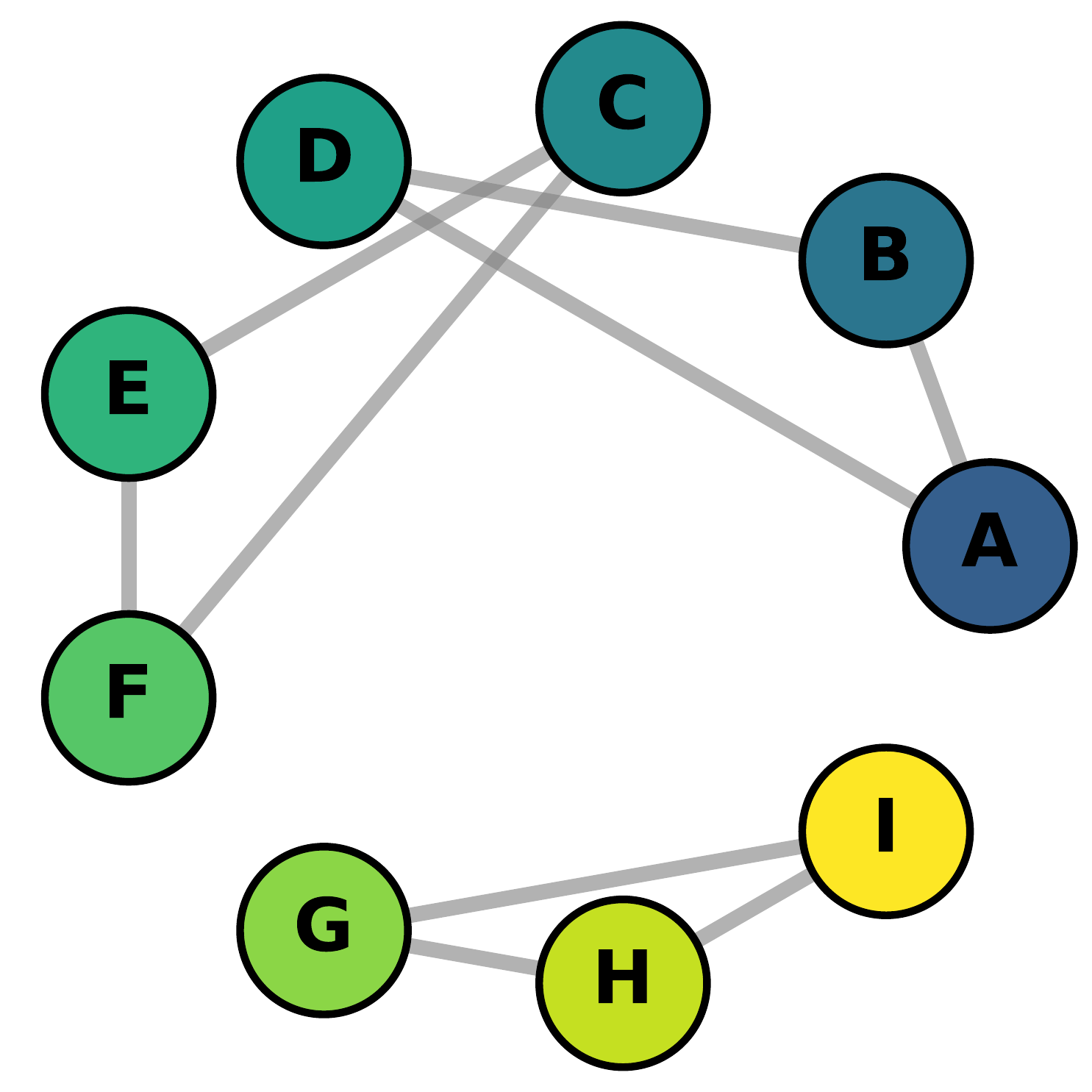}
    }
    \caption{Paper (WordNet).}
    \label{fig:appx-rwre-wordnet-paper}
\end{figure*}

\begin{figure*}[th]
    \centering
    \subfigure[Ground Truth]{
        \includegraphics[width=0.3\linewidth]{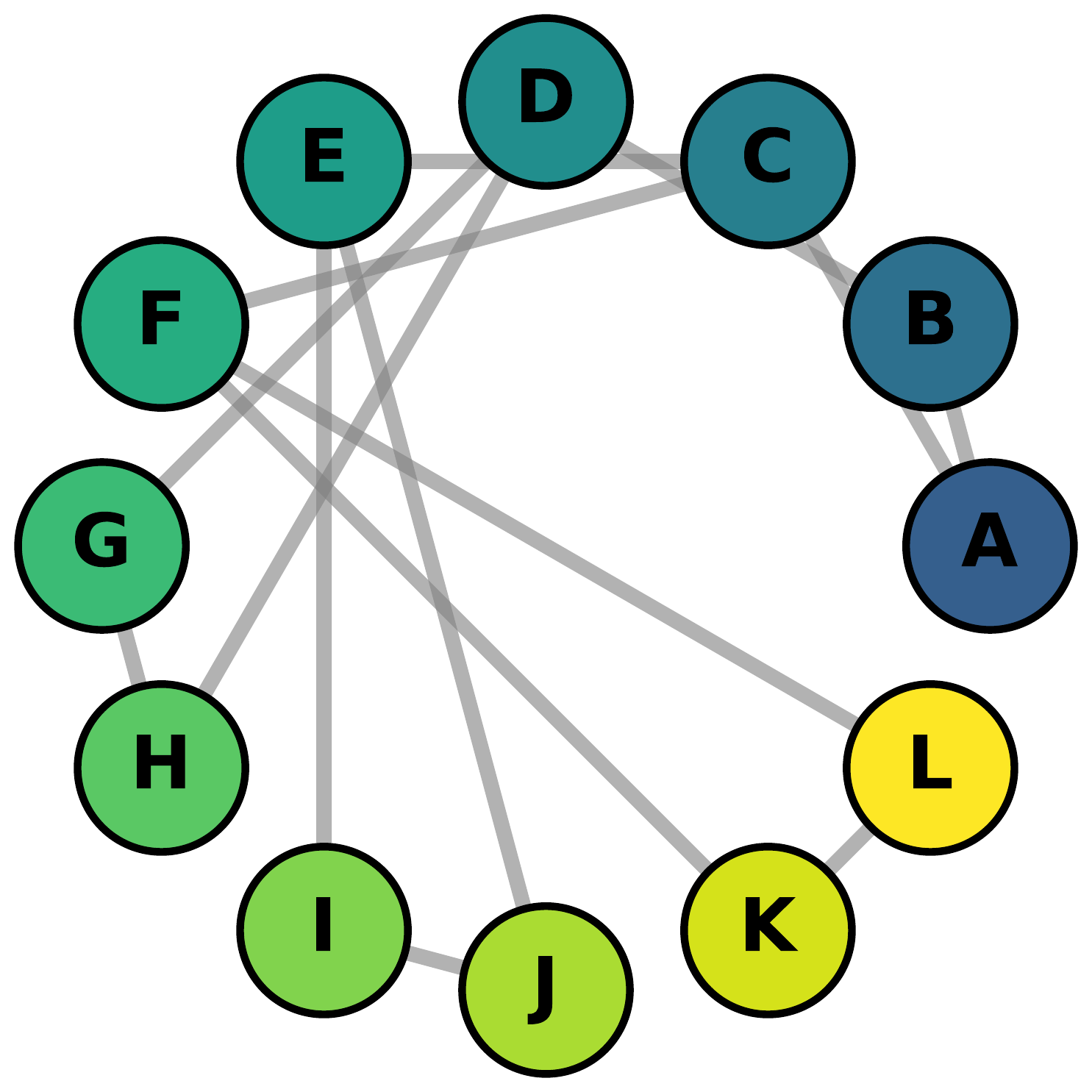}
    }
    \subfigure[GPT-3.5]{
        \includegraphics[width=0.3\linewidth]{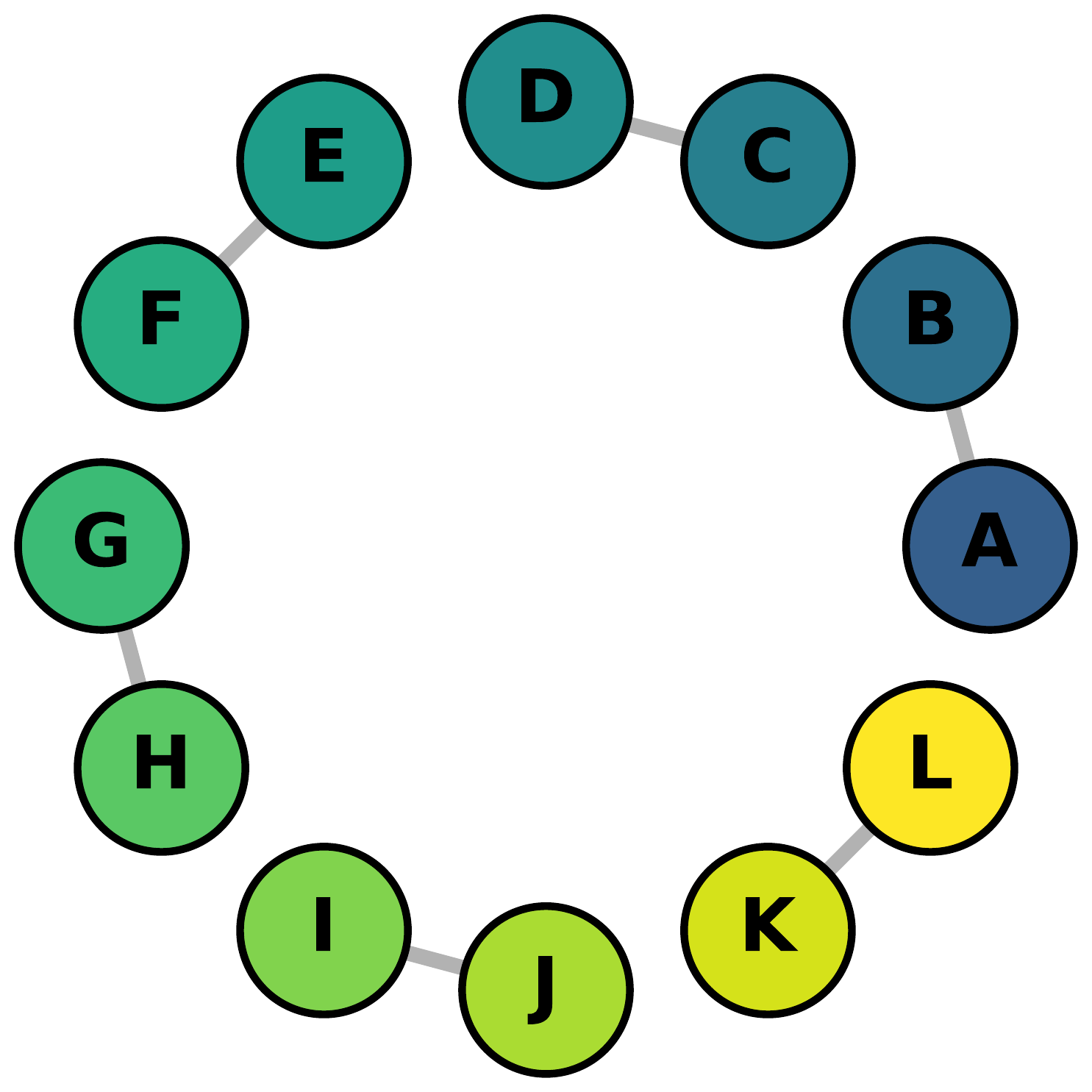}
    }
    \subfigure[GPT-4]{
        \includegraphics[width=0.3\linewidth]{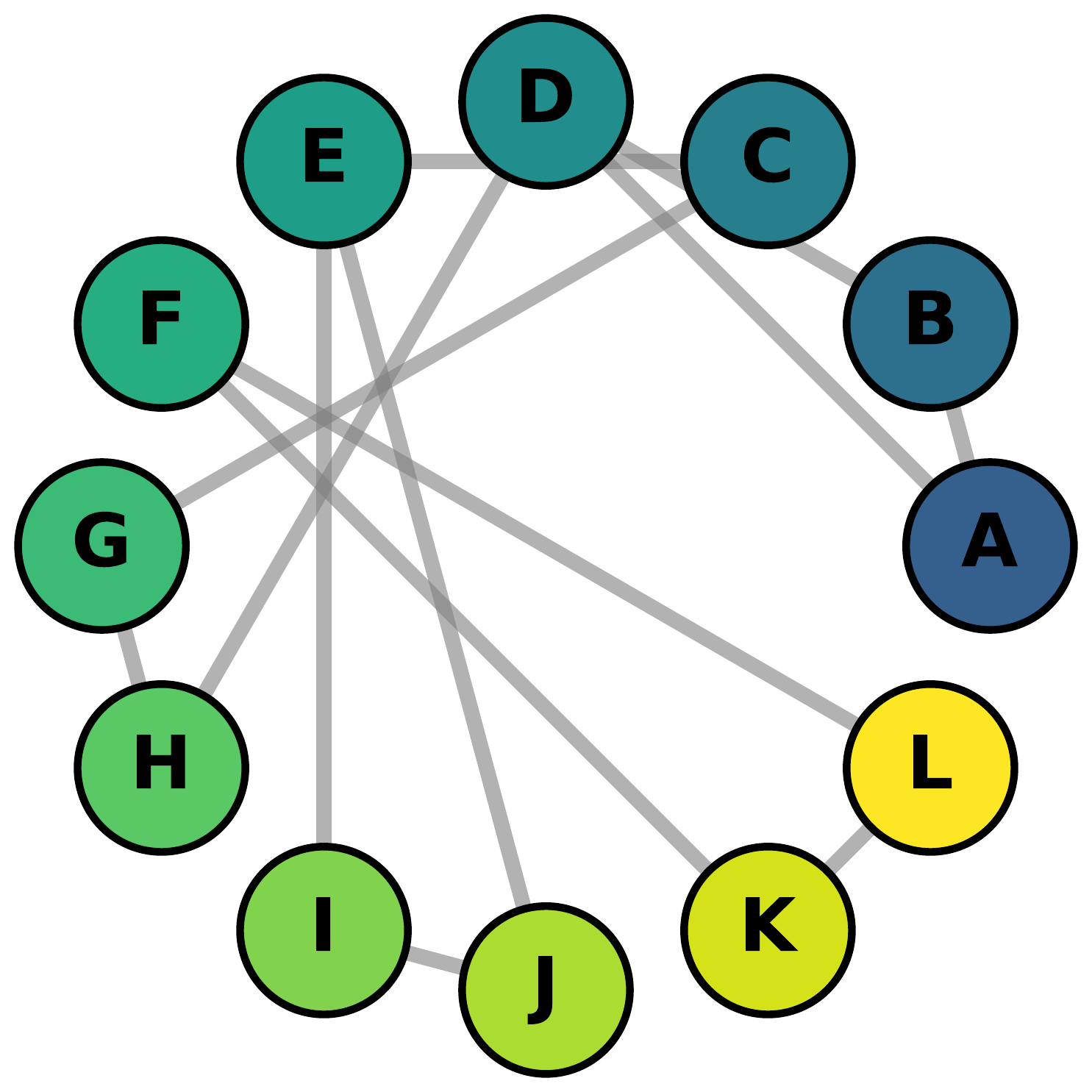}
    }
    \caption{Zoo (WordNet).}
    \label{fig:appx-rwre-wordnet-zoo}
\end{figure*}

\subsubsection{Relation Evaluation with Hypergraphs}
We focus on identifying pairwise relationships rather than the full hypergraph structures in \cref{subsec:experiments-rwre} primarily for the ease of making fair comparisons between the evaluated relations and the ground truth. To the best of our knowledge, there is no widely used database that characterizes entity relations in the form of hypergraphs. Although the evaluated hypergraphs cannot be directly compared with an established ground truth, the recovered relations align with our common knowledge.
The PTMs are only asked to identify the two most related entities in \cref{subsec:experiments-rwre}. We slightly adapt the prompts to make them recover the full hypergraph structures, instructing the PTMs to directly output the lists of hyperedges for the extracted entities in the prompts. We find that the PTMs are capable of reconstructing relational hypergraph structures that align with our existing knowledge. Figures~\ref{fig:appx-rwre-hypergraph-cake} -~\ref{fig:appx-rwre-hypergraph-zoo} are the evaluation results. The correspondences between the entities and the letters used in the figures are summarized in Tables~\ref{tab:entity-letter-1} and~\ref{tab:entity-letter-2}.

\begin{figure*}[th]
    \centering
    \includegraphics[width=0.5\textwidth]{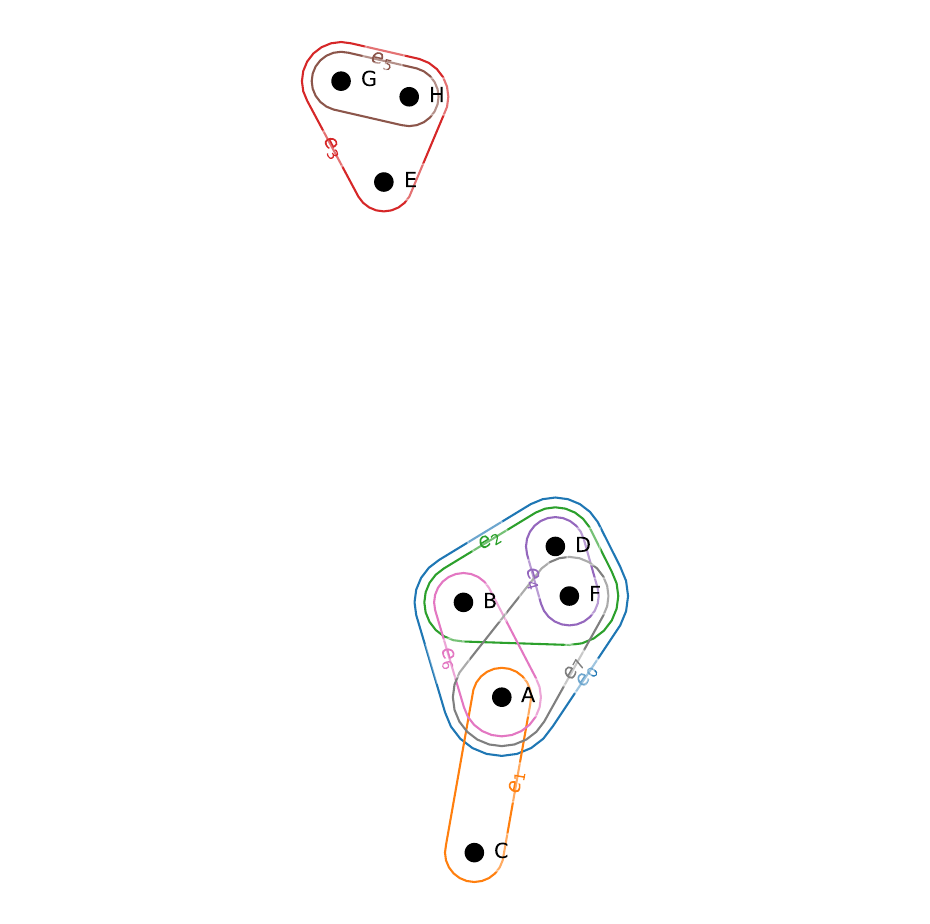}
    \caption{Cake.}
    \label{fig:appx-rwre-hypergraph-cake}
\end{figure*}

\begin{figure*}[th]
    \centering
    \includegraphics[width=0.5\textwidth]{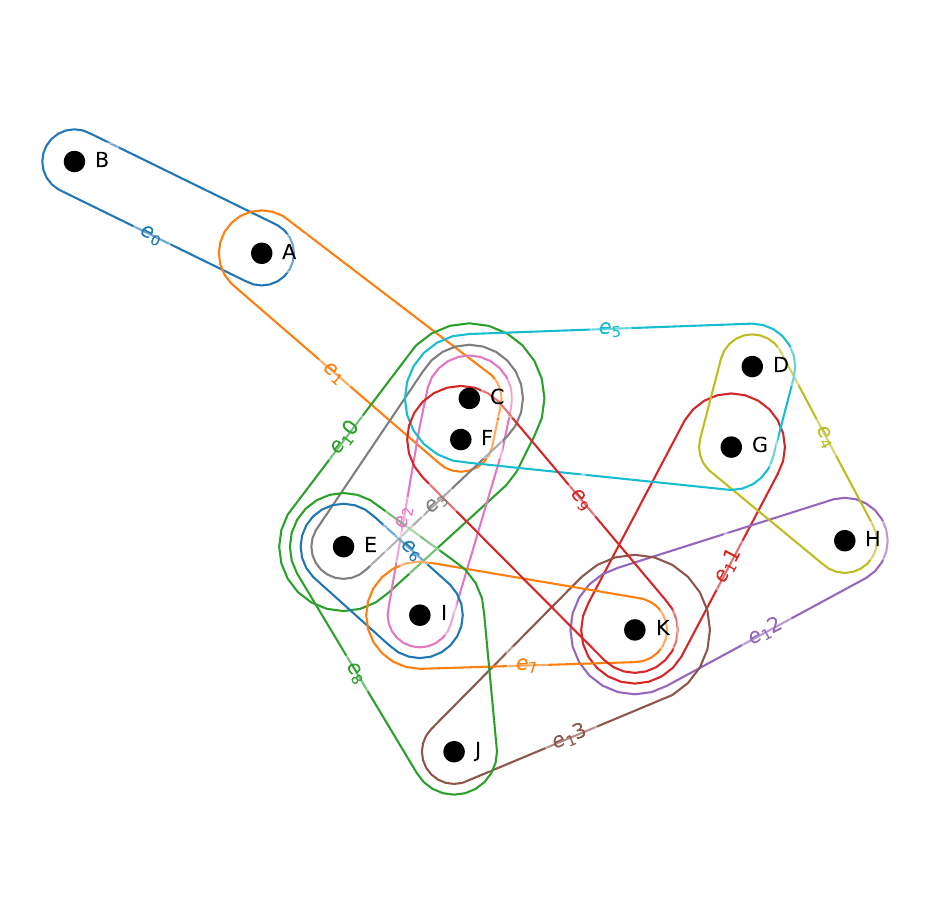}
    \caption{Dog.}
    \label{fig:appx-rwre-hypergraph-dog}
\end{figure*}

\begin{figure*}[th]
    \centering
    \includegraphics[width=0.5\textwidth]{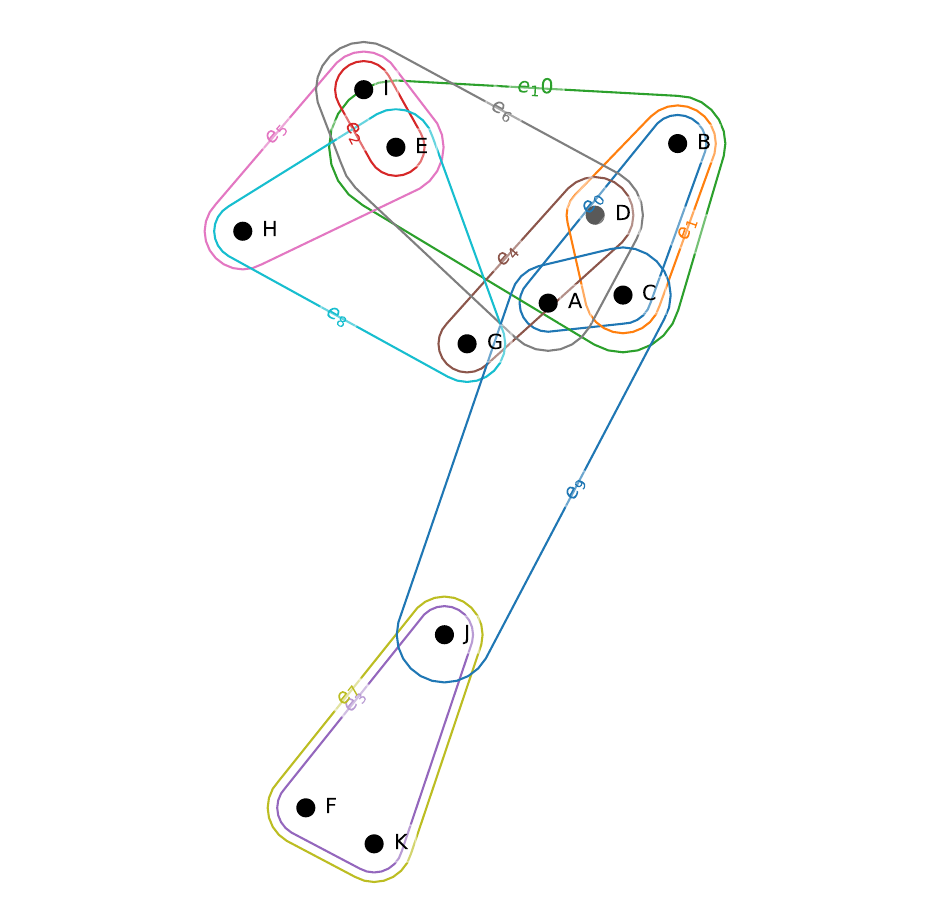}
    \caption{Fly.}
    \label{fig:appx-rwre-hypergraph-fly}
\end{figure*}

\begin{figure*}[th]
    \centering
    \includegraphics[width=0.5\textwidth]{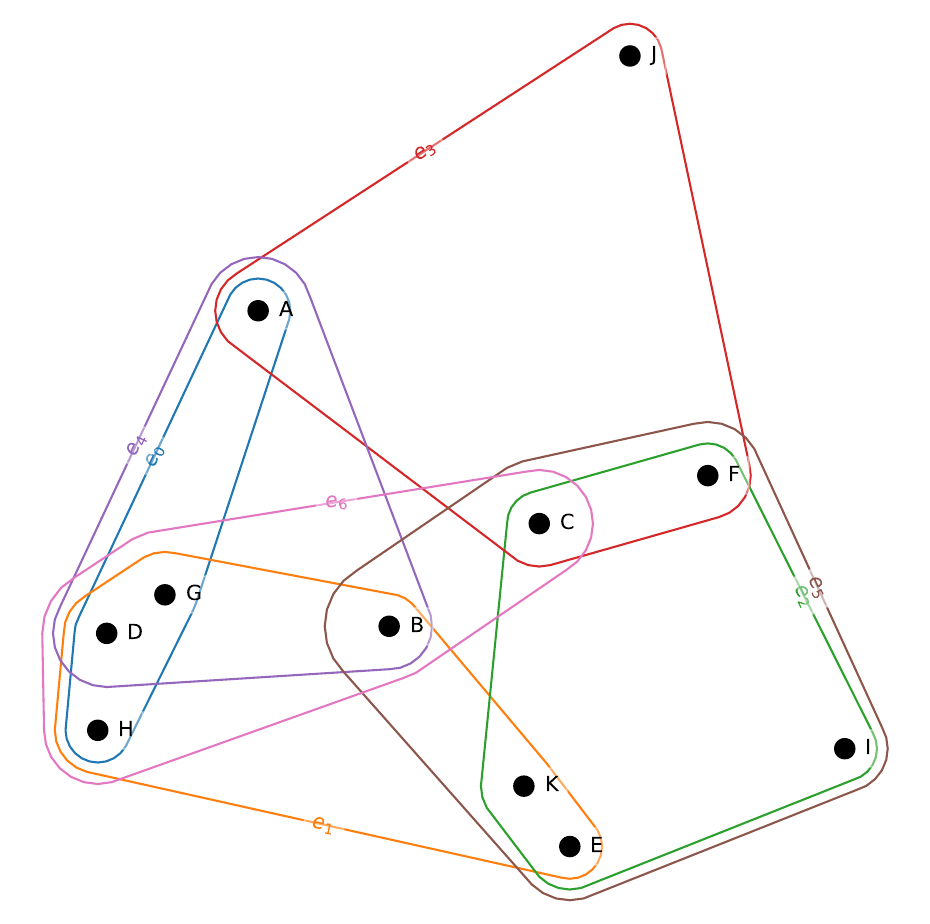}
    \caption{Human.}
    \label{fig:appx-rwre-hypergraph-human}
\end{figure*}

\begin{figure*}[th]
    \centering
    \includegraphics[width=0.5\textwidth]{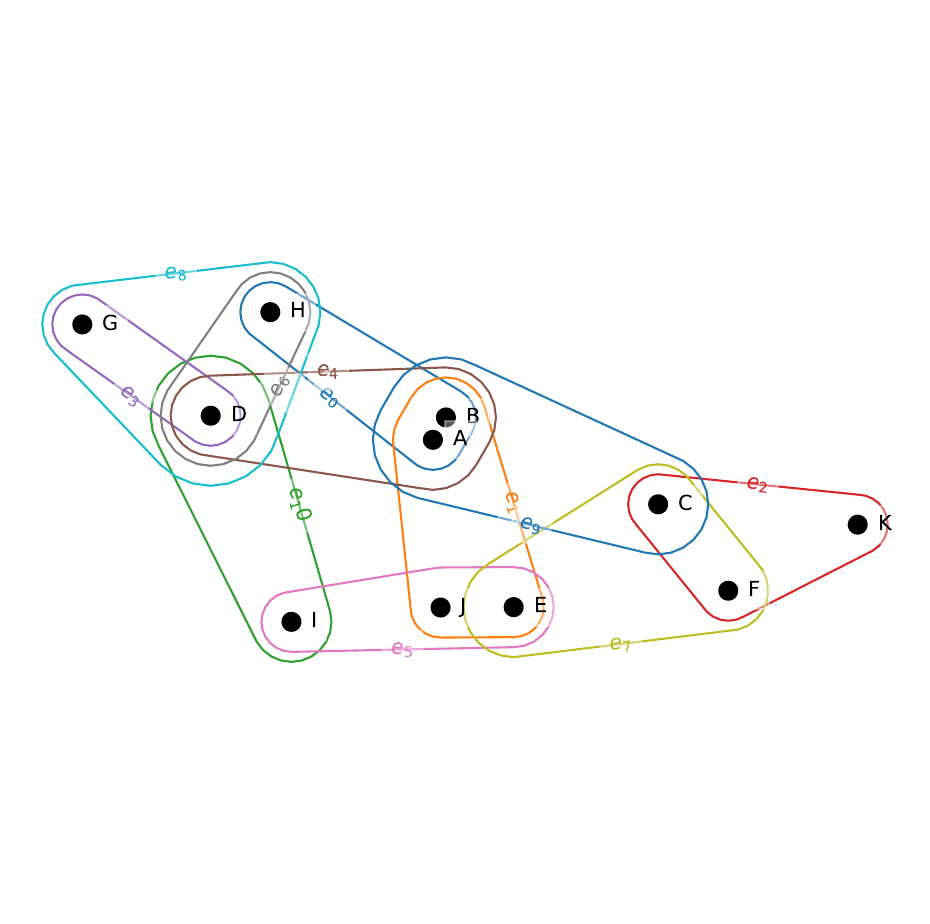}
    \caption{Jacket.}
    \label{fig:appx-rwre-hypergraph-jacket}
\end{figure*}

\begin{figure*}[th]
    \centering
    \includegraphics[width=0.5\textwidth]{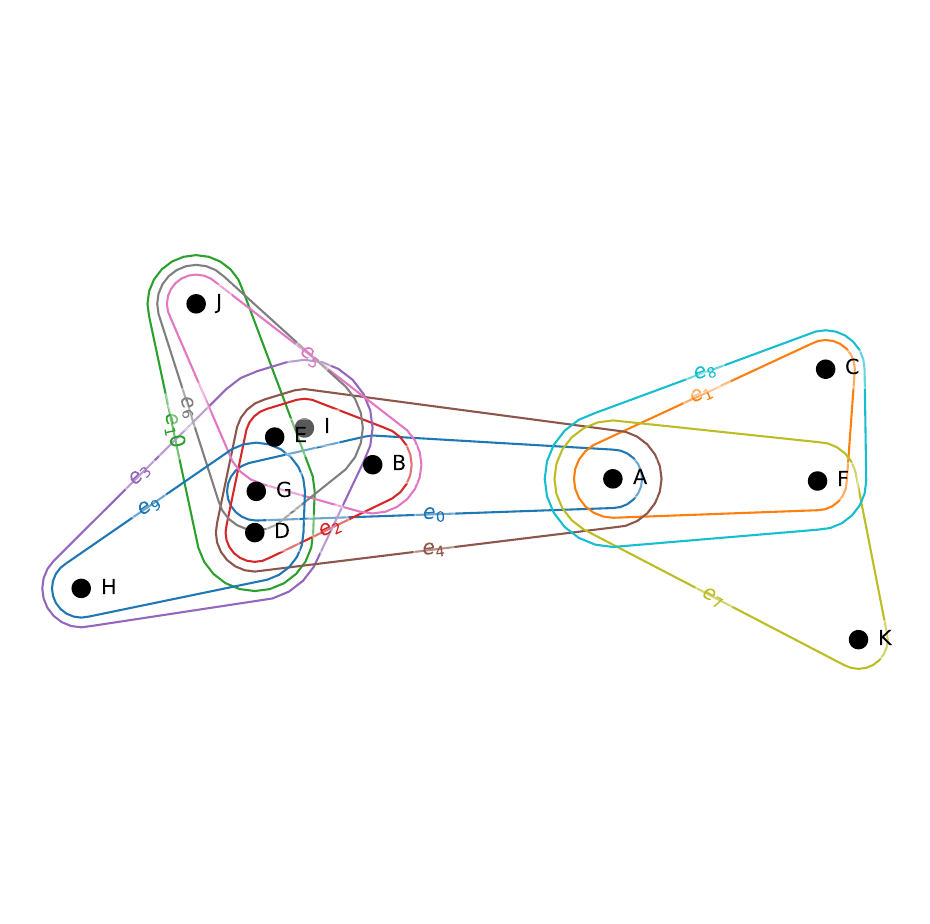}
    \caption{Orange.}
    \label{fig:appx-rwre-hypergraph-orange}
\end{figure*}

\begin{figure*}[th]
    \centering
    \includegraphics[width=0.5\textwidth]{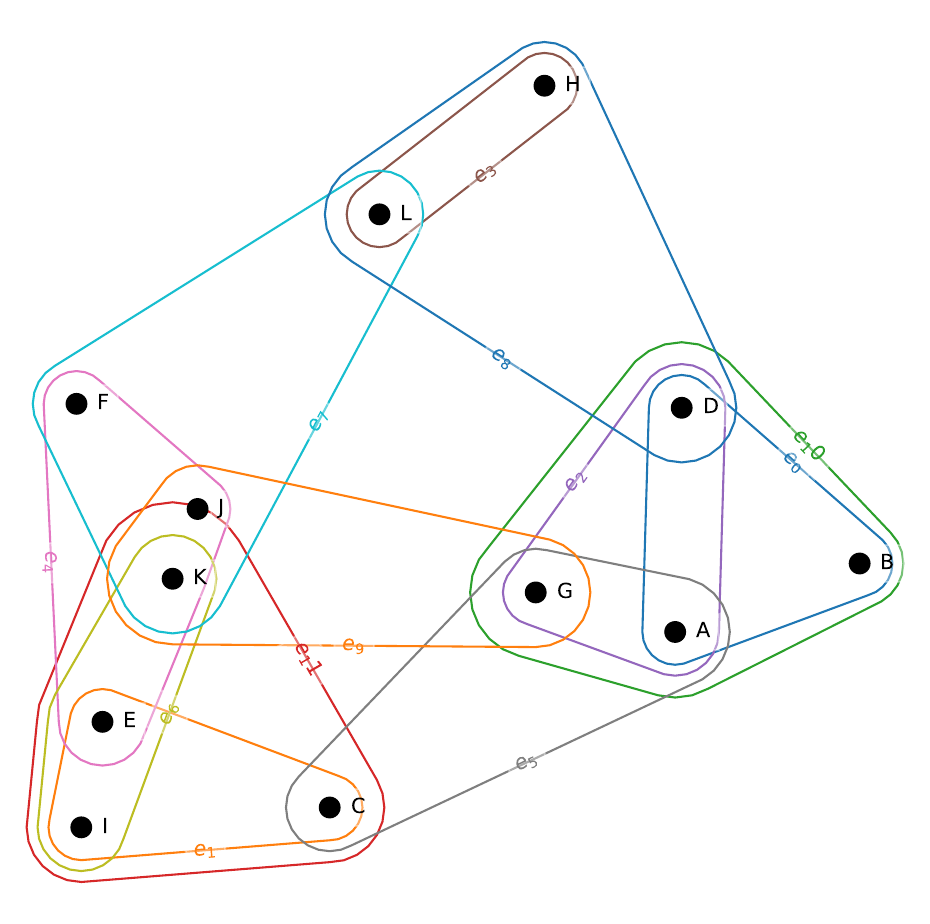}
    \caption{Paper.}
    \label{fig:appx-rwre-hypergraph-paper}
\end{figure*}

\begin{figure*}[th]
    \centering
    \includegraphics[width=0.5\textwidth]{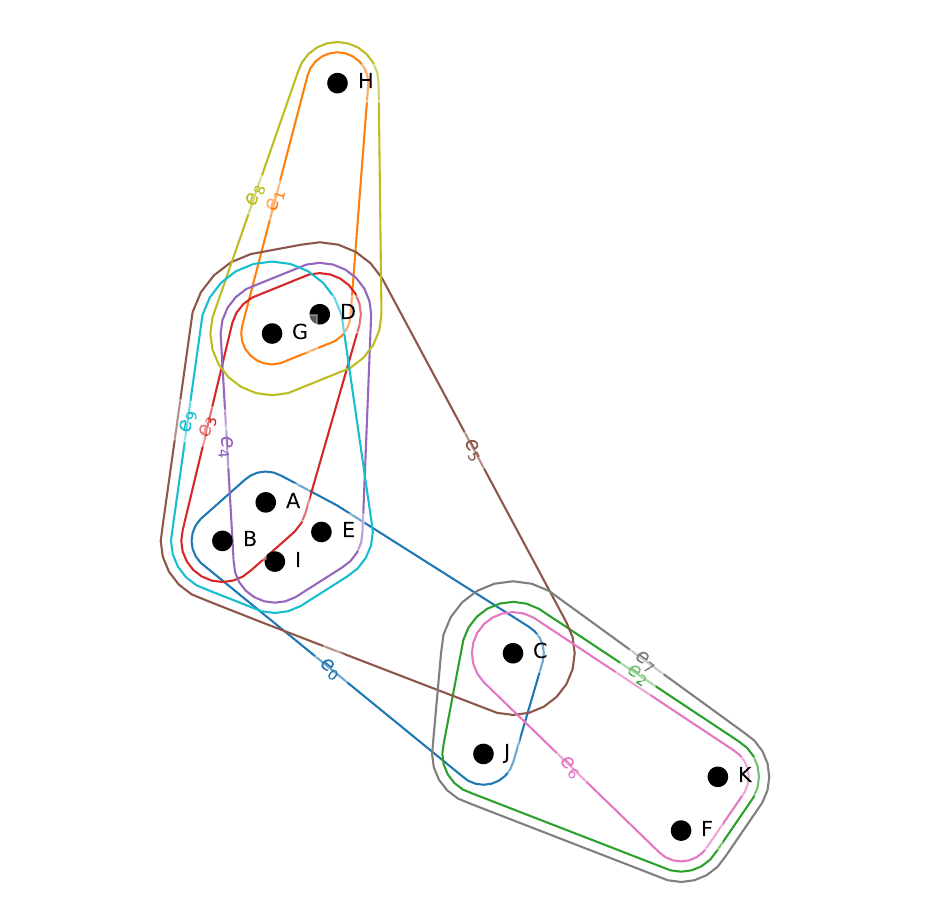}
    \caption{Sea.}
    \label{fig:appx-rwre-hypergraph-sea}
\end{figure*}

\begin{figure*}[th]
    \centering
    \includegraphics[width=0.5\textwidth]{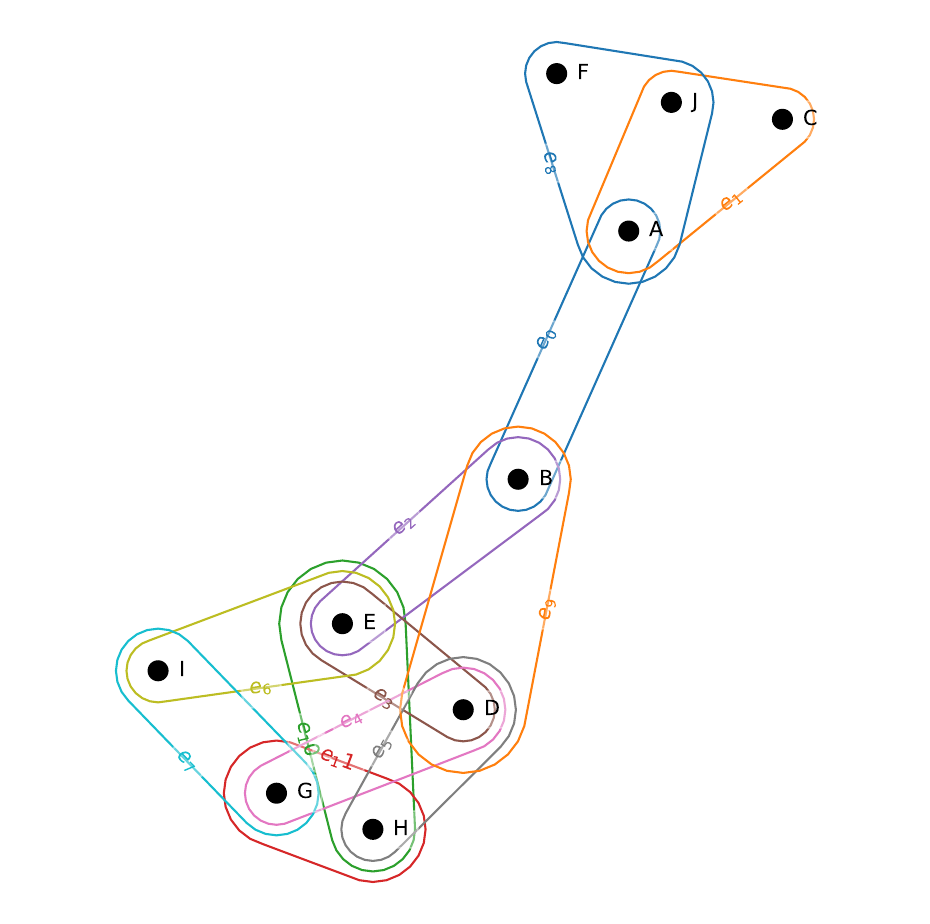}
    \caption{Table.}
    \label{fig:appx-rwre-hypergraph-table}
\end{figure*}

\begin{figure*}[th]
    \centering
    \includegraphics[width=0.5\textwidth]{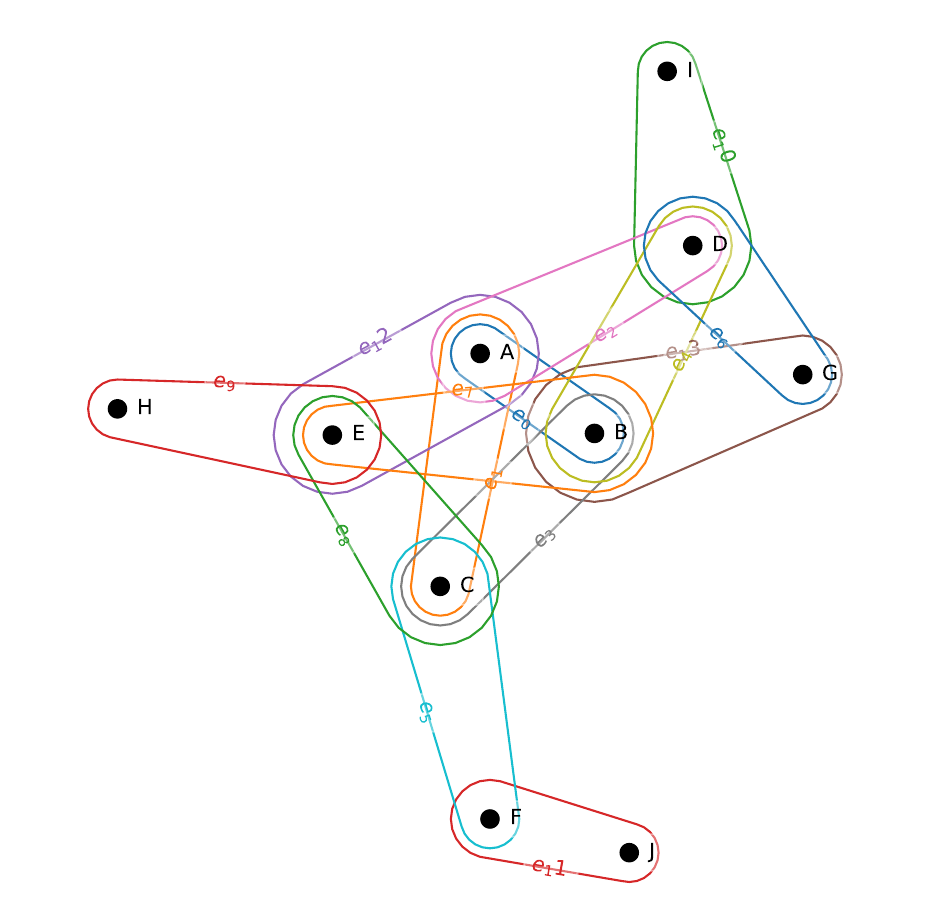}
    \caption{Zoo.}
    \label{fig:appx-rwre-hypergraph-zoo}
\end{figure*}


\end{document}